\documentclass{article}

\usepackage{graphicx} 
\usepackage{subfig}

\usepackage[numbers,sort&compress]{natbib}

\usepackage{algorithm}
\usepackage{algorithmic}

\usepackage[hyperindex,breaklinks]{hyperref}
\usepackage{url}
\usepackage{breakurl}
\usepackage{grffile}

\usepackage[accepted]{icml2010}
\usepackage{pdfsync}
\usepackage{placeins}
\usepackage{multirow}
\usepackage{array}
\usepackage{tabularx}
\usepackage{tikz}
\usetikzlibrary{shapes,shadows,positioning,calc,patterns}
\usepackage{ifthen}

\icmltitlerunning{Regularized Risk Minimization by Nesterov's Accelerated Gradient Methods}

\input{./Definitions/Definitions}
\graphicspath{{./Figures/}}

\newtheorem{assumption}[theorem]{Assumption}

\newcommand{\newfootnotemark}[1]{\addtocounter{footnote}{#1} \footnotemark[\value{footnote}]}
\newcommand{\newfootnotetext}[2]{\addtocounter{footnote}{#1} \footnotetext[\value{footnote}]{#2}}

\newboolean{pdflatexvar}
\setboolean{pdflatexvar}{false} 

\begin{document}

\twocolumn[
\icmltitle{Regularized Risk Minimization by Nesterov's Accelerated Gradient Methods: Algorithmic Extensions and Empirical Studies}

\icmlauthor{Xinhua Zhang}{xinhua.zhang.cs@gmail.com}
\icmladdress{Department of Computing Science, University of Alberta, Edmonton, AB T6G 2E8, Canada}
\icmlauthor{Ankan Saha}{ankans@cs.uchicago.edu}
\icmladdress{Department of Computer Science, University of Chicago, Chicago, IL 60637, USA}
\icmlauthor{S.V.N. Vishwanathan}{vishy@stat.purdue.edu}
\icmladdress{Department of Statistics and Department of Computer Science, Purdue University, IN 47906, USA}

\icmlkeywords{Convex optimization, Nesterov's methods, max-margin models, Bregman divergence}

\vskip 0.3in
]

\begin{abstract}
Nesterov's accelerated gradient methods (AGM) have been successfully applied in many machine learning areas.  However, their empirical performance on training max-margin models has been inferior to existing specialized solvers.  In this paper, we first extend AGM to strongly convex and composite objective functions with Bregman style prox-functions.  Our unifying framework covers both the $\infty$-memory and 1-memory styles of AGM, tunes the Lipschiz constant adaptively, and bounds the duality gap.  Then we demonstrate various ways to apply this framework of methods to a wide range of machine learning problems.  Emphasis will be given on their rate of convergence and how to efficiently compute the gradient and optimize the models.  The experimental results show that with our extensions AGM outperforms state-of-the-art solvers on max-margin models.
\end{abstract}

\section{Introduction}

There has been an explosion of interest in machine learning over the
past decade, much of which has been fueled by the phenomenal success of
binary Support Vector Machines (SVMs). Driven by numerous applications,
recently, there has been increasing interest in support vector learning
with linear models.  At the heart of SVMs is the following regularized
risk minimization (RRM) problem:

\begin{align}
  \label{eq:reg-risk}
  \min_{\wb} J(\wb) &:= \underbrace{\lambda \Omega(\wvec)}_{\text{regularizer}} +
  \underbrace{R_{\emp}(\wb)}_{\text{empirical risk}} \\
  \text{with} \quad \Omega(\wvec) &:= \frac{1}{2} \nbr{\wvec}_2^2 \\
   R_{\emp}(\wb) &:= \frac{1}{n} \max_{b \in \RR} \sum_{i=1}^{n} [1 -
  y_{i} (\inner{\wb}{\xb_{i}}+b)]_+,
\end{align}
where $[x]_+ = x$ if $x \ge 0$ and 0 otherwise.  Here we assume access to a training set of $n$ labeled examples $\{(\xb_i, y_i)\}_{i=1}^{n}$ where $\xvec_i \in \RR^p$ and $y_i \in \cbr{-1,+1}$, and use the half square Euclidean norm $\nbr{\wb}_2^{2} = \sum_i w_i^2$ as the regularizer. The parameter $\lambda$ controls the trade-off between the empirical risk and the regularizer.

There has been significant research devoted to developing specialized optimizers which minimize $J(\wb)$ efficiently.  \citet{ZhaSahVis10a} proved that cutting plane and bundle methods may require at least $O(np/\epsilon)$ computational efforts to find an $\epsilon$ accurate solution to \eqref{eq:reg-risk}, and they suggested using Nesterov's accelerated gradient method (AGM) which provably costs $O(np / \sqrt{\epsilon})$ time complexity.  In general, AGM takes $O(1/\sqrt{\epsilon})$ times of gradient query to find an $\epsilon$ accurate solution to
\begin{align}
\label{eq:obj_agm}
\min_{\xvec \in Q} f(\xvec),
\end{align}
where $f$ is convex and has $L$-Lipschitz continuous gradient ($L$-\lcg), and $Q$ is a closed convex set in the Euclidean space.  AGM is especially suitable for large scale optimization problems because each iteration it only requires the gradient of $f$.

Unfortunately, despite some successful application of AGM in learning sparse models \citep{Lu09,LiuCheYe09} and game playing \citep{GilSanSor08}, it does not compare favorably to existing specialized optimizers when applied to training large margin models \citep{ZhaSahVis09}.  It turns out that special structures exist in those problems, and to make full use of AGM, one must utilize the computational and statistical properties of the learning problem by properly reformulating the objectives and tailoring the optimizers accordingly.

To this end, our first contribution is to show that in both theory and practice smoothing $R_{\emp}(\wvec)$ as in \citep{Nesterov05} is advantageous to the primal-dual versions of AGM.  The dual of \eqref{eq:reg-risk} is
\begin{align}
  \label{eq:adjoint}
  \max_{\val} D(\alphab) &= \sum_i \alpha_i - \frac{1}{2 \lambda} \alphab^{\top} Y
  X^{\top} X Y \alphab, \\
  s.t. \ \ \alphavec \in Q_2 &:= \Big\{\alphab
  \in [0, n^{-1}]^n : \sum_i y_i \alpha_i = 0 \Big\}.
\end{align}
Comparing \eqref{eq:obj_agm} with \eqref{eq:reg-risk} and \eqref{eq:adjoint}, it seems more natural to apply AGM to \eqref{eq:adjoint} because it is smooth.  However in practice, most $\alpha_i$ at the optimum will be on the boundary of $[0, n^{-1}]$.  According to \citep{Platt99a}, such $\alpha_i$'s are easy to identify and so the corresponding entries in the gradient are wasted by AGM.  This structure of support vector is unique for max-margin models, which will also be manifested in our experiments (Section \ref{sec:experiment}).

In contrast, smoothing $R_{\emp}$ has a lot of advantages.  First, it directly optimizes in the primal $J$, avoiding the indirect translation from the dual solution to the primal.  Second, the resulting optimization problem is unconstrained.  If $\Omega$ is strongly convex, then linear convergence can be achieved.  Third, gradient of the smoothed $\Rtil_{\emp}$ can often be computed efficiently, and details will be given in Section \ref{sec:computation_trick}.  Fourth, the diameter of the dual space $Q_2$ often grows slowly with $n$, or even decreases.  This allows using a loose smoothing parameter.  Fifth, in practice most $\alpha_i$ at the optimum are 0, where $\Rtil_{\emp}$ best approximates $R_{\emp}$.  Therefore, the approximation is actually much tighter than the worst case theoretical bound, and a good solution for $\Rtil_{\emp}$ is more likely to optimize $R_{\emp}$ too.  Last but most important, the smoothed $\Rtil_{\emp}$ themselves are reasonable risk measures \citep{JorBarMcA06}, which also deliver good generalization performance in statistics.  Now that it is much easier to optimize the smoothed objectives, a model which generalizes well can be quickly obtained with the homotopy scheme (\ie\ anneal the smoothing parameter).

Using the same idea of smoothing $R_{\emp}$, AGM can be applied to a much wider variety of RRM problems by utilizing its composite structure.  Given a model $\psi$ of $\Rtil$, if $\Omega(\wvec) + \psi(\wvec)$ can be solved efficiently, then \citep{Tseng08} showed that $\Omega(\wvec) + \Rtil(\wvec)$ can be solved in $O(1/\sqrt{\epsilon})$ steps, even if $\Omega$ is not differentiable, \eg\ $L_1$ norm \cite{BecTeb09}.  Similar approach is applied to the $L_{1,\infty}$ regularizer and the elastic net \citep{ZouHas05} regularizer by \citep{MaiBacPonSap10}:
\begin{align}
  \label{eq:reg_elastic_net}
  \Omega(\wvec) = \frac{\gamma}{2} \nbr{\wvec}_2^2 + \nbr{\wvec}_1 = \frac{\gamma}{2} \sum_i w_i^2 + \sum_i |w_i|.
\end{align}
This $\Omega$ is strongly convex with respect to (wrt) the $L_2$ norm, and similarly in many RRM problems $\Omega$ is strongly convex wrt some norm $\nbr{\cdot}$.  For example, the relative entropy regularizer in boosting \citep{WarGloVis08}:
\begin{align}
  \label{eq:reg_entropy}
  \Omega(\wvec) = \sum_i w_i \log w_i
\end{align}
is strongly convex wrt $L_1$ norm, and the log determinant of a matrix in \citep{dAsBanElG08,JaiKulDhiGra09,KulBar10}:
\begin{align}
  \label{eq:mat_log_det}
  \Omega(W) = -\log \det W
\end{align}
is strongly convex wrt the Frobenius norm.  By exploiting the strong convexity, \citep{Nesterov07} accelerated the convergence rate from $O(1/\sqrt{\epsilon})$ to $O(\log \frac{1}{\epsilon})$.  However, the prox-function in this case must be strongly convex wrt $\nbr{\cdot}$ too.  Existing methods either ignore the strong convexity in $\Omega$ \citep{Tseng08}, or restrict the norm to $L_2$ \citep{Nesterov07,BecTeb09}.  As one major contribution of this paper, we extend AGM to exploit this strong convexity in the context of Bregman divergence.  In particular, we allow $\Omega$ to be strongly convex wrt a Bregman divergence induced by a smooth convex function $d$ (to be formalized later), where $d$ is in turn strongly convex wrt certain norm $\nbr{\cdot}$.  By using $d$ as a prox-function, we manage to achieve linear convergence for a wide range of RRM problems.

There are two types of first order methods that both achieve the optimal rate. The first type is the original AGM pioneered by Nesterov \citep{Nesterov83,Nesterov03a,Nesterov05,Nesterov05a,Nesterov07}, which uses a sequence of {\bf e}stimation {\bf f}unctions (hence we call it AGM-EF).  In particular, it uses the whole past iterates to progressively build a sequence of estimate functions which approximate the objective function.  The second type was developed by a number of other researchers and a unified treatment was given by \citep{Tseng08}.  Intuitively, it generalizes the idea of gradient descent by {\bf p}roximal {\bf r}egularization (hence we call it AGM-PR), which can be further accelerated by momentum.  Therefore, these two types of methods are different in concept.  In addition, both AGM-EF and AGM-PR a $\infty$-memory version which builds a model of the objective by using \emph{all} the past gradients, and a 1-memory version which approximates that model by a \emph{single} Bregman divergence.

We choose to base our extensions on AGM-EF, because compared with AGM-PR it provides much more flexibility in adaptively tuning $L$.\footnote{All APM-PR variants with adaptive $L$, \eg\ \citep{Tseng08,Nemirovski94,BecTeb09,PonTseJiYe10}, require the estimate of $L$ grow monotonically through iterations.  And their technique does not extend to asymmetric Bregman divergence.}  This is because the inductive relationship maintained by AGM-EF involves a single iteration, while that for AGM-PR involves two successive steps.  The novelty and generality of our method in the context of existing methods are summarized in Table \ref{tab:novelty_agmef}.  We further provide bounds on the duality gap which amounts to effective termination criteria.  As another important contribution, we derive \emph{linear convergence for the duality gap} in the context of strong convexity.  Computationally, at each iteration our method requires only one projection and one gradient evaluation within the feasible region.\footnote{Some AGM algorithms require two projections \citep{Nesterov05} or two gradients \citep{Nesterov07} per iteration, or evaluate the gradient outside the feasible region \citep[][Section 2.2.4]{Nesterov03a}.}

\begin{table}[t]
  \centering
  \begin{tabularx}{\linewidth}{cc|cc|cc}
  & & \multicolumn{2}{c|}{No composite} & \multicolumn{2}{c}{Composite} \\
  & & cvx & sc & cvx & sc \\
  \hline
 \multirow{2}{*}{Euclidean} & 1-memory & \citep{Nesterov03a} & \citep{Nesterov03a} & $\times$ & $\times$ \\
  & $\infty$-memory &  \citep{Nesterov05} & \citep{Nesterov07} &  \citep{Nesterov07} &  \citep{Nesterov07} \\
  \hline
 \multirow{2}{*}{Bregman} & 1-memory & \citep{AusTeb06} & $\times$ & $\times$  & $\times$  \\
  & $\infty$-memory &  \citep{Nesterov05} & $\times$ &  $\times$ &  $\times$
  \end{tabularx}
  \caption{Summary of AGM-EF.  ``sc" means strongly convex and ``cvx" means just convex. $\times$ means novel contribution of this paper.  AGM-PR can handle all but sc.}\label{tab:novelty_agmef}
\end{table}

\paragraph{Outline of the paper.}
In Section \ref{sec:preliminaries}, we follow \citep[][Section 4.1, Definition 3]{Shalev-Shwartz07} to extend the concept of strong convexity to the context of Bregman divergence.  We show several properties that will play a key role in the subsequent development of the new algorithms.  In Section \ref{sec:inf_mem_nest} and \ref{sec:one_mem_nest}, two novel variants of AGM-EF are developed along the lines of $\infty$-memory and 1-memory.  They both achieve global linear convergence by utilizing the Bregman generalized strong convexity in either $\Omega$ or $R_{\emp}$.  Section \ref{sec:app_rrm} elaborates on how to \emph{effectively} apply our method to solve Bregman regularized risk minimization problems, and many examples of machine learning models are discussed.  Also presented is the algorithms which \emph{efficiently} compute the gradient and solve the model.  Experimental results are given in Section \ref{sec:experiment}, where we show empirically that by smoothing $R_{\emp}$ and exploiting the generalized strong convexity in $\Omega$, the $L_2$ and entropic regularized risk minimization problems can be solved significantly faster than the state-of-the-art optimizers.

A ready reckoner of the convex analysis concepts used in the paper can be found in Appendix \ref{sec:app:convex_ana}.

\section{Preliminaries}
\label{sec:preliminaries}

From the optimization perspective, the objectives considered in this paper have the same form as in \citep{Tseng08}.  Let $\RR^p$ be endowed with a norm $\nbr{\cdot}$.  Consider the following nonsmooth convex objective:
\begin{align}
\label{eq:primal_obj_composite}
  \min_{\xvec} J(\xvec) = f(\xvec) + \Psi(\xvec),
\end{align}
where $\Psi: \RR^p \mapsto \RRbar := (-\infty, +\infty]$ and $f: \RR^p \mapsto \RRbar$ are proper, lower semicontinuous (lsc) and convex.  Assume $\dom \Psi$ is closed, $f$ is differentiable on an open set containing $\dom \Psi$, and $\grad f$ is Lipschitz continuous on $\dom \Psi$, \ie\ there exists $L > 0$ such that
\[
\nbr{\grad f (\xvec) - \grad f(\yvec)}^* \le L \nbr{\xvec - \yvec} \quad \xvec, \yvec \in \dom \Psi.
\]

Some special cases are in order.  The first is constrained smooth optimization, where $\Psi$ is the indicator function for a nonempty closed convex set $Q \subseteq \RR^p$:
\begin{align*}
\Psi(\xvec) = \begin{cases}
  0 & \text{if } \xvec \in Q \\
  +\infty & \text{otherwise}
\end{cases}.
\end{align*}
Therefore, in the sequel we will always discuss unconstrained minimization for $J(\wvec)$, although this is just a matter of notation.  A second example is the $L_1$ regularization, where
\[
\Psi (\xvec) = \sum_{i=1}^p \abr{x_i}.
\]
In fact, many machine learning problems are special cases of \eqref{eq:primal_obj_composite} and details can be found in Section \ref{sec:app_rrm} and \citep[][Table 5]{TeoVisSmoLe10}.

Next, we will present in detail two additional assumptions: strong convexity of $f$ and $\Psi$ in the sense of Bregman divergence, and efficiently solvable ground optimization problems.

\subsection{Extending strong convexity to Bregman divergence}

Let $d$ be a differentiable and $\sigma$ strongly convex function with respect to some norm $\nbr{\cdot}$.\footnote{AGM capitalizes on two properties of the norm: convexity and linearity ($\nbr{c \cdot \xvec} = \abr{c} \nbr{\xvec}$).}  Then we can define a Bregman divergence:
\[
\Delta_d(\xvec, \yvec) = d(\xvec) - d(\yvec) - \inner{\grad d(\yvec)}{\xvec - \yvec}.
\]
By the definition of $\sigma$-sc, we have
\[
\Delta_d(\xvec, \yvec) \ge \frac{\sigma}{2} \nbr{\xvec - \yvec}^2, \quad \text{for all } \xvec, \yvec.
\]
Furthermore, Bregman divergence can be used to generalize the concept of strong convexity \citep[][Definition 3, Chapter 4]{Shalev-Shwartz07}.
\begin{definition}[Strong convexity for Bregman divergence]
\label{def:gen_sc}
A convex function $f$ is said to be $\lambda$ strongly convex with respect to $d$ ($\lambda$-sc wrt $d$) with $\lambda \ge 0$ if for all $\xvec$ and $\yvec$ we have
\[
f(\xvec) \ge f(\yvec) + \inner{\gvec}{\xvec - \yvec} + \lambda \Delta_d(\xvec, \yvec) \ \ \forall \ \gvec \in \partial f(\yvec).
\]
If $\lambda > 0$, we say $f$ is strictly strongly convex.
\end{definition}


For example, with $d(\xvec) = \frac{1}{2} \nbr{\xvec}^2$ where the norm is Euclidean, we recover the conventional strong convexity.  Here we allow $\lambda$ to be 0 for a unified exposition, and trivially all convex functions are $0$-sc wrt any $d$.  It is noteworthy that Definition \ref{def:gen_sc} preserves some important properties of the conventional strong convexity.
\begin{property}
\label{prpty:two_point_sc}
  If $f$ is $\lambda$-sc wrt $d$, then $f$ must be $\lambda \sigma$-sc wrt $\nbr{\cdot}$. Hence for any $\alpha \in [0, 1]$ and $\xvec, \yvec$, we have
  \begin{align*}
  f(\alpha \xvec + (1 - \alpha) \yvec) &\le \alpha f(\xvec) + (1 - \alpha) f(\yvec) \\
  &\qquad \quad - \frac{\lambda \sigma }{2} \alpha (1 - \alpha) \nbr{\xvec - \yvec}^2.
  \end{align*}
\end{property}
\begin{property}
\label{prpty:sum_sc}
  If $\alpha_i \ge 0$ and $f_i$ is $\lambda_i$-sc wrt $d$ ($\lambda_i \ge 0$), then $\sum_i \alpha_i f_i$ is $\sum_i \alpha_i \lambda_i$-sc wrt $d$.
\end{property}
\begin{property}
\label{prpty:breg_div_sc}
  $d(\xvec)$ is 1-sc wrt $d$.  So by Property \ref{prpty:sum_sc}, $\Delta_d(\xvec, \xvec_0)$ is also 1-sc wrt $d$ for any fixed $\xvec_0$.
\end{property}

Many problems are constrained to a feasible region $Q$.  In the sequel we will always assume that $Q \subseteq \dom d$ and $Q$ is closed and convex.
\begin{property}
\label{prpty:sc_min_quad}
  Suppose $f: \RR^n \mapsto \RRbar$ is proper, lsc, and $\lambda$-sc wrt $d$ and $\xvec^* = \argmin_{\xvec} f(\xvec)$.  Then
  \[
  f(\xvec) - f(\xvec^*) \ge \lambda \Delta_d(\xvec, \xvec^*) \quad \text{for all } \xvec \in \dom f.
  \]
\end{property}
The proof simply uses the definition of $\lambda$-sc and the optimality condition of $\xvec^*$: $\inner{\gvec}{\xvec - \xvec^*} \ge 0$ for all $\gvec \in \partial f(\xvec^*)$ and $\xvec \in \dom f$.

A direct application of Property \ref{prpty:sum_sc}, \ref{prpty:breg_div_sc} and \ref{prpty:sc_min_quad} gives a very important inequality which is also used extensively in \citep[][Property 1]{Tseng08} and \citep[][Lemma 6]{LanLuMon09}:
\begin{property}
\label{prpty:sc_min_key_one}
  Suppose $f$ is proper, lsc, and convex with range $\RRbar$.  Let $\xvec^* = \argmax_{\xvec} f(\xvec) + \Delta_d(\xvec, \xvec_0)$, then for all $\xvec$
  \[
  f(\xvec) + \Delta_d(\xvec, \xvec_0) \ge f(\xvec^*) + \Delta_d(\xvec^*, \xvec_0) + \Delta_d(\xvec, \xvec^*).
  \]
\end{property}

The following property of Bregman divergence plays a key role in keeping a compact expression of our estimation functions.
\begin{property}
\label{prpty:simplify_bregman}
  For all $\alpha_i \ge 0$ and $\xvec_i$ in the interior of $\dom d$, define
  \begin{align*}
  q(\xvec) := \inner{\svec}{\xvec} + \sum_i \alpha_i \Delta_d(\xvec, \xvec_i).
  \end{align*}
  Then $q(\xvec)$ can be equivalent expressed as
  \[
  q(\xvec) = a \Delta_d(\xvec, \xvec^*) + b,
  \]
  where $a = \sum_i \alpha_i$, $\xvec^* = \argmin_{\xvec} q(\xvec)$, and $b = q(\xvec^*)$.  Note $\xvec^*$ is the \emph{unconstrained} minimizer of $q(\xvec)$.
\end{property}

\begin{proof}
  By the optimality condition of $\xvec^*$ we have
  \begin{eqnarray}
  \label{eq:breg_comp_uncons}
  \inner{\svec + \sum_i \alpha_i (\grad d(\xvec^*) \! - \! \grad d(\xvec_i))}{\xvec - \xvec^*} = \zero \ \ \forall \xvec.
  \end{eqnarray}
  This equality must be changed to $\ge$ if $\xvec^*$ is the minimizer of $q(\xvec)$ over a constrained set $Q \varsubsetneq \dom d$.  By definition,

  \vspace{-1.5em}
  \[
  q(\xvec^*) = \inner{\svec}{\xvec^*} + \sum_i \alpha_i \Delta_d(\xvec^*, \xvec_i).
  \]
  Subtracting it from the definition of $q(\xvec)$ we get
  {
  \allowdisplaybreaks
  \begin{align*}
    q(\xvec) &= q(\xvec^*) + \inner{\svec}{\xvec - \xvec^*}  \\
    &\quad + \sum_i \alpha_i (d(\xvec) - d(\xvec^*) - \inner{\grad d(\xvec_i)}{\xvec - \xvec^*})\\
    &= q(\xvec^*) -  \inner{\sum_i \alpha_i (\grad d(\xvec^*) - \grad d(\xvec_i))}{\xvec - \xvec^*} \\
    &\quad + \sum_i \alpha_i (d(\xvec) - d(\xvec^*) - \inner{\grad d(\xvec_i)}{\xvec - \xvec^*}) \\
    &= q(\xvec^*) + \rbr{\sum_i \alpha_i} \Delta_d(\xvec, \xvec^*). \qedhere
  \end{align*}
  }
\end{proof}

\begin{assumption}
In the objective \eqref{eq:primal_obj_composite}, we will assume that $f$ is $\lambda_1$-sc and $\Psi$ is $\lambda_2$-sc wrt a given $d$ ($\lambda_1, \lambda_2 \ge 0$).  Then $f + \Psi$ is $\lambda$-sc, where
\[
\lambda := \lambda_1 + \lambda_2.
\]
\end{assumption}

\subsection{Assumption on the ground optimization problem}

We assume it is possible to efficiently solve the following ground problem:

\begin{assumption}
\label{assume:eff_solve_infty}
  Given an arbitrary linear function $\inner{\uvec}{\xvec}$, $\alpha_i \ge 0$ and $\xvec_i \in \dom \Psi$ $(i \in [k] := \cbr{1,\ldots,k})$, assume the following optimization problem can be solved efficiently:

\vspace{-2em}
\begin{align}
\label{eq:assume_eff_comp_infty}
\min_{\xvec} \inner{\uvec}{\xvec} + b + \sum_{i=1}^k \alpha_i \Delta_d(\xvec, \xvec_i) + \Psi(\xvec).
\end{align}
For different $k$, we call the assumption BD-$k$.
\end{assumption}

In \citep{Nesterov83} and \citep{Nesterov03a}, the 1-memory AGM-EF for general convex objective assumes BD-1.  In \citep{Nesterov05} and \citep{Nesterov07}, BD-$\infty$ is assumed in the sense that for arbitrary $k < \infty$, \eqref{eq:assume_eff_comp_infty} is assumed to be efficiently solvable.  In our later 1-memory AGM-EF, we will assume BD-2 if $\lambda_1 > 0$.  Although most literature assume BD-1, it is actually not hard to see that extension to BD-2 does not cause any real difficulty.  In fact, even BD-$\infty$ is feasible as long as $\sum_i \alpha_i \grad d(\xvec_i)$ can be aggregated efficiently (which is often true).

As a direct consequence of BD-1, now that the $f$ in \eqref{eq:primal_obj_composite} is $\lambda_1$-sc and $L$-\lcg, $J(\xvec)$ can be solved in one step if $L = \sigma \lambda_1$.  To see this, by definition for all $\xvec$
\begin{align*}
f(\xvec) &\le f(\xvec_0) + \inner{\grad f(\xvec_0)}{\xvec - \xvec_0} + \frac{L}{2} \nbr{\xvec - \xvec_0}^2,
\end{align*}

\vspace{-1em}
and

\vspace{-2em}
\begin{align*}
f(\xvec) &\ge f(\xvec_0) + \inner{\grad f(\xvec_0)}{\xvec - \xvec_0} + \lambda_1 \Delta_d(\xvec, \xvec_0) \\
&\ge f(\xvec_0) + \inner{\grad f(\xvec_0)}{\xvec - \xvec_0} + \frac{\lambda_1 \sigma}{2} \nbr{\xvec - \xvec_0}^2.
\end{align*}
So clearly $L \ge \sigma \lambda_1$.  If $L = \sigma \lambda_1$, then
\begin{align*}
f(\xvec) \equiv f(\xvec_0) + \inner{\grad f(\xvec_0)}{\xvec - \xvec_0} + \lambda_1 \Delta_d(\xvec, \xvec_0).
\end{align*}
Hence, $f(\xvec) + \Psi(\xvec)$ exactly satisfies the precondition of BD-1.  Therefore, in the sequel we will assume
\[
L > \sigma \lambda_1.
\]
$c := \frac{L}{\sigma \lambda_1}$ can be viewed as the condition number.

BD-2 allows us to inductively apply Property \ref{prpty:simplify_bregman} to simplify the expression of the following function
\begin{align*}
q_n(\xvec) &:= a_0 \Delta_d(\xvec, \xvec_0) + \sum_{i=1}^n \sbr{b_i +  \inner{\uvec_i}{\xvec} + a_i \Delta_d(\xvec, \xvec_i)} \\
\text{into} \ \ & \\
q_n(\xvec) &= \rbr{\sum_{i=0}^n a_i} \Delta_d(\xvec, \xvec_n^*) + q_n(\xvec_n^*),  \quad n \ge 1
\end{align*}
where $\xvec^*_n = \argmin_{\xvec} q_n(\xvec)$.  Let $q_0(\xvec) = a_0 \Delta_d(\xvec, \xvec_0)$. Then simplify $q_1(\xvec)$ into the sum of a constant and a Bregman divergence by Property \ref{prpty:simplify_bregman}:
\begin{align}
\label{eq:q1_inductive}
q_1(\xvec) &= (a_0 + a_1) \Delta_d(\xvec, \xvec^*_1) + q_1(\xvec^*_1), \\
\xvec^*_1 &= \argmin_{\xvec} q_0 (\xvec) + b_1 + \inner{\uvec_1}{\xvec} + a_1 \Delta_d(\xvec, \xvec_1), \nonumber
\end{align}
since $\xvec^*_1$ can be computed efficiently according to assumption BD-2.  Next, $q_2(\xvec)$ can be simplified by using \eqref{eq:q1_inductive} and Property \ref{prpty:simplify_bregman} again:
\begin{align*}
q_2(\xvec) &= (a_0 + a_1) \Delta_d(\xvec, \xvec^*_2) + q_2(\xvec^*_2), \\
\xvec^*_2 &= \argmin_{\xvec} q_1(\xvec) + b_2 + \inner{\uvec_2}{\xvec} + a_2 \Delta_d(\xvec, \xvec_2).
\end{align*}
This incremental scheme is especially useful when the $\argmin$ of all $q_k(\xvec)$ is readily available, \citep[\eg][Section 5]{AusTeb06}.

\textbf{Notations.}
Lower bold case letters (\eg, $\xvec$, $\val$) denote
vectors, $x_{i}$ denotes the $i$-th component of $\xvec$, $\zero$ refers
to the vector with all zero components, $\evec_i$ is the $i$-th
coordinate vector (all 0's except 1 at the $i$-th coordinate) and
$\Scal_n$ refers to the $n$ dimensional simplex $\cbr{\xvec \in [0, 1]^n : \sum_{i=1}^n x_i = 1}$. Unless specified otherwise, $\inner{\cdot}{\cdot}$ denotes the Euclidean dot product
$\inner{\xb}{\wb} = \sum_{i} x_{i} w_i$. We denote $\overline{\RR} := \RR \cup \{\infty\}$, and $[t]:= \{1, \ldots, t\}$.  From now on, we will always fix the $d$ in the context and omit the subscript $d$ in $\Delta_d$.

We follow the definition of norms in \citep{Nesterov05} which we recap here.  Suppose a finite dimensional real vector space $E$ (\eg\ $\RR^p$) is endowed with a norm $\nbr{\cdot}$.  The space of linear functions on $E$ is called the dual space which we denote as $E^*$.  The norm of $E^*$ is defined as
\[
\nbr{\svec}^* := \max_{\xvec \in E: \nbr{\xvec} = 1} \inner{\svec}{\xvec}.
\]
Suppose $A$ is a linear operator from $E_1$ to $E_2^*$, and $E_i$ has norm $\nbr{\cdot}_i$ for $i=1,2$.  Then the norm of $A$ is defined as
\begin{align}
\label{eq:def_matrix_norm}
\nbr{A} := \max_{\xvec \in E_1, \val \in E_2, \nbr{\xvec}_1 = \nbr{\val}_2 = 1} \inner{A \xvec}{\val}.
\end{align}
If we define an adjoint operator $A^*: E_2 \mapsto E_1^*$ as
\[
\inner{A^* \val}{\xvec} := \inner{A\xvec}{\val}, \quad \forall \xvec \in E_1, \val \in E_2.
\]
Then it can be shown that
\begin{align*}
\nbr{A^*} &= \max_{\xvec \in E_1, \val \in E_2, \nbr{\xvec}_1 = \nbr{\val}_2 = 1} \inner{A^* \val}{ \xvec} \\
&= \max_{\xvec \in E_1, \val \in E_2, \nbr{\xvec}_1 = \nbr{\val}_2 = 1} \inner{A \xvec}{\val} = \nbr{A}.
\end{align*}
The definition of matrix norm in \eqref{eq:def_matrix_norm} implies that
\begin{align*}
\nbr{A \xvec}^* &\le \nbr{A} \nbr{\xvec} \quad \forall \ \xvec \in E_1, \\
\nbr{A^* \val}^* &\le \nbr{A^*} \nbr{\val} \quad \forall \ \val \in E_2.
\end{align*}

To simplify notation we denote
\begin{align*}
\ell_{f} (\xvec; \yvec, \lambda_1) := f(\yvec) + \inner{\grad f(\yvec)}{\xvec - \yvec} + \lambda_1 \Delta(\xvec, \yvec).
\end{align*}
If $f$ is $\lambda_1$-sc, then $\ell_{f} (\xvec; \yvec, \lambda_1) \le f(\xvec)$ for all $\yvec$ and $\xvec$.

\section{$\infty$-memory AGM-EF}
\label{sec:inf_mem_nest}

The $\infty$-memory version of AGM-EF refers to the class of algorithms which use in each iteration all the past gradients $\grad f(\uvec_1), \ldots, \grad f(\uvec_k)$.  We present the method in Algorithm \ref{algo:inf_mem_nest}.

{
\begin{algorithm}[t]
\begin{algorithmic}[1]
\caption{\label{algo:inf_mem_nest} $\infty$-memory AGM-EF (AGM-EF-$\infty$).}
 \STATE Arbitrarily initialize $\xvec_{0} \in \dom \Psi$.  Set $\zvec_{0}  \leftarrow \xvec_{0}$.

 \STATE Set $A_{0} \leftarrow 0$.

 \STATE $\psi_{0}(\xvec) \leftarrow \Delta(\xvec, \xvec_{0})$.

 \FOR{$k = 0, 1, \ldots$}

   \STATE Denote as $a_{k+1}$ the positive root (in $a$) of $\phantom{aaaaaa}$ $(a + A_k) (\lambda_1 a + \lambda A_k + 1) + a \lambda_2 A_k = L \sigma^{-1} a^2$.

   \STATE $A_{k+1} \leftarrow A_k + a_{k+1}$, \\ $\tau_1 \leftarrow 1 + \lambda A_k$, $\tau_2 \leftarrow \lambda_1 a_{k+1}$, $\tau_3 \leftarrow \lambda_2 a_{k+1} \frac{A_k}{A_{k+1}}$, \\ $\tau \leftarrow \tau_1 + \tau_2 + \tau_3$.

   \STATE $\uvec_{k+1} \leftarrow \frac{a_{k+1} \tau_1 \zvec_k + (\tau A_k + \tau_3 a_{k+1}) \xvec_k}{\tau A_{k+1} - \tau_2 a_{k+1}}$.\newfootnotemark{1}

   \STATE $\psi_{k+1} (\xvec)  \leftarrow  \psi_{k} (\xvec) + a_{k+1} [\Psi(\xvec) + \ell_f(\xvec; \uvec_{k+1}, \lambda_1)]$.

   \STATE Find $\zvec_{k+1} \leftarrow \argmin_{\xvec} \psi_{k+1}(\xvec)$.

   \STATE $\xvec_{k+1} \leftarrow \rbr{A_k \xvec_k + a_{k+1} \zvec_{k+1}} / A_{k+1}$.

 \ENDFOR
\end{algorithmic}
\end{algorithm}
\newfootnotetext{0}{One can verify by simple algebra that $\uvec_{k+1}$ is a convex combination of $\zvec_k$ and $\xvec_k$.}
}

The main idea of the algorithm is to approximate $J(\xvec)$ by a sequence of functions $\psi_k$ that are constructed in Step 8 of Algorithm \ref{algo:inf_mem_nest}, and then ensure the following relationship at all iterations ($k \ge 0$):
\begin{align}
\label{eq:inf_mem_recur}
  A_k J(\xvec_k) \le \min_{\xvec} \psi_k(\xvec).
\end{align}
By construction, for all $k \ge 0$
\begin{align}
\label{eq:inf_mem_psi_k_exp}
  \psi_k(\xvec) = \Delta(\xvec, \xvec_0) + \! \sum_{i=1}^k a_i [\Psi(\xvec) + \ell_f(\xvec; \uvec_i, \lambda_1) ].
\end{align}
Summation from 1 to $0$ is assumed to be 0.  Now it is not hard to see that relationship \eqref{eq:inf_mem_recur} implies rates of convergence:
\begin{lemma}
\label{lem:inf_mem_simple_rate}
  If \eqref{eq:inf_mem_recur} holds for all $k \ge 1$, then for any $\xvec \in \dom \Psi$, we have
  \begin{align}
  \label{eq:inf_mem__rate_use_Ak}
    J(\xvec_k) - J(\xvec) \le A_k^{-1} \Delta(\xvec, \xvec_0).
  \end{align}
\end{lemma}
\begin{proof}
  By \eqref{eq:inf_mem_psi_k_exp}, we have for all $k \ge 1$
  \begin{align*}
    \psi_k(\xvec) &= \Delta(\xvec, \xvec_0) + \sum_{i=0}^k a_i [\Psi(\xvec) + \ell_f(\xvec; \uvec_i, \lambda_1) ] \\
    &\le \Delta(\xvec, \xvec_0) + \sum_{i=0}^k a_i [\Psi(\xvec) + f(\xvec)] \\
    &= \Delta(\xvec, \xvec_0) + A_k J(\xvec).
  \end{align*}
  Combining with \eqref{eq:inf_mem_recur}, we get \eqref{eq:inf_mem__rate_use_Ak}.
\end{proof}
Therefore, the rate of convergence totally depends on how fast $A_k$ grows.  We will show that Algorithm \ref{algo:inf_mem_nest} yields $A_k \sim k^2$ if $\lambda = 0$, or $A_k \sim e^k$ if $\lambda > 0$.  All updates are also kept efficient.  We next prove \eqref{eq:inf_mem_recur} and lower bound the growth rate of $A_k$.

\begin{lemma}[Eq \eqref{eq:inf_mem_recur}]
\label{lem:inf_mem_recur}
  The sequence $\cbr{\xvec_k}$ generated by Algorithm \ref{algo:inf_mem_nest} satisfy for all $k \ge 0$
  \[
  A_k J(\xvec_k) \le \min_{\xvec} \psi_k(\xvec).
  \]
\end{lemma}
\begin{proof}
  We prove by induction.  First check both sides are 0 for $k = 0$.  Now suppose \eqref{eq:inf_mem_recur} holds for some step $k \ge 0$.  By \eqref{eq:inf_mem_psi_k_exp} and Property \ref{prpty:sum_sc}, $\psi_k$ must be $(\lambda A_k + 1)$-sc wrt $d$.  So by Property \ref{prpty:sc_min_quad} and the fact that $\zvec_k$ minimizes $\psi_k$, we have
  \begin{align}
    \psi_k(\zvec_{k+1}) &\ge \psi_k(\zvec_k) + (\lambda A_k + 1) \Delta(\zvec_{k+1}, \zvec_k) \nonumber \\
    \label{eq:inf_mem_psi_k_gt}
    &\ge A_k J(\xvec_k)  + (\lambda A_k + 1) \Delta(\zvec_{k+1}, \zvec_k),
  \end{align}
  where the second inequality is by induction assumption.  So
  {
  \allowdisplaybreaks
  \begin{align*}
    &\min_{\xvec} \psi_{k+1}(\xvec) = \psi_{k+1}(\zvec_{k+1}) \\
    &= \psi_k(\zvec_{k+1}) + a_{k+1} \ell_f(\zvec_{k+1}; \uvec_{k+1}, \lambda_1) + a_{k+1} \Psi(\zvec_{k+1}) \\
    &\overset{(a)}{\ge} A_k f(\xvec_k) + A_k \Psi(\xvec_k) + (1 + \lambda A_k) \Delta(\zvec_{k+1}, \zvec_k) \\
    &\quad + a_{k+1} \ell_f(\zvec_{k+1}; \uvec_{k+1}, \lambda_1)  + a_{k+1} \Psi(\zvec_{k+1}) \\
    &\overset{(b)}{\ge} A_k \sbr{f(\uvec_{k+1}) + \inner{\grad f(\uvec_{k+1})}{\xvec_k - \uvec_{k+1}}} \\
    &\quad + (1 + \lambda A_k) \Delta(\zvec_{k+1}, \zvec_k) + A_{k} \Psi(\xvec_{k}) + a_{k+1} \Psi(\zvec_{k+1}) \\
    &\quad + a_{k+1} [f(\uvec_{k+1}) + \inner{\grad f(\uvec_{k+1})}{\zvec_{k+1} - \uvec_{k+1}} \\
    &\qquad \qquad + \lambda_1 \Delta(\zvec_{k+1}, \uvec_{k+1})] \\
    &= A_{k+1} f(\uvec_{k+1}) + A_{k} \Psi(\xvec_{k}) + a_{k+1} \Psi(\zvec_{k+1}) \\
    &\quad + \inner{\grad f(\uvec_{k+1})}{A_k \xvec_k - A_{k+1} \uvec_{k+1} + a_{k+1} \zvec_{k+1}} \\
    & \quad + \tau_1 \Delta(\zvec_{k+1}, \zvec_k) + \tau_2 \Delta(\zvec_{k+1}, \uvec_{k+1}) \\
    &\overset{(c)}{\ge} A_{k+1} f(\uvec_{k+1})  \\
    &\quad + A_{k+1} \Psi(\xvec_{k+1}) + \frac{\sigma}{2} \tau_3 \nbr{\zvec_{k+1} - \xvec_k}^2 \\
    & \quad + \inner{\grad f(\uvec_{k+1})}{A_k \xvec_k - A_{k+1} \uvec_{k+1} + a_{k+1} \zvec_{k+1}} \\
    & \quad + \frac{\sigma}{2} \tau_1 \nbr{\zvec_{k+1} - \zvec_k}^2 + \frac{\sigma}{2}\tau_2 \nbr{\zvec_{k+1} - \uvec_{k+1}}^2 \\
    &\overset{(d)}{\ge}  A_{k+1} \Psi(\xvec_{k+1}) + A_{k+1} \Big[ f(\uvec_{k+1}) \\
    &\quad + \frac{a_{k+1}}{A_{k+1}} \! \inner{\! \grad f(\uvec_{k+1})}{\zvec_{k+1} - \frac{A_{k+1} \uvec_{k+1} \! - \! A_k \xvec_k}{a_{k+1}}} \\
    &\quad + \frac{\sigma}{2} \frac{\tau_1 + \tau_2 + \tau_3}{A_{k+1}} \nbr{\zvec_{k+1} - \frac{\tau_1 \zvec_k + \tau_2 \uvec_{k+1} + \tau_3 \xvec_{k}}{\tau_1 + \tau_2 + \tau_3}}^2 \Big ]\\
    &\overset{(e)}{=} A_{k+1} \Psi(\xvec_{k+1}) + A_{k+1} \Big[ f(\uvec_{k+1}) \\
    & \quad + \frac{a_{k+1}}{A_{k+1}} \! \inner{\! \grad f(\uvec_{k+1})}{\zvec_{k+1} \! - \! \frac{A_{k+1} \uvec_{k+1} \! - \! A_k \xvec_k}{a_{k+1}}} \\
    &\quad + \frac{L}{2} \rbr{\frac{a_{k+1}}{A_{k+1}}}^2\nbr{\zvec_{k+1} - \frac{\tau_1 \zvec_k + \tau_2 \uvec_{k+1} + \tau_3 \xvec_k}{\tau_1 + \tau_2 + \tau_3}}^2 \Big ] \\
    &\overset{(f)}{=} A_{k+1} \Psi(\xvec_{k+1}) + A_{k+1} \Big [f(\uvec_{k+1}) \\
    &\quad + \inner{\grad f(\uvec_{k+1})}{\xvec_{k+1} - \uvec_{k+1}} + \frac{L}{2} \nbr{\xvec_{k+1} \! - \! \uvec_{k+1}}^2 \Big] \\
    &\overset{(g)}{\ge} A_{k+1} \Psi(\xvec_{k+1}) + A_{k+1} f(\xvec_{k+1}) = A_{k+1} J(\xvec_{k+1}).
  \end{align*}
  }

\vspace{-1.5em}
Here, step (a) is by \eqref{eq:inf_mem_psi_k_gt}.  (b) is by the convexity of $f$ (at $\uvec_{k+1}$).  (c) is by the $\lambda_2$-sc of $\Psi$ and Property \ref{prpty:two_point_sc}.  (d) is by the convexity and linearity of $\nbr{\cdot}$.  (e) is by the rule of choosing $a_{k+1}$ in Step 5 of Algorithm \ref{algo:inf_mem_nest}. (f) is by the definition of $\xvec_{k+1}$ and $\uvec_{k+1}$. (g) is by $L$-\lcg\ of $f$.
\end{proof}

Next, we can lower bound the growth rate of $A_k$.
\begin{lemma}
\label{lem:inf_mem_Ak}
Let $k \ge 1$.  Then
\begin{align*}
  A_k &\ge \max \cbr{ \frac{\sigma}{4L} \rbr{k + 1}^2,  \frac{\sigma}{L-\sigma \lambda_1} \rbr{1+\sqrt{\frac{\sigma \lambda}{4L}}}^{2k-2}}.
\end{align*}
\end{lemma}
\begin{proof}
  Since $A_{0} = 0$, so by solving Step 5 in Algorithm \ref{algo:inf_mem_nest}, we get $A_1 = \frac{\sigma}{L - \sigma \lambda_1}$.  Hence the lemma clearly holds for $k = 1$.  For all $k \ge 1$, denote
  \begin{align*}
  M &= (a_{k+1} + A_k) (\lambda_1 a_{k+1} + \lambda A_k + 1) + a_{k+1} \lambda_2 A_k \\
  &= A_{k+1} + \lambda A_k A_{k+1} + \lambda_1 a_{k+1} A_{k+1} + \lambda_2 a_{k+1} A_k.
  \end{align*}
  By the choice of $a_{k+1}$ in Step 5 of Algorithm \ref{algo:inf_mem_nest}, we get
  \begin{align}
    A_{k+1} &\le M = \frac{L}{\sigma} (A_{k+1} - A_k)^2 \nonumber \\
    &= \frac{L}{\sigma} \rbr{\sqrt{A_{k+1}} + \sqrt{A_k}}^2 \rbr{\sqrt{A_{k+1}} - \sqrt{A_k}}^2 \nonumber \\
    \label{eq:inf_mem_Ak_recur}
    &\le \frac{4L}{\sigma} A_{k+1} \rbr{\sqrt{A_{k+1}} - \sqrt{A_k}}^2.
  \end{align}
  So when $\lambda = 0$ we have
  \[
  A_k \ge \rbr{\frac{k-1}{2} \sqrt{\frac{\sigma}{L}} + \sqrt{A_1}}^2 = \frac{\sigma}{4L} (k+1)^2.
  \]
  When $\lambda > 0$, we have
  \begin{align*}
    \lambda A_k A_{k+1} \le M \le \frac{4L}{\sigma} A_{k+1} \rbr{\sqrt{A_{k+1}} - \sqrt{A_k}}^2
  \end{align*}
  where the last step is by \eqref{eq:inf_mem_Ak_recur}.  So
  \[
  \sqrt{A_{k+1}} \ge \rbr{1+\sqrt{\frac{\lambda \sigma}{4L}}} \sqrt{A_k},
  \]
  which directly implies the second term in $\max$.
\end{proof}

Combining Lemma \ref{lem:inf_mem_simple_rate}, \ref{lem:inf_mem_recur} and \ref{lem:inf_mem_Ak}, we derive
\begin{theorem}
\label{thm:inf_mem_rate_conv}
  For all $k \ge 1$ and $\xvec \in \dom \Psi$,
  \begin{align*}
  J(\xvec_k) - J(\xvec) &\le \Delta(\xvec, \xvec_0) \min \Bigg \{ \frac{4L}{\sigma (k+1)^2},  \\
  &\qquad \quad \frac{L-\sigma \lambda_1}{\sigma} \rbr{1+\sqrt{\frac{\sigma \lambda}{4L}}}^{-2k+2} \Bigg \}.
  \end{align*}
\end{theorem}
Therefore, as long as one of $\lambda_1$ and $\lambda_2$ is strictly positive such that $\lambda = \lambda_1 + \lambda_2 > 0$, $J(\xvec_k)$ converges linearly.  When $\lambda_1 = 0$ and $\lambda_2 > 0$, $\psi_k$ contains only one Bregman divergence making it easier to optimize.

\textbf{Remark 1.}
   If \eqref{eq:inf_mem_psi_k_gt} is replaced by
  \begin{align*}
    \psi_k(\zvec_{k+1}) &\ge \psi_k(\zvec_k) + (\lambda A_k + 1) \Delta(\zvec_{k+1}, \zvec_k) \\
    &\ge A_k J(\xvec_k)  + (\lambda A_k + 1) \frac{\sigma}{2} \nbr{\zvec_{k+1} -  \zvec_k}^2,
  \end{align*}

  \vspace{-0.85em}
  then it is not hard to see that the proof of Lemma \ref{lem:inf_mem_recur} still goes through.  So $\Psi$ does not need to be $\lambda_2$-sc wrt $d$, and it suffices to be $\lambda_2 \sigma$ strongly convex wrt $\nbr{\cdot}$.  In practice, checking and satisfying the latter condition can be much easier.  Similar remark can be made later for AGM-EF-1, and for the ease of exposition we will still assume $\Psi$ is $\lambda_2$-sc wrt $d$.

\subsection{Notes on the Computations}

The whole algorithm relies on solving $\zvec_{k}$ efficiently, and it can be dealt with in two ways.  First, by \eqref{eq:inf_mem_psi_k_exp}, minimizing $\psi_k(\xvec)$ only requires solving the following form of problem:
\vspace{-1em}
\begin{align*}
\min_{\xvec} \ & A_k \Psi(\xvec) + \Delta(\xvec, \uvec_0) + \lambda_1 \sum_{i=0}^n a_i \Delta(\xvec, \uvec_i) \\
&+ \inner{\sum_{i=0}^n a_i \grad f(\uvec_i)}{\xvec}
\end{align*}

\vspace{-0.5em}
This is feasible by Assumption \ref{assume:eff_solve_infty}, and in practice the gradients of $f$ and $d$ can be aggregated on the fly.

The second method requires making one more assumption, in addition to the usual assumption $\dom \Psi \subseteq \dom d$.
\begin{assumption}
\label{assump:dom_d_in Q}
  $\dom d \subseteq \dom \Psi$.
\end{assumption}

This assumption is often met when $d$ is the entropy and $\dom \Psi$ is the simplex.  It ensures that $\zvec_k := \min_{\xvec \in \dom \Psi} \psi_k(\xvec)$ is also a solution of the unconstrained optimization $\min_{\xvec \in \dom d} \psi_k(\xvec)$.  Then when $\Psi$ is affine on its domain, we can apply Property \ref{prpty:simplify_bregman} and the subsequent discussion on inductively updating $\psi_k(\xvec)$.  This scheme is particularly useful in Algorithm \ref{algo:inf_mem_nest} because the minimizer $\zvec_{k}$ is already available.

Even if Assumption \ref{assump:dom_d_in Q} does not hold and $\zvec_{k}$ is not an unconstrained minimizer of $\psi_k(\xvec)$, one can still spend extra computations to find the unconstrained minimizer and inductively update $\psi_k(\xvec)$.  This idea will be useful if the gradient aggregation in the first method is not viable.

\subsection{Adaptively tuning the Lipschitz constant}

The Algorithm \ref{algo:inf_mem_nest} requires the explicit value of $L$.  This is usually not available, or the global maximum curvature is much larger than the local directional curvature.  As a result, the steps size $1/L$ becomes too conservative.  From the proof of Lemma \ref{lem:inf_mem_recur}, it is clear that $L$ is used only to ensure \eqref{eq:inf_mem_recur}.  So we can probe smaller values of $L$.  The modified algorithm is given in Algorithm \ref{algo:inf_mem_nest_autoL}.

\begin{algorithm}[t]
\begin{algorithmic}[1]
\caption{\label{algo:inf_mem_nest_autoL}AGM-EF-$\infty$ with adaptive $L$.}
 \REQUIRE{Down scaling factor $\gamma_d$ and up scaling factor $\gamma_u$ ($\gamma_d, \gamma_u > 1$).  An optimistic estimate $\Ltil \le L$.}

 \STATE Arbitrarily initialize $\xvec_{0} \in \dom \Psi$.  Set $\zvec_{0} \leftarrow \xvec_{0}$.

 \STATE Set $A_{0} \leftarrow 0$.

 \STATE $\psi_{0}(\xvec) \leftarrow \Delta(\xvec, \xvec_{0})$.

 \STATE $L_{0} \leftarrow \Ltil * \gamma_d * \gamma_u$.

 \FOR{$k = 0, 1, \ldots$}

   \STATE  $L_{k+1} \leftarrow L_{k} / (\gamma_d * \gamma_u)$.

   \REPEAT

       \STATE $L_{k+1} \leftarrow L_{k+1} * \gamma_u$.

       \STATE Assign to $a_{k+1}$ the positive root (in $a$) of $\phantom{aaaa}$ $(a + A_k) (\lambda_1 a + \lambda A_k + 1) + a \lambda_2 A_k = L_{k+1} \sigma^{-1} a^2$.

       \STATE Do step 6 to 10 of Algorithm \ref{algo:inf_mem_nest}.

    \UNTIL{$A_{k+1} J(\xvec_{k+1}) \le \psi_{k+1}(\zvec_{k+1})$.}

 \ENDFOR
\end{algorithmic}
\end{algorithm}

The inner ``repeat" loop must terminate in a finite number of steps because $L_k$ grows exponentially and once $L_k \ge L$ the ``until" condition must be satisfied.  And the number of steps in this inner loop is logarithmic in $L$, with the final $L_{k} < \gamma_u L$.  Moreover, this $L_k$ is decayed by a factor of $\gamma_d$ before being used to initialize $L_{k+1}$.  This is in sharp contrast to AGM-PR where the estimates of $L$ must grow monotonically through iterations.  Let us formally characterize how adaptively tuning $L$ leads to faster convergence rates through faster growth rate of $A_k$.
\begin{lemma}
\label{lem:inf_mem_Ak_adapt}
For all $k \ge 1$,
\begin{align*}
  A_k &\ge \max \Bigg \{ \frac{\sigma}{L_1 - \sigma \lambda_1} \prod\limits_{i=2}^{k} \rbr{1 + \sqrt{\frac{\sigma \lambda}{L_i}}}^{2}, \\
  &\qquad \qquad \qquad \frac{\sigma}{4} \rbr{\sqrt{\frac{4}{L_1}} + \sum_{i=2}^{k} \sqrt{\frac{1}{L_i}} }^2 \Bigg \}.
\end{align*}
\end{lemma}
\begin{proof}
  Simply replace the $L$ in \eqref{eq:inf_mem_Ak_recur} by $L_{i+1}$.
\end{proof}

In practice, we observed that the $L_k$ is often only 10 per cent of the real $L$ and therefore by Lemma \ref{lem:inf_mem_Ak_adapt} the convergence rate is 10 times faster than using $L$.  Moreover, the $L_k$ in successive iterations are quite close so the inner loop terminates in only 2-3 steps.

This adaptive scheme relies on the fact that the key relationship \eqref{eq:inf_mem_recur} is independent of $L$ and involves function values only at two points (rather than globally).  In contrast, the algorithm and analysis in \citep{LanLuMon09} keep a global relationship which explicitly involves $L$, making it hard to accommodate adaptive $L$.

We also tried to adaptively tune $\lambda$, but not successful.  This is turns out to be very hard because the proof uses $\lambda$ as a a global property (recall the fact that $\psi_k$ must be $(\lambda A_k+1)$-sc wrt $d$), while $L$ is used only at $\uvec_{k+1}$ and $\xvec_{k+1}$ in Step (g) of the proof of Lemma \ref{lem:inf_mem_recur}.

\subsection{Bounding the Duality Gap}
\label{sec:inf_mem_duality_gap}

Algorithm \ref{algo:inf_mem_nest} does not have a termination criterion, and a natural criterion will be based on the duality gap.  Furthermore, in some applications like \eqref{eq:reg-risk} the primal problem is nonsmooth and AGM-EF-$\infty$ is applied only to its dual problem which is \lcg.  So it is necessary to convert the dual iterates at each step into the primal, and characterize the convergence rate in the primal.  In this subsection, we extend the technique in \citep[][Section 2]{Lu09} to the case of composite objective.  Except the strong convexity, our whole setting and procedure bear much resemblance to \citep{Tseng08}, \citep[][Theorem 2.2]{Lu09}, \citep[][Theorem 3]{Nesterov05} and \citep[][Section 6]{Nesterov07}.  We are unaware of any existing result which shows \emph{linear convergence of the duality gap} as we will describe below.

Consider a minimax problem
\[
\min_{\xvec} \max_{\val \in Q_2} \phi(\xvec, \val) + \Psi(\xvec).
\]
Here $\Psi: \RR^p \mapsto \RRbar$ is proper, lower semicontinuous and $\lambda_2$-sc wrt $d$ ($\lambda_2 \ge 0$).  Let $\Psi$ satisfy Assumption \ref{assume:eff_solve_infty}. $Q_2$ is a compact convex set in the Euclidean space. $\phi: \RR^p \times Q_2 \mapsto \RRbar$ is continuous on $\dom \Psi \times Q_2$. For all fixed $\val \in Q_2$, $\phi(\cdot, \val)$ is $\lambda_1$-sc wrt $d$ ($\lambda_1 \ge 0$) and is differentiable on a open set containing $\dom \Psi$.  For all fixed $\xvec \in \dom \Psi$, $\phi(\xvec, \cdot)$ is strictly concave.  Therefore, the $\argmax_{\val \in Q_2} \phi(\xvec, \val)$ is unique and we denote it as $\val(\xvec)$.

Let us define
\begin{align}
\label{eq:dual_f_def_inf_mem}
f(\xvec) := \max_{\val \in Q_2} \phi(\xvec, \val).
\end{align}
Then by Denskin's theorem \citep[][Theorem B.25]{Bertsekas95}, $f$ must be convex and differentiable on $\dom \Psi$.  We further assume that $f$ is $L$-\lcg\ on $\dom \Psi$.  A key strong convexity property of $f$ is:
\begin{lemma}
  Given all the above assumptions on $\phi$, $f(\xvec)$ must be $\lambda_1$-sc.  However, the converse is not necessarily true, \ie\ $f(\xvec)$ being $\lambda_1$-sc does not entail that $\phi(\cdot, \val)$ is $\lambda_1$-sc for all fixed $\val \in Q_2$.
\end{lemma}
\begin{proof}
  For any $\xvec_1, \xvec_2 \in \dom \Psi$, we have
  \begin{align*}
    f(\xvec_2) &= \max_{\val \in Q_2} \phi(\xvec_2, \val) \ge \phi(\xvec_2, \val(\xvec_1)) \\
    &\ge \phi(\xvec_1, \val(\xvec_1)) + \inner{\grad_{\xvec}\phi(\xvec_1, \val(\xvec_1))}{\xvec_2 - \xvec_1} \\
    &\qquad \qquad + \lambda_1 \Delta(\xvec_2, \xvec_1) \\
    &= f(\xvec_1) + \inner{\grad f(\xvec_1)}{\xvec_2 - \xvec_1} + \lambda_1 \Delta(\xvec_2, \xvec_1),
  \end{align*}
  where the last step is by Denskin's theorem.
\end{proof}

We also define a dual objective
\begin{align}
J(\xvec) &:= \Psi(\xvec) + \max_{\val \in Q_2} \phi(\xvec, \val) \nonumber \\
\label{eq:dual_g_def_inf_mem}
D(\val) &:= \min_{\xvec} \cbr{\phi(\xvec, \val) + \Psi(\xvec)}  \quad \text{for } \val \in Q_2
\end{align}
where the $\argmin$ in \eqref{eq:dual_g_def_inf_mem} may be not unique and $D(\val)$ may be nonsmooth.  Our assumptions above ensure that for any $\val \in Q_2$ and any $\xvec$, the following is true:
\[
D(\val) \le J(\xvec), \quad \text{and} \quad \max_{\val \in Q_2} D(\val) = \min_{\xvec} J(\xvec).
\]

When applied to minimize $J(\xvec)$, AGM-EF-$\infty$ (with or without adaptive $L$) produces a sequence of $\cbr{\xvec_k, \uvec_k, \zvec_k}$.  It is our goal to design a sequence of dual variables $\cbr{\val_k}$ based on $\cbr{\xvec_i, \uvec_i, \zvec_i : i \le k}$ such that the duality gap
\[
\delta_k := J(\xvec_k) - D(\val_k)
\]
goes to 0 fast.  Since
\[
\delta_k \ge J(\xvec_k) - \max_{\val \in Q_2} D(\val) = J(\xvec_k) - \min_{\xvec} J(\xvec),
\]
so once $\delta_k$ falls below a prescribed tolerance $\epsilon$, $\xvec_k$ is guaranteed to be an $\epsilon$ accurate solution of $J$.  Indeed we will show that the following construction of $\val_k$ meets our need:
\begin{align}
\label{eq:inf_mem_x_to_val}
\val_k = \frac{1}{A_k} \sum_{i=1}^k a_i \val(\uvec_i).
\end{align}
where $a_i$ and $A_k$ are also from AGM-EF-$\infty$.  \eqref{eq:inf_mem_x_to_val} can be equivalently reformulated into a recursion which allows efficient update of $\val_k$:
\[
\val_1 = \val(\uvec_1), \text{ and } \val_{k+1} = \frac{A_k}{A_{k+1}} \val_k + \frac{a_{k+1}}{A_{k+1}}\val(\uvec_{k+1}).
\]

\begin{theorem}
\label{thm:inf_mem_duality_gap}
  Suppose a sequence $\cbr{\xvec_k, \uvec_k, \zvec_k}$ is produced when AGM-EF-$\infty$ is applied to minimize $J(\xvec)$ by treating $f$ as $\lambda_1$-sc.  Then the $\cbr{\val_k}$ defined by \eqref{eq:inf_mem_x_to_val} satisfies $\val_k \in Q_2$ and
  \begin{align}
  \label{eq:inf_mem_duality_bound}
  \delta_k = J(\xvec_k) - D(\val_k) \le \frac{1}{A_k} \max_{\xvec \in \dom \Psi} \Delta(\xvec, \uvec_0).
  \end{align}
\end{theorem}
\begin{proof}
  Since $\uvec_i \in \dom \Psi$, so $\val(\uvec_i) \in Q_2$.  And $\val_k$ is a convex combination of $\val(\uvec_i)$ ($i \le k$), so $\val_k \in Q_2$.  Using the fact that $\phi(\xvec, \val)$ is $\lambda_1$-sc in $\alphavec$ for all fixed $\xvec$, we have
  \begin{align}
    &\ell_f(\xvec; \uvec_i, \lambda_1) \nonumber \\
    &= f(\uvec_i) + \inner{\grad f(\uvec_i)}{\xvec - \uvec_i} + \lambda_1 \Delta(\xvec, \uvec_i) \nonumber \\
    &= \phi(\uvec_i, \val(\uvec_i)) \! + \! \inner{\grad_{\xvec} \phi(\uvec_i, \val(\uvec_i))}{\xvec - \uvec_i} \! + \! \lambda_1 \Delta(\xvec, \uvec_i) \nonumber \\
    \label{eq:inf_mem_duality_keystep}
    &\le \phi(\xvec, \val(\uvec_i)).
  \end{align}
  Now by using relationship \eqref{eq:inf_mem_recur} and \eqref{eq:inf_mem_psi_k_exp}, we have
  {
  \allowdisplaybreaks
  \begin{align*}
    &A_k J(\xvec_k) \\
    &\le \min_{\xvec} \cbr{\Delta(\xvec, \uvec_0) + A_k \Psi(\xvec) + \sum_{i=1}^k a_i \ell_f(\xvec; \uvec_i, \lambda_1)} \\
    &\le \min_{\xvec} \Delta(\xvec, \uvec_0) + A_k \Psi(\xvec) + \sum_{i=1}^k a_i \phi(\xvec, \val(\uvec_i)) \\
    &\le \min_{\xvec} \Delta(\xvec, \uvec_0) + A_k \Psi(\alphavec) + A_k \phi \rbr{\xvec, \frac{1}{A_k} \sum_{i=1}^k a_i \val(\uvec_i)} \\
    &\le \max_{\xvec \in \dom \Psi} \Delta(\xvec, \uvec_0) + A_k \min_{\xvec} \cbr{\Psi(\xvec) + \phi(\xvec, \val_k)} \\
    &= \max_{\xvec \in \dom \Psi} \Delta(\xvec, \uvec_0) + A_k D(\val_k). \qedhere
  \end{align*}
  }
\end{proof}
So $\delta_k$ converges linearly as long as $\lambda_1 + \lambda_2 > 0$.  If $\dom \Psi$ is unbounded and $\max_{\xvec \in \dom \Psi} \Delta(\xvec, \uvec_0) = \infty$, then the bound in \eqref{eq:inf_mem_duality_bound} becomes vacuous.

We emphasize that in Theorem \ref{thm:inf_mem_duality_gap}, AGM-EF-$\infty$ is invoked by treating $f$ as $\lambda_1$-sc, although the real strong convexity constant $\lambda'_1$ of $f$ may be greater than $\lambda_1$.  In this case, the duality gap will decay at a slower rate than that for the gap of $J$ (by using $\lambda'_1$ in AGM-EF-$\infty$).  However the strong convexity of $\Psi$ is still fully utilized in the duality gap, and in many machine learning problems the strong convexity does come from $\Psi$ rather than $f$ (\ie\ $\lambda_1 = \lambda'_1 = 0$).

\section{1-memory AGM-EF}
\label{sec:one_mem_nest}

Note that AGM-EF-$\infty$ keeps a nonparametric form \eqref{eq:inf_mem_psi_k_exp} of the model $\psi_k(\xvec)$ whose complexity grows with iteration.  In 1-memory AGM-EF, the model is compressed to a simple parametric form in each iteration.  \citet{AusTeb04} gave a Bregman version for unconstrained optimization.  \citep{Nesterov83} provided an algorithm for constrained problems with Euclidean distance as the prox-function.  However, only \citep{LanLuMon09} and \citep{Tseng08} accommodate both Bregman divergence and constraints.  But their algorithms do not extend to strongly convex objectives and restrict the estimate of $L$ to be nondecreasing through iterations.  Therefore, we propose in this section a 1-memory AGM-EF which uses Bregman prox-function, and allows constraints and non-monotonic adaptive tuning of $L$.

\begin{algorithm}[t]
\begin{algorithmic}[1]
\caption{\label{algo:one_mem_nest}1-memory AGM-EF (AGM-EF-$1$).}
 \STATE Arbitrarily pick $\uvec_0 \in \dom \Psi$.

 \STATE Initialize $c_0 \leftarrow \frac{L}{\sigma} + \lambda_2$.

 \STATE $q_{0}(\xvec) \leftarrow \frac{L}{\sigma} \Delta(\xvec, \uvec_0) + \Psi(\xvec) + \ell_f(\xvec; \uvec_0, 0)$.

 \STATE $\xvec_0  = \zvec_0 \leftarrow \argmin_{\xvec} q_0(\xvec)$.

 \FOR{$k = 0, 1, \ldots$}

   \STATE Assign to $a_{k+1}$ the positive root (in $a$) of $\phantom{aaaaaa}$ $\sigma (1-a)(c_k + \lambda_2 a) + \sigma \lambda_1 a = L a^2$.

   \STATE $c_{k+1} \leftarrow (1-a_{k+1}) c_k  + (\lambda_1 + \lambda_2) a_{k+1}$. \\ $\tau_1 \leftarrow (1-a_{k+1})c_k$, $\tau_2 \leftarrow \lambda_1 a_{k+1}$, \\
   $\tau_3 \leftarrow \lambda_2 a_{k+1} (1-a_{k+1})$,
   $\tau \leftarrow \tau_1 + \tau_2 + \tau_3$.

   \STATE $\uvec_{k+1} \leftarrow \frac{(\tau - (\tau_1 + \tau_2) a_{k+1}) \xvec_k + \tau_1 a_{k+1} \zvec_k}{\tau - \tau_2 a_{k+1}}$.\newfootnotemark{1}

   \STATE $\psi_{k+1}(\xvec) \leftarrow (1  -  a_{k+1}) q_k(\xvec)   $ \\
    $ \qquad \qquad \qquad \quad + a_{k+1}[\ell_f(\xvec; \uvec_{k+1}, \lambda_1) +  \Psi(\xvec)]$.

   \STATE $\zvec_{k+1} \leftarrow \argmin_{\xvec} \psi_{k+1}(\xvec)$.

   \STATE $\xvec_{k+1} \leftarrow (1-a_{k+1}) \xvec_k + a_{k+1} \zvec_{k+1}$.

   \STATE $q_{k+1}(\xvec) \leftarrow c_{k+1} \Delta(\xvec, \zvec_{k+1}) + \psi_{k+1}(\zvec_{k+1})$.

 \ENDFOR
\end{algorithmic}
\end{algorithm}
\newfootnotetext{0}{$\uvec_{k+1}$ is clearly a convex combination of $\zvec_k$ and $\xvec_k$.}

Arbitrarily pick $\uvec_0 \in \dom \Psi$ and initialize by
{
\allowdisplaybreaks
\begin{align*}
q_0 (\xvec) &:= \frac{L}{\sigma} \Delta(\xvec, \uvec_0) + f(\xvec_0) + \inner{\grad f(\uvec_0)}{\xvec - \uvec_0} + \Psi(\xvec) \\
\xvec_0 &= \zvec_0 = \argmin_{\xvec} q_0(\xvec) \\
c_0 &= \frac{L}{\sigma} + \lambda_2.
\end{align*}
}
Then for all $k \ge 0$, define:
\begin{align*}
  \psi_{k+1}(\xvec) &= (1 \! - \! a_{k+1}) q_k (\xvec) \! + \! a_{k+1} [\ell_{f}(\xvec; \uvec_{k+1}, \lambda_1) \! + \! \Psi(\xvec)] \\
  \zvec_{k+1} &= \argmin_{\xvec} \psi_{k+1}(\xvec) \\
  c_{k+1} &= (1-a_{k+1})c_k + \lambda a_{k+1} \\
  q_{k+1} (\xvec) &= c_{k+1} \Delta(\xvec, \zvec_{k+1}) + \psi_{k+1}(\zvec_{k+1}).
\end{align*}

By construction for all $k \ge 0$, $q_k$ is $c_k$-sc and $\psi_{k+1}$ is strongly convex with constant $(1-a_{k+1})c_k + \lambda a_{k+1}$, \ie\ $c_{k+1}$-sc.  Clearly, for all $k \ge 1$
\begin{align}
\label{eq:one_mem_nadir_match}
\min_{\xvec} \psi_{k}(\xvec) = \psi_{k}(\zvec_{k}) = q_k(\zvec_k) = \min_{\xvec} q_{k} (\xvec).
\end{align}
But except at $\xvec = \zvec_k$, $q_k(\xvec) \neq \psi_k(\xvec)$ in general.  The only case where $q_k(\xvec) \equiv \psi_k(\xvec)$ is when $\Psi(\xvec)$ is an affine function on $\dom d$ and $\dom d \subseteq \dom \Psi$.  Then an inductive application of Property \ref{prpty:simplify_bregman} reveals $q_k(\xvec) \equiv \psi_k(\xvec)$.  Lemma 5.2 of \citep{AusTeb06} is exactly this case with $\Psi(\xvec) \equiv 0$.  However, when $\dom d \nsubseteq \dom \Psi$ then $\zvec_{k+1}$ actually solves a constrained optimization, and then \eqref{eq:breg_comp_uncons} must be changed to $\ge$ which breaks Property \ref{eq:q1_inductive}.

The proof of rate of convergence for Algorithm \ref{algo:one_mem_nest} relies on the following two relations: for all $k \ge 0$ and $\xvec \in \dom \Psi$,
\begin{align}
\label{eq:one_mem_qx}
  q_{k+1}(\xvec) - J(\xvec) &\le (1 - a_{k+1}) (q_k(\xvec) - J(\xvec)) \\
\label{eq:one_mem_fleq}
  J(\xvec_k) &\le q_{k}(\zvec_{k}).
\end{align}

From these three inequalities, we get for all $\xvec \in \dom \Psi$,
\begin{align}
  J(\xvec_k) &\overset{\eqref{eq:one_mem_fleq}}{\le} q_k(\zvec_k) \overset{\eqref{eq:one_mem_nadir_match}}{\le} q_{k}(\xvec) \nonumber \\
  \label{eq:one_mem_gap_bound}
  &\overset{\eqref{eq:one_mem_qx}}{\le} \! J(\xvec) + \rbr{q_0(\xvec) \! - \! J(\xvec)} \cdot \prod_{i=1}^k (1 - a_i).
\end{align}
So the gap $J(\xvec_k) - J(\xvec)$ decays at the same rate as $\prod_{i=1}^k (1-a_i)$.\footnote{The last inequality of \eqref{eq:one_mem_gap_bound} does not require $q_0(\xvec) \ge J(\xvec)$.  But $q_0(\xvec) \ge J(\xvec)$ can be easily proved by Lemma \ref{lemma:one_mem_qx_nonneg}.} Compared with the $\infty$-memory AGM-EF, the additional inequality \eqref{eq:one_mem_qx} is now needed because the models $q_k$ here are approximations of the $\psi_k$ in \eqref{eq:inf_mem_psi_k_exp}.  Next, we prove the three relations one by one.

\begin{lemma}[Eq \eqref{eq:one_mem_qx}]
  \label{lemma:one_mem_qx}
  For all $k \ge 0$ and $\xvec$, we have
  \[
  q_{k+1}(\xvec) - J(\xvec) \le (1 - a_{k+1}) (q_k(\xvec) - J(\xvec)).
  \]
\end{lemma}
\begin{proof}
Since $\zvec_{k+1}$ minimizes $\psi_{k+1}(\xvec)$ and $\psi_{k+1}(\xvec)$ is $c_{k+1}$-sc,  so by Property \ref{prpty:sc_min_key_one} we have
\begin{align}
\label{eq:one_mem_qx_key}
  \psi_{k+1}(\xvec) \ge \psi_{k+1}(\zvec_{k+1}) + c_{k+1} \Delta(\xvec, \zvec_{k+1}).
\end{align}
So for all $\xvec \in Q$,
  \begin{align}
  &(1 - a_{k+1}) q_k(\xvec) + a_{k+1} J(\xvec) \nonumber \\
  \label{eq:one_mem_prop_line1}
    &\ge (1 \! - \! a_{k+1})q_k(\xvec) \! + \! a_{k+1} [\ell_{f}(\xvec; \! \uvec_{k+1}, \lambda_1) \! + \! \Psi(\xvec)]\\
  &= \psi_{k+1}(\xvec) \nonumber \\
  &\ge \psi_{k+1}(\zvec_{k+1}) + c_{k+1} \Delta(\xvec, \zvec_{k+1}) \quad \text{ (by \eqref{eq:one_mem_qx_key})} \nonumber \\
  \label{eq:one_mem_prop_line2}
  &= q_{k+1}(\xvec).
  \end{align}
\end{proof}

\begin{lemma}[Eq \eqref{eq:one_mem_fleq}]
  \label{lemma:one_mem_fleq}
  For all $k \ge 0$, $J(\xvec_k) \le q_{k}(\zvec_{k})$.
\end{lemma}
\begin{proof}
  We prove by induction.  First, when $k=0$ $q_0(\zvec_0) = J(\xvec_0)$.  Now suppose \eqref{eq:one_mem_fleq} holds for certain $k \ge 0$.  Then
  {
  \allowdisplaybreaks
  \begin{align*}
    &q_{k+1}(\zvec_{k+1}) = \psi_{k+1}(\zvec_{k+1}) \\
    &= (1 \! - \! a_{k+1})q_k(\zvec_{k+1}) \! + \! a_{k+1} [\ell_f(\zvec_{k+1}; \uvec_{k+1}, \lambda_1) \! + \! \Psi(\zvec_{k+1})] \\
    &\overset{(a)}{\ge} (1-a_{k+1}) [q_k(\zvec_k) + c_k \Delta(\zvec_{k+1}, \zvec_k)] \\
    &\qquad \qquad + a_{k+1} [\ell_f(\zvec_{k+1}; \uvec_{k+1}, \lambda_1) + \Psi(\zvec_{k+1})] \\
    &\overset{(b)}{\ge} (1-a_{k+1}) [f(\xvec_k) + \Psi(\xvec_k) + c_k \Delta(\zvec_{k+1}, \zvec_k)] \\
    &\qquad \qquad + a_{k+1} \ell_f(\zvec_{k+1}; \uvec_{k+1}, \lambda_1) + a_{k+1} \Psi(\zvec_{k+1}) \\
    &\overset{(c)}{\ge} (1-a_{k+1}) [f(\uvec_{k+1}) + \inner{\grad f(\uvec_{k+1})}{\xvec_k - \uvec_{k+1}} \\
    &\quad + \frac{c_k \sigma}{2} \nbr{\zvec_{k+1} - \zvec_{k}}^2] + a_{k+1} [f(\uvec_{k+1})\\
    &\quad +\inner{\grad f(\uvec_{k+1})}{\zvec_{k+1} - \uvec_{k+1}} + \frac{\lambda_1 \sigma}{2} \nbr{\zvec_{k+1} - \uvec_{k+1}}^2] \\
    &\quad + (1-a_{k+1}) \Psi(\xvec_k) + a_{k+1} \Psi(\zvec_{k+1}) \\
    &\overset{(d)}{\ge} \Psi(\xvec_{k+1}) + f(\uvec_{k+1}) \\
    &\quad + \inner{\grad f(\uvec_{k+1})}{(1-a_{k+1}) \xvec_k + a_{k+1} \zvec_{k+1} - \uvec_{k+1}} \\
    &\quad + \frac{\sigma}{2} c_k (1 - a_{k+1}) \nbr{\zvec_{k+1} - \zvec_{k}}^2  \\
    &\quad + \frac{\sigma}{2} \lambda_1 a_{k+1} \nbr{\zvec_{k+1} - \uvec_{k+1}}^2 \\
    &\quad + \frac{\sigma}{2} \lambda_2 a_{k+1} (1-a_{k+1}) \nbr{\zvec_{k+1} - \xvec_k}^2 \\
    &\overset{(e)}{\ge} \Psi(\xvec_{k+1}) + f(\uvec_{k+1}) \\
    & \quad +\inner{\grad f(\uvec_{k+1})}{(1-a_{k+1}) \xvec_k + a_{k+1} \zvec_{k+1} - \uvec_{k+1}} \\
    & \quad + \frac{\sigma}{2} (\tau_1 + \tau_2 + \tau_3) \nbr{\zvec_{k+1} - \frac{\tau_1 \zvec_k + \tau_2 \uvec_{k+1} + \tau_3 \xvec_k}{\tau_1 + \tau_2 + \tau_3}}^2 \\
    &\overset{(f)}{=} \Psi(\xvec_{k+1}) + f(\uvec_{k+1}) \\
    &\quad + \inner{\grad f(\uvec_{k+1})}{\xvec_{k+1} - \uvec_{k+1}} + \frac{L}{2} \nbr{\xvec_{k+1} - \uvec_{k+1}}^2 \\
    &\overset{(g)}{\ge} \Psi(\xvec_{k+1}) + f(\xvec_{k+1}) = J(\xvec_{k+1}),
  \end{align*}
  }

  where (a) is because $\zvec_k$ minimizes $q_k$ and $q_k$ is $c_k$-sc.  (b) is by the induction assumption.  (c) is by the convexity of $f$ and $\sigma$-sc of $d$.  (d) is by the $\lambda_2$-sc of $\Psi$ and Property \ref{prpty:two_point_sc}.  (e) is by the convexity of norm.  (f) is by the definition of $\xvec_{k+1}$ and $\uvec_{k+1}$, and the choice of $a_{k+1}$.  (g) is by the $L$-\lcg\ of $f$.
\end{proof}

Noting that $c_0 \ge \lambda$ by definition, we can bound $\prod_{i=1}^k (1-a_i)$ by invoking Lemma 2.2.4 of \citep{Nesterov03a} with the strong convexity constant being $\lambda$ and the Lipschitz constant of the gradient being
\[
L' := \frac{L}{\sigma} + \lambda_2.
\]
It is easy to verify that the condition number $L' / \lambda$ is monotonically decreasing in $\lambda_2$.
\begin{lemma}
  \label{lemma:one_mem_ck}
  For all $k \ge 1$, we have
  \[
  \prod_{i=1}^k (1-a_i) \le \min \cbr{\rbr{1-\sqrt{\frac{\lambda}{L'}}}^k, \frac{4L'}{(2\sqrt{L'} + k\sqrt{c_0})^2}}.
  \]
\end{lemma}
Finally we bound $q_0 (\xvec) - J(\xvec)$ by
\begin{align}
q_0 &(\xvec) - J(\xvec) \nonumber \\
&= \frac{L}{\sigma}\Delta(\xvec, \uvec_0) + \inner{\grad f(\uvec_0)}{\xvec - \uvec_0} + f(\uvec_0) - f(\xvec) \nonumber \\
\label{eq:one_mem_q0_minus_J}
&\le \rbr{\frac{L}{\sigma} - \lambda_1} \Delta(\xvec, \uvec_0). \quad (\text{by } \lambda_1\text{-sc of } f)
\end{align}
By \eqref{eq:one_mem_gap_bound} and the definition $c_0 = L'$, we get
\begin{theorem}
For all $k \ge 1$ and $\xvec \in \dom \Psi$,
  \begin{align*}
  &J(\xvec_k) - J(\xvec) \\
  &\le \rbr{q_0(\xvec)  -  J(\xvec)} \min \cbr{  \rbr{1  -  \sqrt{\frac{\lambda}{L'}}}^k , \frac{4}{(2 + k)^2} } \\
  &\le \rbr{\frac{L}{\sigma} \! - \! \lambda_1} \Delta(\xvec, \uvec_0) \min \cbr{\rbr{1-\sqrt{\frac{\lambda}{L'}}}^k\!, \frac{4}{(2 + k)^2}}.
  \end{align*}
\end{theorem}

This rate is completely independent of $\Psi$ (except $\lambda_2$).  Although not needed by the proof, we can further show that $q_k(\xvec) \ge J(\xvec)$ for all $k \ge 0$ and $\xvec \in \dom \Psi$.

\begin{lemma}
  \label{lemma:one_mem_qx_nonneg}
  $q_k(\xvec) \ge J(\xvec)$ for all $k \ge 0$ and $\xvec \in \dom \Psi$.
\end{lemma}
\begin{proof}
When $k = 0$,
\begin{align*}
  &q_0 (\xvec) \! = \! \frac{L}{\sigma} \Delta(\xvec, \uvec_0) \! + \! f(\uvec_0) \! + \! \inner{\grad f(\uvec_0)}{\xvec \! - \! \uvec_0} \! + \! \Psi(\xvec) \\
  &\ge \frac{L}{2} \nbr{\xvec - \uvec_0}^2 + f(\uvec_0) + \inner{\grad f(\uvec_0)}{\xvec - \uvec_0} + \Psi(\xvec) \\
  &\ge f(\xvec) + \Psi(\xvec) = J(\xvec).
\end{align*}
Suppose $k \ge 1$. By \eqref{eq:one_mem_nadir_match}, $q_k(\xvec) \ge q_k(\zvec_k)$.  By Lemma \ref{lemma:one_mem_fleq}, $q_k(\zvec_k) \ge J(\xvec_k)$.  So
\[
q_k(\xvec) \ge q_k(\zvec_k) \ge J(\xvec_k) \ge J(\xvec).  \qedhere
\]
\end{proof}

\subsection{Adaptive $L$}

It is straightforward to incorporate backtracking of $L$ into the algorithm.  We present this variant in Algorithm \ref{algo:one_mem_nest_adapt_L}.  Suppose at each iteration the inner loop terminates with $L_k$ and define $L'_k = L_k / \sigma + \lambda_2$.  Noting $c_0 = L'_0$ and slightly changing the proof, Lemma \ref{lemma:one_mem_ck} can be extended as follows:
\begin{lemma}
  \label{lemma:one_mem_ck_adapt}
  For all $k \ge 1$, we have
  \begin{align*}
  \prod_{i=1}^k (1-a_i) \le &\min \Bigg \{\prod_{i=1}^k \rbr{1-\sqrt{\frac{ \lambda}{L'_i}}}, \\ &\qquad \quad \frac{4}{L'_0} \rbr{\frac{2}{\sqrt{L'_0}} + \sum_{i=1}^k \frac{1}{\sqrt{L'_i}}}^{-2} \Bigg \}.
  \end{align*}
  Obviously, when $L'_i = L'$ we recover Lemma \ref{lemma:one_mem_ck}.
\end{lemma}
Furthermore, \eqref{eq:one_mem_q0_minus_J} needs to be changed into
\[
q_0 (\xvec) - J(\xvec) \le \rbr{\frac{L_0}{\sigma} - \lambda_1} \Delta(\xvec, \uvec_0).
\]
So we conclude for all $k \ge 1$ and $\xvec \in \dom \Psi$,
\begin{align*}
  &J(\xvec_k) - J(\xvec) \le \rbr{L'_0 - \lambda} \Delta(\xvec, \uvec_0) \\
  &\cdot \min \Bigg \{ \prod_{i=1}^k \rbr{1 \! - \! \sqrt{\frac{\lambda}{L'_i}}}, \frac{4}{L'_0} \rbr{\frac{2}{\sqrt{L'_0}} + \sum_{i=1}^k \frac{1}{\sqrt{L'_i}}}^{-2} \Bigg \}.
\end{align*}
This bound does not involve the true $L$, and does not depend on $\Psi$ or the function value of $f$ (which could be used to hide $L$).

\begin{algorithm}[t]
\begin{algorithmic}[1]
\caption{\label{algo:one_mem_nest_adapt_L}AGM-EF-1 with adaptive $L$.}

 \REQUIRE{Down scaling factor $\gamma_d$ and up scaling factor $\gamma_u$ ($\gamma_d, \gamma_u > 1$).  An optimistic estimate $\Ltil \le L$.}

 \STATE Arbitrarily pick $\uvec_0 \in \dom \Psi$.  $L_0 \leftarrow \Ltil / \gamma_u$.

 \REPEAT

     \STATE $L_0 \leftarrow L_0 * \gamma_u$.

     \STATE Initialize $c_0 \leftarrow \frac{L_0}{\sigma} + \lambda_2$.

     \STATE $q_{0}(\xvec) \leftarrow \frac{L_0}{\sigma} \Delta(\xvec, \uvec_0) + \Psi(\xvec) + \ell_f(\xvec; \uvec_0, 0)$.

     \STATE $\xvec_0  = \zvec_0 \leftarrow \argmin_{\xvec} q_0(\xvec)$.

 \UNTIL{$J(\xvec_0) \le \min_{\xvec} q_0(\xvec_0)$}

 \FOR{$k = 0, 1, \ldots$}

   \STATE $L_{k+1} \leftarrow L_k / (\gamma_d * \gamma_u)$.

   \REPEAT
       \STATE $L_{k+1} \leftarrow L_{k+1} * \gamma_u$.

       \STATE Assign to $a_{k+1}$ the positive root (in $a$) of \\
        $\phantom{aa} \sigma (1-a)(c_k + \lambda_2 a) + \sigma \lambda_1 a = L_{k+1} a^2$.

       \STATE Do step 7 to 12 of Algorithm \ref{algo:one_mem_nest}.

   \UNTIL{$J(\xvec_{k+1}) \le q_{k+1}(\zvec_{k+1})$}

 \ENDFOR

\end{algorithmic}
\end{algorithm}

\subsection{Bounding the duality gap}

It is also not hard to extend AGM-EF-1 to the same primal-dual settings as in Section \ref{sec:inf_mem_duality_gap}.

Using \eqref{eq:one_mem_prop_line1} and \eqref{eq:one_mem_prop_line2}, we derive for all $\xvec \in \dom \Psi$:
\begin{align}
\label{eq:one_mem_qk_rel}
  q_{k+1}(\xvec) \le \! (1 \! - \! a_{k+1})q_k(\xvec) \! + \! a_{k+1} [\ell_{f}(\xvec; \uvec_{k+1}, \lambda_1) \! + \! \Psi(\xvec)].
\end{align}
This inequality allows us to express $q_k$ in terms of the linearizations of $f$ at $\uvec_i$.  For notational convenience, define $a_0 = 1$ and
\[
b_k(i) := a_i \prod_{j=i+1}^k (1 - a_j) \quad \text{for all } 0 \le i \le k,
\]
then it is easy to see that $\sum_{i=0}^k b_k(i) = 1$ for all $k \ge 1$.

\begin{lemma}
  \label{lem:one_mem_q_k_approx}
  For all $\xvec \in \dom \Psi$ and $k \ge 1$,
  \begin{align}
  &q_k(\xvec) \le b_k(0) q_0(\xvec) + \sum_{i=1}^k b_k(i) [\ell_f(\xvec; \uvec_i, \lambda_1) + \Psi(\xvec)] \nonumber \\
\label{eq:one_mem_qk_lt_sum_line}
  &= \! \frac{L}{\sigma} b_k(0) \Delta(\xvec, \uvec_0) \! + \! \Psi(\xvec) \! + \! \sum_{i=0}^k b_k(i) \ell_f(\xvec; \uvec_i, \lambda_1). \! \! \!
  \end{align}
\end{lemma}
\begin{proof}
  The inequality is obvious by inductively applying \eqref{eq:one_mem_qk_rel}.  The equality is by the definition of $q_0(\xvec)$ and the fact that $\sum_{i=0}^k b_k(i) = 1$.
\end{proof}

Go back to the settings of Section \ref{sec:inf_mem_duality_gap}.  We minimize $J(\xvec)$ by AGM-EF-1 and find some dual iterates $\val_k$ such that the duality gap $J(\xvec_k) - D(\val_k)$ goes to 0 fast.  Similar to \eqref{eq:inf_mem_x_to_val}, we construct
\begin{align}
  \label{eq:one_mem_x_to_val}
\val_k = \sum_{i=0}^k b_k(i) \val(\uvec_i).
\end{align}
Comparing with \eqref{eq:inf_mem_x_to_val}, we can see that both formulae are convex combinations of all the past $\val(\uvec_i)$ and higher weights are given to the later $\val(\uvec_i)$.  Computationally, $\val_k$ can be efficiently updated by recursion
\[
\val_0 = \val(\uvec_0), \text{ and } \val_{k+1} = (1 - a_{k+1}) \val_k + a_{k+1} \val(\uvec_{k+1}).
\]

To be self-contained, we state and prove the counterpart of Theorem \ref{thm:inf_mem_duality_gap} here.

\begin{theorem}[Bounds on the duality gap]
\label{thm:one_mem_duality_gap}
  Suppose a sequence $\cbr{\xvec_k, \uvec_k, \zvec_k}$ is produced when AGM-EF-1 is applied to minimize $J(\xvec)$ by treating $f$ as $\lambda_1$-sc.  Then the $\cbr{\val_k}$ defined by \eqref{eq:one_mem_x_to_val} satisfies $\val_k \in Q_2$ and
  \begin{align}
  \label{eq:one_mem_duality_bound}
  J(\xvec_k) - D(\val_k) \le \frac{L}{\sigma} b_k(0) \max_{\xvec \in \dom \Psi} \Delta(\xvec, \uvec_0).
  \end{align}
\end{theorem}
\begin{proof}
  Since $\val(\uvec_i) \in Q_2$ and $\val_k$ is a convex combination of them, so $\val_k \in Q_2$.  Clearly, \eqref{eq:inf_mem_duality_keystep} still holds.
  Denote the right-hand side of \eqref{eq:one_mem_duality_bound} as $M$.  Now by using relationship \eqref{eq:one_mem_qk_lt_sum_line} and Lemma \ref{lemma:one_mem_fleq}, we have
  \begin{align*}
    &J(\xvec_k) \le \min_{\xvec} q_k(\xvec) \\
    &\le \min_{\xvec} \cbr{ \frac{L}{\sigma} b_k(0) \Delta(\xvec, \uvec_0) \! + \! \Psi(\xvec) \! + \! \sum_{i=0}^k b_k(i) \ell_f(\xvec; \uvec_i, \lambda_1)} \\
    &\le M + \min_{\xvec} \cbr{ \Psi(\xvec) + \sum_{i=0}^k b_k(i) \phi(\xvec, \val(\uvec_i))} \\
    &\le M + \min_{\xvec} \cbr{ \Psi(\xvec) + \phi \rbr{\xvec, \sum_{i=0}^k b_k(i) \val(\uvec_i)}} \\
    &\le M + D(\val_k).  \qedhere
  \end{align*}
\end{proof}

\section{Application to Regularized Risk Minimization}
\label{sec:app_rrm}

Regularized risk minimization (RRM) is extensively used in machine learning.  In this section, we describe and compare in theory many different ways of training these models by APM.  The objective of RRM with linear models can be written as
\begin{align}
\label{eq:primal_obj_rrm}
  \min_{\wvec \in Q_1} J(\wvec) = \Omega(\wvec) + \gstar(A \wvec),
\end{align}
where $Q_1$ is a closed convex set.  Here, $\Omega(\wvec)$ corresponds to the regularizer and is assumed to be $\lambda$-sc wrt some prox-function $d_1$ on $Q_1$.  $d_1$ is in turn assumed to be $\sigma_1$-sc wrt a norm $\nbr{\cdot}$ on $Q_1$\footnote{In the sequel, $\nbr{\cdot}_p$ will stand for the $L_p$ norm.  Since each space has a single prescribed norm and the space that a variable belongs to is clear from the context, we will not use $\nbr{\cdot}_1$ to represent the norm on $Q_1$.}.  $A \wb$ stands for the output of a linear model, and $\gstar$ (the Fenchel dual of function $g$) encodes the empirical risk measuring the discrepancy between the correct labels and the output of the linear model ($A\wvec$).  Let the domain of $g$ be $Q_2$, which is also assumed to be closed and convex.

Using the definition of Fenchel dual, the primal objective \eqref{eq:primal_obj_rrm} can be rewritten as a minimax problem:
\begin{align}
\label{eq:lagrange_obj_rrm}
  \min_{\wvec \in Q_1} \max_{\alphavec \in Q_2} \Lcal(\wvec, \alphavec) := \Omega(\wvec) + \inner{A \wvec}{\alphavec} - g(\alphavec),
\end{align}
which further leads to the adjoint problem
\begin{align}
   &\max_{\alphavec \in Q_2} \cbr{-g(\alphavec) + \min_{\wvec \in Q_2} \cbr{\inner{A\wvec}{\alphavec} + \Omega(\wvec)}} \nonumber \\
\label{eq:adjoint_obj_rrm}
  \Leftrightarrow &\max_{\alphavec \in Q_2} D(\alphavec) := -g(\alphavec) - \Omega^{\star}(-A^{\top} \alphavec).
\end{align}

It is well known \citep[\eg][Theorem 3.3.5]{BorLew00} that under some mild constraint qualifications, the primal form $J(\wvec)$ and the adjoint form $D(\alphavec)$ satisfy
\[
J(\wb) \ge D(\alphab) \quad \text{and} \quad \inf_{\wb \in Q_1} J(\wb) = \sup_{\alphab \in Q_2} D(\alphab).
\]

Let us see some examples in machine learning which have the form \eqref{eq:primal_obj_rrm}.  Assume we have access to a training set of $n$ labeled examples
$\{(\xb_i, y_i)\}_{i=1}^{n}$ where $\xvec_i \in \RR^p$ and $y_i \in
\cbr{-1,+1}$.  Denote $Y := \diag (y_1, \ldots, y_n)$ and $X :=
(\xb_1, \ldots, \xb_n)$.

\paragraph{Example 1: binary SVMs with bias.}
The primal form of the binary linear SVM with bias is:
\begin{align*}
  J(\wvec) = \frac{\lambda}{2} \nbr{\wvec}^2 + \min_{b \in \RR} \frac{1}{n} \sum_{i=1}^n \sbr{1 - y_i (\inner{\xvec_i}{\wvec} + b)}_+.
\end{align*}
This can be posed in our framework by setting $Q_1 := \RR^p$, $A := -Y X^{\top}$, $\Omega(\wvec) = \frac{\lambda}{2} \nbr{\wvec}^2$, $\gstar(\uvec) = \min_{b \in \RR} \frac{1}{n} \sum_{i=1}^n \sbr{1 + u_i - y_i b}_+$.  This $\gstar$ corresponds to
\begin{align}
  g(\val) =
  \begin{cases}
    -\sum_{i} \alpha_i & \text{if }
    \val \in Q_2 \\
    +\infty & \text{otherwise},
  \end{cases}
\end{align}
where $Q_2$, the domain of $g$, is
\[
  Q_2 = \Big\{\val \in [0, n^{-1}]^n : \sum_i y_i \alpha_i = 0 \Big\}.
\]
Then the adjoint form turns out to be the well known SVM dual objective:
\begin{align}
  \label{eq:adjoint_linear_SVM}
  D(\alphab) &= \sum_i \alpha_i \! - \! \frac{1}{2 \lambda} \alphab^{\top} Y
  X^{\top} X Y \alphab, \ s.t. \ \val\in Q_2
\end{align}

\paragraph{Example 2: $L_1$ regularized SVM.}
The primal form of the $L_1$ regularized SVM ($L_1$-SVM, \citep{BenMan92}) is:
\begin{align*}
  J(\wvec) = \lambda \nbr{\wvec}_1 + \min_{b \in \RR} \frac{1}{n} \sum_{i=1}^n \sbr{1 - y_i (\inner{\xvec_i}{\wvec} + b)}_+.
\end{align*}
This can be posed in our framework by using exactly the same configurations as above, except that now $\Omega(\wvec) = \lambda \nbr{\wvec}_1$.  One can show that $\Omega^{\star}(\vvec) = 0$ if $\nbr{\vvec}_{\infty} \le \lambda$, and $\infty$ otherwise.
The adjoint form is:
\begin{align}
  \label{eq:adjoint_L1_SVM}
  D(\alphab) =  \begin{cases}
    \sum_i \alpha_i & \text{ if } \nbr{X Y \val}_{\infty}  \le  \lambda \\
    -\infty & \text{ otherwise}.
  \end{cases} \ s.t. \ \val\in Q_2.
\end{align}

\paragraph{Example 3: multivariate scores.}  \citet{Joachims05} proposed a max-margin model which directly optimizes the $F_1$ score.  Assume there are $n_+$ positive examples and $n_-$ negative examples.  $F_1$-score is defined by using the contingency table: $\Delta(\yvec', \yvec) := \frac{2a}{2a+b+c}$.
\begin{figure}[h]
\begin{minipage}[b]{0.21\textwidth}
\setlength{\tabcolsep}{2pt}
\begin{center}
Contingency table.
\end{center}
\begin{tabular}{c|c|c}
  \hline
   & $y \! = \! 1$ & $y \! = \! -1$ \\
   \hline
  $y' \! = \! 1$ & $a$ & $c$ \\
  $y' \! = \! -1$ & $b$ & $d$ \\
  \hline
\end{tabular}
\vspace{1em}
\begin{center}
$b$: false negative \\
$c$: false positive
\end{center}
\end{minipage}
\begin{minipage}[b]{0.26\textwidth}
\begin{align*}
  b &= \sum\nolimits_{i=1}^n \delta(y_i = 1, y'_i = -1) \\
  c &= \sum\nolimits_{i=1}^n \delta(y_i = -1, y'_i = 1) \\
  a &= n_+ - b \\
  d &= n_- - c. \quad (n = n_+ + n_-)
\end{align*}
$\delta(x)=1$ if $x$ is true.  Else 0.
\end{minipage}
\end{figure}

The primal objective proposed by \citet{Joachims05} is
\begin{align}
\label{eq:primal_multivar}
  J(\wvec) &= \frac{\lambda}{2} \nbr{\wvec}^2 \\
  &+ \max_{\yvec' \in \cbr{-1,1}^n} \sbr{\Delta(\yvec', \yvec) + \frac{1}{n} \sum_{i=1}^n \inner{\wvec}{\xvec_i} (y'_i - y_i)}. \nonumber
\end{align}
This can be recovered by setting $Q_1 = \RR^p$, $\Omega(\wvec) = \frac{\lambda}{2} \nbr{\wvec}^2$, and letting $A$ be a $2^n$-by-$p$ matrix where the $\yvec'$-th row is $\sum_{i=1}^n \xvec_i^{\top} (y'_i - y_i)$ for each $\yvec' \in \cbr{-1,+1}^n$.  Then $\gstar(\uvec) = \max_{\yvec'} \sbr{\Delta(\yvec', \yvec) + \frac{1}{n} u_{\yvec'}}$ which is induced by
\begin{align}
\label{eq:gval_multivariate}
  g(\val) =
  \begin{cases}
    -n \sum_{\yvec'} \Delta(\yvec', \yvec) \alpha_{\yvec'} & \text{if }
    \val \in Q_2 \\
    +\infty & \text{otherwise}
  \end{cases}.
\end{align}
Here $Q_2$, the domain of $g$, is
\[
Q_2 = \cbr{\alphab \in [0, n^{-1}]^{2^n} : \sum_{\yvec'} \alpha_{\yvec'} = \frac{1}{n} }.
\]
So we get the adjoint form
\begin{align*}
  D(\alphab) = -\frac{1}{2\lambda} \alphavec^{\top} A A^{\top} \alphavec + n \sum_{\yvec'} \Delta(\yvec', \yvec) \alpha_{\yvec'}, \ \alphavec \in Q_2.
\end{align*}

\paragraph{Example 4: Max-margin Markov Networks.}

The conditional random fields (CRFs) \citep{LafMcCPer01} and max-margin Markov network (\mcn s), \citep{TasGueKol04} are also instances of RRM.  First, they both minimize a regularized risk
with a square norm regularizer.  Second, they assume that there is a
joint feature map $\phivec$ which maps $(\xb, \yb)$ to a feature vector
in $\RR^{p}$.
Third, they assume a label loss $\ell(\yb, \yb^{i}; \xb^{i})$ which
quantifies the loss of predicting label $\yb$ when the correct label of
input $\xb^{i}$ is $\yb^{i}$. Finally, they assume that the space of
labels $\Ycal$ is endowed with a graphical model structure and that
$\phivec(\xb, \yb)$ and $\ell(\yb, \yb^{i}; \xb^{i})$ factorize
according to the cliques of this graphical model. The main difference is
in the loss function employed. CRFs minimize the $L_{2}$-regularized
logistic loss:
\begin{align}
  \label{eq:crf-objective}
  J(\wb) = \frac{\lambda}{2} \nbr{\wb}^2 + \frac{1}{n} &\sum_{i=1}^{n}
  \log \sum_{\yb \in \Ycal} \exp (\ell(\yb, \yb^{i}; \xb^{i})
      \\
  &-\inner{\wb}{\phivec(\xb^{i}, \yb^{i}) - \phivec(\xb^{i}, \yb)}), \nonumber
\end{align}
while the \mcn s minimize the $L_{2}$-regularized hinge loss
\begin{align}
  \label{eq:m3n-objective}
  J(\wb) = \frac{\lambda}{2} \nbr{\wb}^2 + \frac{1}{n} &\sum_{i=1}^{n}
  \max_{\yb \in \Ycal} \{ \ell(\yb, \yb^{i}; \xb^{i}) \\
  & - \inner{\wb}{\phivec(\xb^{i}, \yb^{i}) - \phivec(\xb^{i}, \yb)} \}. \nonumber
\end{align}
Clearly, both cases employ $Q_1 = \RR^p$ and $\Omega(\wvec) = \frac{\lambda}{2} \nbr{\wvec}^2$.  With shorthand $\psivec^{i}_{\yb} := \phivec(\xb^{i},
\yb^{i}) - \phivec(\xb^{i}, \yb)$ and $\ell^{i}_{\yb} := \ell(\yb,
\yb^{i}; \xb^{i})$, they both use an $(n \abr{\Ycal})$-by-$p$ matrix $A$ whose $(i, \yb)$-th row is $(-\psivec^{i}_{\yb})^{\top}$.  For \mcn s, $\gstar(\ub) = \frac{1}{n} \sum_i \max_{\yb} \cbr{\ell^{i}_{\yb} + u^i_{\yb}}$ and it can be verified that the corresponding $g$ is
\begin{align}
  \label{eq:m3n-gdef}
  g(\val) = 
  \begin{cases}
    -\sum_i \sum_{\yb} \ell^{i}_{\yb} \alpha^{i}_{\yb} & \text{if }
    \val \in Q_2 \\
    +\infty & \text{otherwise},
  \end{cases}
\end{align}
where $Q_2$, the domain of $g$, is
\[
Q_2 = \Scal^n := \cbr{\val \in [0, 1]^{n \abr{\Ycal}} : \sum_{\yb} \alpha^{i}_{\yb} = \frac{1}{n}, \ \forall \ i}.
\]
Clearly, $Q_2$ is convex and compact.  Now the adjoint form can be written as
\begin{align}
  \label{eq:dval-m3n}
  D(\val) \! = \! -\frac{1}{2\lambda}
  \val^{\top} A A^{\top} \val \! + \! \sum_i \sum_{\yb} \ell^{i}_{\yb}
  \alpha^{i}_{\yb},\  \val \in \Scal^{n}.
\end{align}

For CRFs, $\gstar(\uvec) = \frac{1}{n} \sum_i \log \sum_{\yvec \in \Ycal} \exp (\ell^i_{\yvec} + u^i_{\yvec})$, and the corresponding $g$ is
\begin{align}
  \label{eq:crf-gdef}
  g(\val) =
  \begin{cases}
    \sum\limits_{i=1}^{n} \sum\limits_{\yb} \alpha^{i}_{\yb} (\log \alpha^{i}_{\yb} - \ell^{i}_{\yb}) + \log n & \text{if }
    \val \in Q_2 \\
    +\infty & \text{otherwise},
  \end{cases}
\end{align}
The domain of $g$ is also $Q_2 = \Scal^n$.  Then the adjoint form is
\begin{align}
  \label{eq:dval-crf}
  D(\val) \! = \! -\frac{1}{2\lambda}
  \val^{\top} A A^{\top} \val \! &+ \! \sum\limits_{i=1}^{n} \sum\limits_{\yb} \alpha^{i}_{\yb} (\log \alpha^{i}_{\yb} \! - \! \ell^{i}_{\yb}) \\
  &+ \log n, \quad \val \in \Scal^{n}. \nonumber
\end{align}

\paragraph{Example 5: Entropy regularized LPBoost}
In \citep{WarGloVis08}, the entropy regularized LPBoost needs to minimize
\begin{align}
\label{eq:primal_ent_lpboost}
  J(\wvec) &= \lambda \Delta(\wvec, \wvec^0) + \max_{i \in [t]} \inner{\uvec_i}{\wvec}, \\
  s.t. \ \wvec &\in Q_1 := \cbr{\wvec \in [0, \nu]^n : \sum_{i=1}^n w_i = 1}. \nonumber
\end{align}
Here $\nu$ is a constant in $[0, 1]$, $\wvec^0 \in Q_1$ is the uniform distribution, and $\Delta$ is the Bregman divergence induced by the entropy (\ie\ $\Delta$ is the relative entropy).  $\uvec_i \in \RR^n$ is the so called edge vector.  This objective corresponds to $\Omega(\wvec) = \lambda \Delta(\wvec, \wvec^0)$, $A = (\uvec_1, \ldots, \uvec_t)^{\top}$, $\gstar(\svec) = \max_i s_i$ which is induced by $g(\alphavec) = 0$ if $\alphavec \in Q_2 := \Scal_t$, and $\infty$ otherwise.  Since
\[
\Omega^{\star}(\svec) = - \min_{\beta_i \ge 0} \cbr{\lambda \log \sum_{i=1}^n w^0_i \exp \rbr{\frac{s_i - \beta_i}{\lambda}} + \nu \sum_{i=1}^n \beta_i},
\]
so the adjoint form can be written as
\begin{align*}
  D(\val) = - \!\min_{\beta_i \ge 0} \cbr{\lambda \log \sum_{i=1}^n w^0_i \exp \rbr{\! \frac{A_{:i}^{\top} \val + \beta_i}{-\lambda} \!} \! + \! \nu \sum_{i=1}^n \beta_i}
\end{align*}
subject to $\val \in Q_2 = \Scal_t$.  Here $A_{:i}$ denotes the $i$-th column of $A$.  Although this form of $D(\val)$ is obscure, the strong convexity of $\Omega$ implies that $D(\val)$ is \lcg. The $\nu$ is introduced by \citep{WarGloVis08} to cap the density, and this cap is removed if $\nu = \infty$.  In that case, $\beta_i$ in the definition of $D(\val)$ will all be optimized to $0$ and we recover the well known log-sum-exp formula of $D(\val)$.

\paragraph{Example 6: Elastic net}
Using square loss as an example of the empirical risk, the primal objective of elastic net regularization is
\begin{eqnarray}
\label{eq:obj_elastic_net}
  J(\wvec) = \lambda \rbr{\gamma \nbr{\wvec}_1 + \frac{1}{2}\nbr{\wvec}_2^2} + \frac{1}{n} \sum_{i=1}^n (y_i - \xvec_i^{\top} \wvec)^2.
\end{eqnarray}
Here the $L_1$ normalizer $\nbr{\wvec}_1$ is introduced to promote the sparsity of the solution.  In this case, $\Omega(\wvec) = \lambda \rbr{\gamma \nbr{\wvec}_1 + \frac{1}{2}\nbr{\wvec}_2^2}$ and it dual is left as an exercise for the reader.  An equivalent formulation of \eqref{eq:obj_elastic_net} is by moving the regularizer into the constraint:
\begin{align*}
  \min_{\wvec} \quad \bar{J} (\wvec) &= \frac{1}{n} \sum_{i=1}^n (y_i - \xvec_i^{\top} \wvec)^2 \\
  s.t. \qquad \gamma \nbr{\wvec}_1 & + \frac{1}{2}\nbr{\wvec}_2^2 \le r.
\end{align*}
It can be shown that for any $\lambda>0$ there exists an $r>0$ such that $\argmin \bar{J} = \argmin J$ and vice versa.

There are also many regularized risk minimization problems which optimize over the space of positive semi-definite matrices, \eg\ \citep{JiYe09,Lu09,dAsBanElG08}.

\paragraph{Summary}

From these examples, we can see the following properties of $\Omega$ and $g$ which will also be assumed for our general treatment of the objective \eqref{eq:primal_obj_rrm} and \eqref{eq:adjoint_obj_rrm}.  Firstly, the function $\Omega(\wvec)$ which serves as a regularizer is strongly convex.  In Example 1, 3, 4, 6, $\Omega(\wvec)$ is $\lambda$-sc wrt the Euclidean norm.  In Example 5, $f(\wvec)$ is $\lambda$-sc wrt the $L_1$ norm.  As a result, $\Omega^{\star}$ must be $\frac{1}{\lambda}$-\lcg\ on $\RR^p$.  Secondly, the \lcg\ constant of $\Omega^{\star}(-A^{\top} \val)$ in $\val$ also depends on the matrix norm of $A$, which in turn depends on the choice of norm on $Q_1$ and $Q_2$.  Thirdly, the $\gstar$ is not necessarily differentiable (\eg, hinge loss), but $g$ is always \lcg\ on $Q_2$.  Finally, $Q_2$ is bounded and its diameter can be well controlled.  This is important for translating dual solutions into the primal.

Our goal is to minimize $J(\wvec)$ over $Q_1$, and we do not really care about solving the dual $D(\val)$ over $Q_2$.  However, since $D(\val)$ has favorable smooth properties, we also often work in the dual as a proxy.  To solve $J(\wvec)$ (and $D(\val)$), there are three main approaches.

\paragraph{Smoothing $\gstar$ to a fixed level.}
To handle the nonsmoothness of $\gstar$, we can smooth it by using the technique introduced by \citet{Nesterov05}.  Then the composite form, $\Omega(\wvec)$ plus the smoothed variant of $\gstar(A\wvec)$, fits the form of AGM-EF and can be solved in $\wvec$ (primal), $\val$ (dual) or primal-dual.  Given a prescribed accuracy $\epsilon$, $\gstar$ only needs to be smoothed to a fixed extent.

\paragraph{Smoothing $\gstar$ with decreasing smoothness.}
\citep{Nesterov05a} introduced a primal-dual method where $\gstar$ is smoothed with decreased smoothness (\ie\ increased closeness to $\gstar$).  As a result, it tends to the optimal solution of $D(\val)$ and $J(\wvec)$, instead of just attaining a prescribed accuracy $\epsilon$.

\paragraph{No smoothing.}
Given the smoothness of the dual problem $D(\val)$, AGM can be applied to maximize it and then convert $\val_k$ into $\wvec_k$ by \eqref{eq:inf_mem_x_to_val} and \eqref{eq:one_mem_x_to_val}.  No smoothing of $\gstar$ is needed in this case.

The next three subsections will describe these schemes in detail, with focus on the rates of convergence and how each iteration can be performed efficiently.  Moreover, we provide intuitions on which scheme is more suitable.  For brevity, we will only use AGM-EF-$\infty$ with fixed $L$ as an example, while similar results can be straightforwardly derived for AGM-EF-1 and adaptive $L$.  In this version of the paper, we illustrate all these ideas on Example 1 (SVM with bias).

\subsection{Smoothing $\gstar$ to a fixed level}
\label{sec:ml_sol_primal_smooth}

A key technique introduced by \citet{Nesterov05} was to tightly approximate the nonsmooth part $\gstar(A\wvec)$ by a smooth surrogate.  The idea of the approach originates from the Theorem \ref{theorem:SC_LCG} in Appendix \ref{sec:app:convex_ana} which connects the strong convexity of a function and \lcg\ of its Fenchel dual.  $\gstar$ is not \lcg\ because $g$ is not strongly convex, therefore to make $\gstar$ smooth a natural idea is to add to $g$ a strongly convex function $d_2$ on $Q_2$ and then dualize it:
\begin{align}
  \gstar_{\mu} (\ub) &:= (g + \mu d_2)^{\star}(\ub) \nonumber \\
  \label{eq:gstar-with-d2}
  &= \sup_{\val \in Q_2}
  \cbr{\inner{\val}{\ub} - g(\val) - \mu \, d_2(\val)}.
\end{align}
Here $\mu \ge 0$ and $d_2$ is assumed to be $\sigma_2$-sc wrt a norm on $Q_2$.\footnote{We can also use the more general form of strong convex as in Definition \ref{def:gen_sc}.  Here we use the conventional definition for simplicity.}  By proper centering, $d_2$ can be assumed to satisfy
\[
\min_{\val \in Q_2} d_2(\val) = 0.
\]
Let us further define
\[
\val_0 = \argmin_{\val \in Q_2} d_2(\val), \quad D:= \max_{\val \in Q_2} d_2(\val).
\]
The main restriction of this approach is that $D$ must be well bounded.  Using the definition in \eqref{eq:gstar-with-d2} we can easily characterize the \emph{uniform} tightness of the approximation: for all $\uvec \in Q_2$
\begin{align}
  \label{eq:bounds-on-gstar}
  \gstar(\ub) - \mu D \leq \gstar_{\mu} (\ub) \leq \gstar(\ub).
\end{align}

Furthermore, the \lcg\ constant of $\gstar_{\mu}(A\wvec)$ in $\wvec$ wrt the norm on $Q_1$ can be estimated as follows.  By Theorem \ref{theorem:SC_LCG}, $\gstar_{\mu}$ is $(\mu \sigma_2)^{-1}$-\lcg\ wrt the dual norm on $Q_2$.  So we can apply the chain rule:
\begin{align*}
  &\nbr{\frac{\partial}{\partial \wvec} \gstar_{\mu}(A \wvec_1) - \frac{\partial}{\partial \wvec} \gstar_{\mu}(A \wvec_2)}^* \\
  &= \nbr{A (\grad \gstar_{\mu}(A \wvec_1) - \grad \gstar_{\mu}(A \wvec_2)}^* \\
  & \le \nbr{A} \frac{1}{\mu \sigma_2} \nbr{A \wvec_1 - A \wvec_2}^* \le \frac{\nbr{A}^2}{\mu \sigma_2} \nbr{\wvec_1 - \wvec_2}.
\end{align*}
That is, $\gstar_{\mu}(A\wvec)$ is \lcg\ in $\wvec$ with constant
\begin{align}
\label{eq:Lgmu}
  L_g(\mu) \le \frac{\nbr{A}^2}{\mu \sigma_2}.
\end{align}

\paragraph{Example 1: smoothing the hinge loss.}
The hinge loss $[1-w]_+$ is the dual of $g(\alpha) = \alpha$ for $\alpha \in [-1,0]$ and $\infty$ elsewhere.  Adding $\frac{\mu}{2} \alpha^2$ to $g$ and dualize it, we get
\begin{align*}
  \gstar_{\mu}(w) = \begin{cases}
    0 & \text{if } w \ge 1 \\
    \frac{(1-w)^2}{2 \mu} & \text{if } w \in [1 - \mu, 1] \\
    1 - w - \frac{\mu}{2} & \text{if } w \le 1 - \mu
  \end{cases}.
\end{align*}
Some smoothed hinge loss $\gstar_{\mu}(w)$ with various $\mu$ are plotted in Figure \ref{fig:smooth_hinge}.

\paragraph{Example 2: smoothing max into soft max.}
In the entropy regularized LPBoost, $\gstar(\svec) = \max_i s_i$ and $g(\uvec) = 0$ if $\Scal_t$ and $\infty$ otherwise..  Then adding prox-function $\sum_i s_i \ln s_i$ to $g$ and dualizing it, we get
\[
\gstar_{\mu}(\svec) = \mu \ln \sum_i \exp \rbr{\frac{s_i}{\mu}}.
\]
When $\mu \to 0$, this soft max recovers max.

With the smoothed $\gstar_{\mu}$ in place, we now discuss how to find an $\epsilon$ accurate solution to $J(\wvec)$ by three different schemes: primal ($\wvec$), dual ($\val$), and primal-dual.

{
\begin{figure}[t]
  \includegraphics[width=0.4\textwidth]{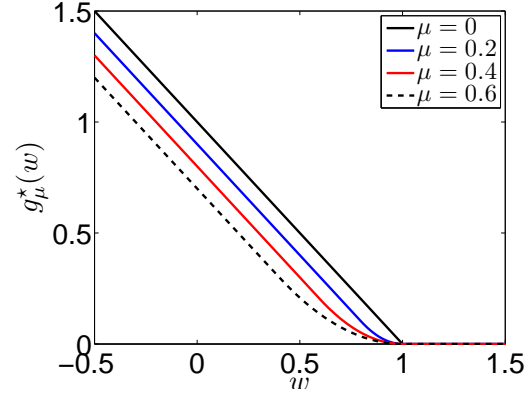}\\
  \caption{Smoothing hinge loss with different $\mu$.}
  \label{fig:smooth_hinge}
\end{figure}
}

\subsubsection{Solving in the primal $\wvec$.}

We will use $\gstar_{\mu}$ to define a new objective function
\begin{align}
\label{eq:reg_primal}
  J_{\mu}(\wb) &:= \Omega(\wb) + \gstar_{\mu}(A\wb) \\
   &= \Omega(\wb) + \max_{\val \in Q_2} \cbr{\inner{A \wb}{\val}-g(\val) -
    \mu \, d_2(\val)}. \nonumber
\end{align}

Since $J_{\mu}(\wvec) \le J(\wvec)$ for all $\wvec$, to make sure that an $\epsilon$ accurate solution to $J_{\mu}$ is a $2 \epsilon$ accurate solution to $J$, a sufficient condition is that their deviation be upper bounded \emph{everywhere} by $\epsilon$, \ie\ $\max_{\wvec} J(\wvec) - J_{\mu}(\wvec) < \epsilon$.  By \eqref{eq:bounds-on-gstar}, this is guaranteed if $\mu$ is small enough
\begin{align}
\label{eq:choice_smooth_mu}
\mu \le \frac{\epsilon}{D}.
\end{align}

Plugging \eqref{eq:choice_smooth_mu} into \eqref{eq:Lgmu}, we obtain that the \lcg\ constant of $\gstar_{\mu}(A\wvec)$ is at most $\frac{\nbr{A}^2 D}{\sigma_2 \epsilon}$.  Let $\wvec^* = \argmin_{\wvec} J(\wvec)$.  Bearing in mind that $\Omega$ is $\lambda$-sc, AGM-EF-$\infty$ is readily applicable to $J_{\mu}(\wvec)$ and the following rate of convergence can be inferred from Theorem \ref{thm:inf_mem_rate_conv}:
\begin{align*}
  J_{\mu}(\wvec_k) - J_{\mu}(\wvec^*) &\le \Delta(\wvec^*, \uvec_0) \min \Bigg \{\frac{4 D \nbr{A}^2}{\sigma_1 \sigma_2 \epsilon (k+1)^2}, \\
  &\frac{4 D \nbr{A}^2}{\sigma_1 \sigma_2 \epsilon} \rbr{1 + \sqrt{\frac{\sigma_1 \sigma_2 \lambda \epsilon}{4 D \nbr{A}^2}}}^{-2k+2} \Bigg \}.
\end{align*}
Once $J_{\mu}(\wvec_k) - J_{\mu}(\wvec^*) \le \epsilon$, we must have
\begin{align*}
  J(\wvec_k) - J(\wvec^*) \le J_{\mu}(\wvec_k) + \mu D - J_{\mu}(\wvec^*) \le 2\epsilon.
\end{align*}
Therefore, we obtain the following theorem.
\begin{theorem}
\label{thm:smooth_primal}
  For any given $\epsilon > 0$, setting $\mu$ by the equality in \eqref{eq:choice_smooth_mu} and applying AGM-EF-$\infty$ to $J_{\mu}(\wvec)$, we can guarantee that $\wvec_k$ is a $2\epsilon$ accurate solution of $J(\wvec)$ as long as
  \begin{align}
  \label{eq:step_smooth_primal}
  &k \ge \min \Bigg \{ \frac{1}{\epsilon} \sqrt{\frac{4 D \nbr{A}^2}{\sigma_1 \sigma_2} \Delta(\wvec^*, \uvec_0)}, \\
  & 1 \! + \! \frac{1}{2} \ln \rbr{\frac{4 D \nbr{A}^2}{\sigma_1 \sigma_2 \epsilon^2} \Delta(\wvec^*, \uvec_0)} \! \! \Bigg / \!  \ln \rbr{1 \! + \! \sqrt{\frac{\sigma_1 \sigma_2 \lambda \epsilon}{4 D \nbr{A}^2}}} \! \Bigg \}. \nonumber
  \end{align}
\end{theorem}

Note $\ln (1 + \epsilon) \approx \epsilon$ when $\epsilon$ is close to 0, so the denominator in the second term becomes $O(\sqrt{\epsilon})$ and overall the second term is approximately $O\rbr{\frac{1}{\sqrt{\epsilon}}\ln \frac{1}{\epsilon}}$.  The first term does not depend on $\lambda$.  Note also that this bound does not explicitly depend on the diameter of $Q_1$ which is infinity in many cases.  A closer look shows that $\Delta(\wvec^*, \uvec_0)$ hides the dependence on $\lambda$.  With a small regularization parameter $\lambda$, $\Delta(\wvec^*, \uvec_0)$ may be large and could approach infinity when $\lambda$ tends to 0.

Unfortunately, the bound on the duality gap in \eqref{eq:inf_mem_duality_bound} does use the diameter of $Q_1$, and it cannot be replaced by $\Delta(\wvec^*, \uvec_0)$ as in Theorem \ref{thm:smooth_primal}.  Therefore, we do lose a termination criteria. Fortunately, this problem in duality gap can be avoided if we optimize in $\val$.  Before describing it in detail, let us illustrate the above procedure on training the SVM with bias.

Here, choose $d_1$ and $d_2$ as the Euclidean norm square and the norms on $Q_1$ and $Q_2$ are both Euclidean norm.  Then $\nbr{A}^2 = \lambda_{\max}(A^{\top} A) = \lambda_{\max}(XX^{\top})$ where $\lambda_{\max}$ stands for the maximum eigenvalue. $\sigma_1 = \sigma_2 = 1$.  The diameter of $Q_2$ is $D \le n \frac{1}{n^2} = \frac{1}{n}$.  For a given $\epsilon$, set $\mu = n \epsilon$ by \eqref{eq:choice_smooth_mu}. Suppose all $\xvec_i$ lie in the ball with Euclidean radius $R$.  Then $\lambda_{\max} (XX^{\top}) \le n R^2$ and the second term in \eqref{eq:step_smooth_primal} is essentially
\[
O\rbr{\ln \frac{1}{\epsilon} \sqrt{\frac{4 \lambda_{\max} (XX^{\top})}{\lambda n \epsilon}}} \le O\rbr{\frac{2 R}{\sqrt{\lambda \epsilon}} \ln \frac{1}{\epsilon} }.
\]

Solving in the primal is also advantageous in terms of the condition number.  When $\gstar$ is smoothed by small $\mu$ or when the regularization parameter $\lambda$ is small, the condition number $c := L_g(\mu) / \lambda$ becomes very large.  According to Theorem \ref{thm:inf_mem_rate_conv}, the number of iterations to find an $\epsilon$ accurate solution is the min of
\[
O\rbr{\frac{\log \frac{1}{\epsilon}}{\log \rbr{1+\sqrt{\frac{\lambda}{L_g(\mu)}}}} } \quad \text{and} \quad O\rbr{\frac{\sqrt{L_g(\mu)}}{\sqrt{\epsilon}}}.
\]
So the linear convergence rate depends on $c$ by $O(\sqrt{c})$, as opposed to $O(c)$ in most linearly converging algorithms, \eg\ gradient descent.  Second, the $\min$ in Theorem \ref{thm:inf_mem_rate_conv} implies that when $\lambda$ is very small and the objective is very poorly conditioned, the linear convergence will be automatically superseded by the $1/\sqrt{\epsilon}$ rate which has better ``constant".  Some class of algorithms require manual rewiring in such a case, \eg\ \citep{TeoVisSmoLe10} and \citep{DoLeFoo09}.

Finally, it is noteworthy that this method does not require $g$ be \lcg.

\subsubsection{Solving in the dual $\val$.}

Similar to $J_{\mu}$ in \eqref{eq:reg_primal}, we can also define a smoothed version of $D(\val)$:
\begin{align}
  D_{\mu}(\val) &:= -\mu d_2(\val) - g(\val) + \min_{\wvec} \cbr{\Omega(\wvec) + \inner{A \wvec}{\val}} \nonumber \\
\label{eq:D_mu}
  &= -\mu d_2(\val) - g(\val) - \Omega^{\star}(-A^{\top} \alphavec)
\end{align}
which is to be maximized over $\val \in Q_2$.  So we can pose $-D_{\mu}(\val)$ in the composite form,
\[
f(\val) = g(\val) + \Omega^{\star}(-A^{\top} \alphavec), \text{ and } \Psi(\val) = \mu d_2(\val),
\]
to which AGM-EF-$\infty$ and AGM-EF-1 can be applied.  Since $\Omega^{\star}$ is $1/\lambda$-\lcg, $f(\val)$ must be \lcg\ with constant
\begin{align}
\label{eq:lcg_const_dual}
L_f = \frac{\nbr{A}^2}{\lambda} + L_g,
\end{align}
where $L_g$ is the \lcg\ constant of $g$.  $\Psi$ is $\mu$-sc.  Applying the primal-dual scheme in Section \ref{sec:inf_mem_duality_gap} with $-D_{\mu}$ and $-J_{\mu}$ playing the role of $J$ and $D$ therein respectively, we get
\begin{align*}
J_{\mu}(\wvec_k) - D_{\mu}(\val_k) &\le \max_{\val \in Q_2} \Delta(\val, \uvec_0) \cdot \min \Bigg \{ \frac{4 L_f}{\sigma_2 (k+1)^2}, \\
&\qquad \qquad \frac{L_f}{\sigma_2} \rbr{1 + \sqrt{\frac{\sigma_2 \mu}{4 L_f}}}^{-2k+2} \Bigg \}.
\end{align*}
Once $J_{\mu}(\wvec_k) - D_{\mu}(\val_k) \le \epsilon$, it is ensured that
\begin{align*}
  J(\wvec_k) - \min_{\wvec} J(\wvec) &\le J_{\mu}(\wvec_k) + \mu D - \max_{\val \in Q_2} D(\val) \\
  &\le J_{\mu}(\wvec_k) + \epsilon - D(\val_k) \\
  &\le J_{\mu}(\wvec_k) + \epsilon - D_{\mu}(\val_k) \le 2 \epsilon.
\end{align*}
So we conclude the following theorem.
\begin{theorem}
\label{thm:smooth_dual}
  For any given $\epsilon > 0$, setting $\mu$ by the equality in \eqref{eq:choice_smooth_mu} and applying the primal-dual scheme in Section \ref{sec:inf_mem_duality_gap} to $-D_{\mu}$ and $-J_{\mu}$, we can guarantee that $\wvec_k$ is a $2\epsilon$ accurate solution of $J(\wvec)$ as long as
  \begin{align}
  \label{eq:step_smooth_dual}
  &k \ge \min \Bigg \{ \frac{1}{\sqrt{\epsilon}} \sqrt{\frac{4 M (\nbr{A}^2 + \lambda L_g)}{\lambda \sigma_2}} - 1, \\
  & \qquad \qquad \qquad 1 + \frac{1}{2} \frac{\ln \rbr{\frac{4 M (\nbr{A}^2 + \lambda L_g)}{\lambda \sigma_2 \epsilon}} }{\ln \rbr{1 + \sqrt{\frac{\lambda \epsilon \sigma_2}{4 D ( \nbr{A}^2 + \lambda L_g)}}}} \Bigg \}. \nonumber
  \end{align}
  where $M := \max_{\val \in Q_2} \Delta(\val, \uvec_0)$.
\end{theorem}

It is important to note that this scheme requires $g$ be \lcg, while solving in the primal does not make such a requirement.

Let us apply the scheme to SVM with bias, and use the same choice of norm and prox-function as before.  Now $L_g = 0$ and $M = 1/n$.  Using the approximation $\ln (1+x) \approx x$ when $\abr{x} \ll 1$, \eqref{eq:step_smooth_dual} becomes
\begin{align*}
  k \ge \min \cbr{\frac{2R}{\sqrt{\lambda \epsilon}}-1, 1 + \frac{R}{\sqrt{\lambda \epsilon}} \ln \rbr{\frac{4R^2}{\lambda \epsilon}}}.
\end{align*}

As a final note, the way we smooth the empirical risk is different from \citep{Chapelle06} which changes hinge loss into square hinge loss or higher order.  Our method has a smoothing parameter which trades smoothness for the tightness of the approximation.  In contrast, the square hinge loss is just a heuristic approximation and no bound is available in optimization for its solution.

\subsection{Smoothing $\gstar$ with decreasing smoothness}
A typical primal-dual solver for the objectives in \eqref{eq:primal_obj_rrm} and \eqref{eq:lagrange_obj_rrm} is the excessive gap technique \citep[EGT,][]{Nesterov05a}.  One concrete application is \citep{ZhaSahVis10b} where EGT is used to solve the above Example 4 (\mcn\ and CRF).  Unfortunately, EGT forces a fixed way to initialize $\wvec_0$ and $\val_0$.  This is very inconvenient for homology and other warm-start techniques which utilize the closeness of solutions under small perturbations of the problem parameter (\eg\ $\lambda$).

\subsection{No smoothing of $\gstar$}

Since we assume $g$ is \lcg\ and $\Omega$ is $\lambda$-sc, so the dual \eqref{eq:adjoint_obj_rrm} is \lcg\ and AGM-EF-$\infty$ is applicable.  Since our ultimate goal is to minimize $J(\wvec)$ we adopt the primal-dual scheme in Section \ref{sec:inf_mem_duality_gap}. The \lcg\ constant of $D$ is exactly the $L_f$ in \eqref{eq:lcg_const_dual}.  Treating $-D$ and $-J$ as the $J$ and $D$ therein respectively, we get
\begin{align*}
  J(\wvec_k) - D(\val_k) \le \frac{4 M (\nbr{A}^2 + \lambda L_g)}{\lambda \sigma_2 (k+1)^2}.
\end{align*}
When applied to SVM with bias where $M=1/n$ and $L_g = 0$, we get that $J(\wvec_k) - D(\val_k) < \epsilon$ for all
\[
k \ge \frac{2 R}{\sqrt{\lambda \epsilon}} - 1.
\]

When comparing the rates, it is important to bear in mind that machine learning problems usually do not need a high accuracy solution and so $\epsilon = 10^{-2}$ or $10^{-3}$ might suffice.  In many cases, $\lambda$ will be set to very small such as $10^{-6}$.  Therefore $\frac{1}{\epsilon}$ can be much smaller than $\frac{1}{\lambda}$.  Also, we are currently bounding $\nbr{A}^2$ by $n R^2$ which can be very loose in practice.  The dependence of $\Delta(\wvec^*, \uvec_0)$ on $\lambda$ is not clear either.  Finally in practice when solving in the dual, the box constraints in SVM can cause considerable waste of gradient computation.  Therefore the rates above just provide limited guidance and the most appropriate optimization strategy has to be picked empirically.

\subsection{Efficient computation of the gradient}
\label{sec:computation_trick}

So far, we have ignored the computational complexity per iteration which is dominated by two operations: computing the gradient and minimizing the model $\psi_k$ in AGM-EF-$\infty$ (or $q_k$ in AGM-EF-1).  We first show in this subsection that the gradient in all the above examples can be computed efficiently.  Indeed, the gradients needed are $\frac{\partial}{\partial \wvec} \gstar_{\mu}(A \wvec)$ and $\frac{\partial}{\partial \val} \Omega^{\star}(-A^{\top} \val)$, with the former always being more challenging.  So we focus on calculating $\frac{\partial}{\partial \wvec} \gstar_{\mu}(A \wvec)$.

By chain rule, $\frac{\partial}{\partial \wvec} \gstar_{\mu}(A \wvec) = A^{\top} \grad \gstar_{\mu}(A \wvec)$.  Using \citep[][Theorem X.1.4.4]{HirLem93}, $\grad \gstar_{\mu}(\uvec)$ can be computed by
\begin{align}
\label{eq:grad_gstar_mu}
\grad \gstar_{\mu}(\uvec) = \argmax_{\val \in Q_2} \inner{\uvec}{\val} - g(\val) - \mu d_2(\val).
\end{align}
In the case of multivariate score \eqref{eq:primal_multivar} and \eqref{eq:gval_multivariate}, the dimension of the domain of $g$ is exponentially high in the number of training examples, and therefore it will be intractable to first compute $\grad \gstar_{\mu}(A \wvec)$ and then pre-multiply it with $A^{\top}$ ($A$ has exponentially many rows).  Similar tractability issues appear in learning with structured outputs as in \mcn.  Below we present a dynamic programming based algorithm, which costs  $O(n^2)$ time and space complexity to calculate $A^{\top} \grad \gstar_{\mu}(A \wvec)$ for
\[
d_2(\val) = \sum_i \alpha_i \ln \alpha_i.
\]

In this case, the optimization problem in \eqref{eq:grad_gstar_mu} is
\[
\min_{\val \in Q_2} \mu \sum_{\yvec'} \alpha_{\yvec'} \ln \alpha_{\yvec'} - n \sum_{\yvec'} \Delta(\yvec', \yvec) \alpha_{\yvec'} - \sum_{\yvec'} u_{\yvec'} \alpha_{\yvec'}.
\]
Noting that the $\yvec'$-th row of $A$ is $\varphi^{\top}_{\yvec'} := \sum_i (y'_i - y_i) \xvec_i^{\top}$, we get $u_{\yvec'} = \varphi^{\top}_{\yvec'} \wvec = \sum_i (y'_i - y_i) \xvec_i^{\top} \wvec$.  Following the standard procedures (\eg\ \citep[][Lemma 8]{ZhaSahVis10b}), the optimal solution can be written as
\begin{align*}
\alpha^*_{\yvec'} := \frac{1}{n Z} \exp \rbr{\frac{1}{\mu} \sum_i y'_i \xvec_i^{\top} \wvec + \frac{n}{\mu} \Delta(\yvec', \yvec)}, \\
\where Z := \sum_{\yvec'} \exp \rbr{\frac{1}{\mu} \sum_i y'_i \xvec_i^{\top} \wvec + \frac{n}{\mu} \Delta(\yvec', \yvec)}.
\end{align*}
So $\alpha^*_{\yvec'}$ can be interpreted as a distribution over $\yvec'$ (normalized to $\frac{1}{n}$ rather than 1).  Then
\begin{align*}
\frac{\partial}{\partial \wvec} \gstar_{\mu}(A\wvec) &= \sum_{\yvec'} \alpha^*_{\yvec'} \varphi_{\yvec'} = \sum_{\yvec'} \alpha^*_{\yvec'} \sum_i (y'_i - y_i) \xvec_i \\
&= -2 \sum_i y_i \xvec_i \sum_{\yvec' \sim -y_i} \alpha^*_{\yvec'} \\
&= -2 \sum_i p(y'_i = -y_i) y_i \xvec_i,
\end{align*}
where $\yvec' \sim -y_i$ means summing up all $\yvec'$ whose $i$-th element $y'_i$ equals $-y_i$.  So $\sum_{\yvec' \sim -y_i} \alpha^*_{\yvec'}$ is exactly the marginal probability $p(y'_i = -y_i)$ under the joint distribution $\alpha^*_{\yvec'}$.  Now we show how to compute the marginal distributions efficiently.

\begin{figure}[t]
  \centering
    \begin{tikzpicture}[domain=0:5,samples=100,scale=0.9]

      \draw[color=blue,thick] (0, 0) -- (7, 7);
      \draw[color=blue,thick] (1, 0) -- (7, 6);
      \draw[color=blue,thick] (2, 0) -- (7, 5);
      \draw[color=blue,thick] (3, 0) -- (7, 4);
      \draw[color=blue,thick] (4, 0) -- (7, 3);
      \draw[color=blue,thick] (5, 0) -- (7, 2);
      \draw[color=blue,thick] (6, 0) -- (7, 1);
      \draw[color=blue,thick] (0, 0) -- (7, 0);
      \draw[color=blue,thick] (1, 1) -- (7, 1);
      \draw[color=blue,thick] (2, 2) -- (7, 2);
      \draw[color=blue,thick] (3, 3) -- (7, 3);
      \draw[color=blue,thick] (4, 4) -- (7, 4);
      \draw[color=blue,thick] (5, 5) -- (7, 5);
      \draw[color=blue,thick] (6, 6) -- (7, 6);

      \node[fill,draw,circle, inner sep=2pt,color=black,thick] at (0, 0) {} ;
      \node[fill,draw,circle, inner sep=2pt,color=black,thick] at (1, 0) {} ;
      \node[fill,draw,circle, inner sep=2pt,color=black,thick] at (2, 0) {} ;
      \node[fill,draw,circle, inner sep=2pt,color=black,thick] at (3, 0) {} ;
      \node[fill,draw,circle, inner sep=2pt,color=black,thick] at (4, 0) {} ;
      \node[fill,draw,circle, inner sep=2pt,color=black,thick] at (5, 0) {} ;
      \node[fill,draw,circle, inner sep=2pt,color=black,thick] at (6, 0) {} ;
      \node[fill,draw,circle, inner sep=2pt,color=black,thick] at (7, 0) {} ;
      \node[fill,draw,circle, inner sep=2pt,color=black,thick] at (1, 1) {} ;
      \node[fill,draw,circle, inner sep=2pt,color=black,thick] at (2, 1) {} ;
      \node[fill,draw,circle, inner sep=2pt,color=black,thick] at (3, 1) {} ;
      \node[fill,draw,circle, inner sep=2pt,color=black,thick] at (4, 1) {} ;
      \node[fill,draw,circle, inner sep=2pt,color=black,thick] at (5, 1) {} ;
      \node[fill,draw,circle, inner sep=2pt,color=black,thick] at (6, 1) {} ;
      \node[fill,draw,circle, inner sep=2pt,color=black,thick] at (7, 1) {} ;
      \node[fill,draw,circle, inner sep=2pt,color=black,thick] at (2, 2) {} ;
      \node[fill,draw,circle, inner sep=2pt,color=black,thick] at (3, 2) {} ;
      \node[fill,draw,circle, inner sep=2pt,color=black,thick] at (4, 2) {} ;
      \node[fill,draw,circle, inner sep=2pt,color=black,thick] at (5, 2) {} ;
      \node[fill,draw,circle, inner sep=2pt,color=black,thick] at (6, 2) {} ;
      \node[fill,draw,circle, inner sep=2pt,color=black,thick] at (7, 2) {} ;
      \node[fill,draw,circle, inner sep=2pt,color=black,thick] at (3, 3) {} ;
      \node[fill,draw,circle, inner sep=2pt,color=black,thick] at (4, 3) {} ;
      \node[fill,draw,circle, inner sep=2pt,color=black,thick] at (5, 3) {} ;
      \node[fill,draw,circle, inner sep=2pt,color=black,thick] at (6, 3) {} ;
      \node[fill,draw,circle, inner sep=2pt,color=black,thick] at (7, 3) {} ;
      \node[fill,draw,circle, inner sep=2pt,color=black,thick] at (4, 4) {} ;
      \node[fill,draw,circle, inner sep=2pt,color=black,thick] at (5, 4) {} ;
      \node[fill,draw,circle, inner sep=2pt,color=black,thick] at (6, 4) {} ;
      \node[fill,draw,circle, inner sep=2pt,color=black,thick] at (7, 4) {} ;
      \node[fill,draw,circle, inner sep=2pt,color=black,thick] at (5, 5) {} ;
      \node[fill,draw,circle, inner sep=2pt,color=black,thick] at (6, 5) {} ;
      \node[fill,draw,circle, inner sep=2pt,color=black,thick] at (7, 5) {} ;
      \node[fill,draw,circle, inner sep=2pt,color=black,thick] at (6, 6) {} ;
      \node[fill,draw,circle, inner sep=2pt,color=black,thick] at (7, 6) {} ;
      \node[fill,draw,circle, inner sep=2pt,color=black,thick] at (7, 7) {} ;

      \draw[->, thick] (0, 1.1) -- (0, 3.2);
      \node[] at (1, 3.5) {\# false negative};
      \node[] at (0.1, 0.6) {$(0, 0)$};

      \node[] at (0.5, -1) {$y'_1$};
      \node[] at (2, -1) {$\ldots$};
      \node[] at (4.5, -1) {$y'_{k+1}$};
      \node[] at (3.5, -1) {$y'_{k}$};
      \node[] at (6.5, -1) {$y'_{n_+}$};
      \node[] at (5.5, -1) {$\ldots$};

      \node[] at (1, -0.4) {$1$};
      \node[] at (2, -0.4) {$\ldots$};
      \node[] at (3, -0.4) {$k-1$};
      \node[] at (4, -0.4) {$k$};
      \node[] at (5, -0.4) {$k+1$};
      \node[] at (6, -0.4) {$\ldots$};
      \node[] at (7, -0.4) {$n_+$};

      \node[right] at (7.2, 0) {$\text{fn} = 0$};
      \node[right] at (7.2, 1) {$\text{fn} = 1$};
      \node[right] at (7.2, 2) {$\text{fn} = 2$};
      \node[right] at (7.2, 4) {$\vdots$};
      \node[right] at (7.2, 6) {$\text{fn} = n_+ \! - \! 1$};
      \node[right] at (7.2, 7) {$\text{fn} = n_+$};

      \draw[->,color=blue,thick] (1, 5) -- (2, 6);
      \draw[->,color=blue,thick] (1, 5) -- (2, 5);
      \node[] at (1.5, 4.8) {$c_k$};
      \node[] at (1.3, 5.8) {$\frac{1}{c_k}$};
      \node[right] at (2, 6.2) {$y'_k = -1$};
      \node[right] at (2, 5.1) {$y'_k = 1$};
      \node[] at (2, 4.5) {$k$};
      \node[] at (1, 4.5) {$k \! - \! 1$};

    \end{tikzpicture}
  \caption{Path weight interpretation of the normalizer $Z$ and the marginal distributions $p(y'_k)$.}
    \label{fig:f1score_marginal}
\end{figure}
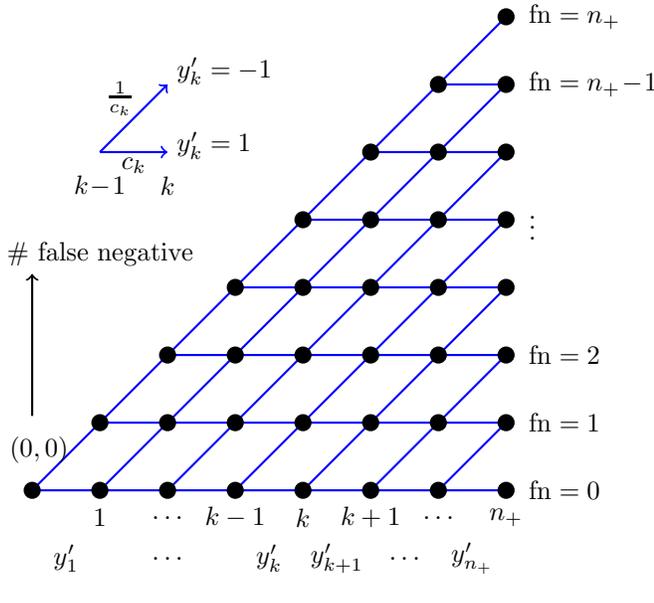

Unlike the inference in graphical models, there is no clique factorization in $\yvec'$.  Fortunately, $\cbr{y'_i}$ are coupled only through the loss $\Delta(\yvec', \yvec)$ which in turn depends only on two ``sufficient statistics" of $\yvec'$: false negative $b$ and false positive $c$.  For simplicity, we sometimes also write $\Delta(\yvec', \yvec)$ as $\Delta(b, c)$.  Without loss of generality, assume the positive training examples are the first $n_+$ ones ($y_1 = \ldots = y_{n_+} = 1$), and the negative examples are the last $n - n_+$ ones ($y_{n_+ + 1} = \ldots = y_n = -1$).  Denote $\yvec'_+ := (y'_1, \ldots, y'_{n_+})^{\top}$ and $\yvec'_- := (y'_{n_+ + 1}, \ldots, y'_{n})^{\top}$.  $\yvec'_+ \sim b$ represents that $\yvec'_+$ commits $b$ false negatives, \ie\ $\sum_{i=1}^{n_+} \delta(y'_i = -1) = b$.  $\yvec'_- \sim c$ represents that $\yvec'_-$ commits $c$ false negatives, \ie\ $\sum_{i=n_+ + 1}^n \delta(y'_i = 1) = c$.  For simplicity, denote
\[
c_k := \exp \rbr{\frac{1}{\mu} \xvec_k^{\top} \wvec}.
\]

Let us first compute the normalizer $Z$ as follows.
{
\allowdisplaybreaks
\begin{align*}
  Z &= \sum_{b=0}^{n_+} \sum_{c = 0}^{n_-} \sum_{\yvec'_+ \sim b} \sum_{\yvec'_- \sim c} \rbr{\frac{1}{\mu} \sum_{i=1}^n y'_i \xvec_i^{\top} \wvec + \frac{n}{\mu} \Delta(\yvec', \yvec)} \\
  &= \sum_{b=0}^{n_+} \sum_{c = 0}^{n_-} \exp\rbr{\frac{n}{\mu}\Delta(b, c)} \underbrace{\sum_{\yvec'_+ \sim b} \exp\rbr{\frac{1}{\mu} \sum_{i=1}^{n_+} y'_i \xvec_i^{\top} \wvec}}_{=: V_+(b)} \\
  & \qquad \qquad \cdot \underbrace{\sum_{\yvec'_- \sim c} \exp\rbr{\frac{1}{\mu} \sum_{i=n_+ + 1}^{n} y'_i \xvec_i^{\top} \wvec}}_{=: V_-(c)}
\end{align*}
}

Therefore, once we have $V_+(b)$ for all $b \in [n_+]$ and $V_-(c)$ for all $c \in [n_-]$, then $Z$ can be computed in $n_+ n_-$ steps.  For simplicity we only show to compute $V_+(b)$, and $V_-(c)$ can be computed in exactly the same way.

For each fixed $b$, $V_+(b)$ can be equivalently reformulated by Figure \ref{fig:f1score_marginal}.  Each node $(k,f)$ represents that $\yvec'$ has committed $f$ false negatives in the first $k$ examples: $\sum_{i=1}^k \delta(y'_i = -1) = f$.  Each node is connected to two nodes on its right: $(k+1, f+1)$ and $(k+1, f)$.  The former corresponds to $y'_{k+1} = -1$, \ie\ one more false negative is committed.  So we attach to the diagonal edge a weight $\exp\rbr{-\frac{1}{\mu} \xvec_{k+1}^{\top} \wvec} = c_{k+1}^{-1}$.  The latter means $y'_{k+1} = 1$ and the false negative is not incremented.  So the horizontal edge is attached with weight $c_{k+1}$.  A \emph{path} from $(k, f)$ to $(k', f')$ ($k \le k'$ and $f \le f'$) is a sequence of nodes moving from $(k, f)$ to $(k', f')$ along the edges of the graph: $(k, f_0) = (k, f) \to (k + 1, f_1) \to \ldots \to (k + s, f_s) = (k', f')$ where $s = k' - k$ and $f_{i+1} - f_i = 0$ or $1$.  The weight of a path is defined as the product of the weight of all edges on that path.

Clearly $V_+(b)$ is equal to the total weight of all paths from $(0, 0)$ to $(n_+, b)$.  To compute it, define $\alpha_k(v)$ as the total weight of all paths from $(0, 0)$ to $(k, v)$.  Then it is not hard to see the following recursion for all $k = 1, \ldots, n_+$ and $v = 0, 1, \ldots, k$:
\begin{align}
  \alpha_{k}(v) = \frac{1}{c_k} \alpha_{k-1}(v-1) + c_k \alpha_{k-1}(v),
\end{align}
where $\alpha_k(-1) := 0$ and $\alpha_k(k+1) := 0$ for all $k$.  Algorithm \ref{algo:forward_f1score} computes $V_+(b) = \alpha_{n_+}(b)$ for all $b \in [n_+]$.  Clearly the computational cost is $O(n^2_+)$.  If we only need $V_+(b)$ then the space complexity is $O(n_+)$.  But later we will need all $\alpha_k(v)$ so we keep $O(n_+^2)$ memory.  Taking into account the similar cost for $V_-(c)$, the total spatial and computational cost is both $O(n^2)$.

\begin{algorithm}[t]
\begin{algorithmic}[1]
\caption{\label{algo:forward_f1score}Forward propagation to compute all $\cbr{V_+(b): 0 \le b \le n_+}$.}

 \STATE Initialize $\alpha_0(0) = 1$.

 \FOR{$k = 1, \ldots, n_+$}

   \FOR{$v = 0, 1, \ldots, k$}

      \STATE $\alpha_k(v) = \frac{1}{c_k} \alpha_{k-1}(v-1) + c_k \alpha_{k-1}(v)$.
   \ENDFOR

 \ENDFOR

 \STATE \textbf{Return:} $V_+(b) = \alpha_{n_+}(b)$ for all $0 \le b \le n_+$.
\end{algorithmic}
\end{algorithm}

To compute the marginal distributions $p(y'_k)$ we need a backward propagation.  For example let us consider $p(y'_k = 1)$ for $k \in [n_+]$, and the case of $k > n_+$ (negative examples) can be dealt with similarly.  By the definition of $\alpha_{\yvec'}$, it suffices to compute
{
\allowdisplaybreaks
\begin{align*}
  Z_k &:= \sum_{\yvec': y'_k = 1} \exp \rbr{\frac{1}{\mu} \sum_i y'_i \xvec_i^{\top} \wvec + \frac{n}{\mu} \Delta(\yvec', \yvec)} \\
  &= \sum_{b=0}^{n_+} \sum_{c=0}^{n_-} \! \! \exp \rbr{\frac{n}{\mu}\Delta(b,c)} \! \! \! \sum_{\yvec'_+ \sim b, y'_k = 1} \! \! \! \! \! \! \! \! \exp\rbr{\frac{1}{\mu} \sum_{i=1}^{n_+} y'_i \xvec_i^{\top} \wvec} \\
  & \qquad \qquad \cdot \sum_{\yvec'_- \sim c} \exp\rbr{\frac{1}{\mu} \sum_{i=n_+ + 1}^{n} y'_i \xvec_i^{\top} \wvec} \\
  &= \sum_{b=0}^{n_+} \underbrace{\sum_{\yvec'_+ \sim b, y'_k = 1} \! \! \! \! \exp\rbr{\frac{1}{\mu} \sum_{i=1}^{n_+} y'_i \xvec_i^{\top} \wvec}}_{=: T_+^k(b)} \\
  &\qquad \qquad \qquad \cdot \underbrace{\sum_{c=0}^{n_-} \exp \rbr{\frac{n}{\mu}\Delta(b,c)} V_-(c)}_{=: \eta_-(b)}.
\end{align*}
}

\begin{algorithm}[t]
\begin{algorithmic}[1]
\caption{\label{algo:backward_f1score}Backward propagation to compute $p(y'_k)$ for all $k \in [n_+]$.}

 \STATE Initialize $\xi_{n_+}(v) = \eta_-(v)$ for all $v = 0, 1, \ldots, n_+$.

 \STATE $Z_{n_+} = c_{n_+} \sum_{v=0}^{n_+} \alpha_{n_+ - 1}(v) \xi_{n_+}(v)$.

 \FOR{$k = n_+ - 1, \ldots, 1$}

   \FOR{$v = 0, 1, \ldots, k$}

      \STATE $\xi_k(v) = c_{k+1} \xi_{k+1}(v) + \frac{1}{c_{k+1}} \xi_{k+1}(v+1)$.

   \ENDFOR

   \STATE $Z_{k} = c_{k} \sum_{v=0}^{k-1} \alpha_{k-1}(v) \xi_k(v)$.

 \ENDFOR

 \STATE \textbf{Return:} $p(y'_k = 1) = \frac{Z_i}{n Z}$ for all $k \in [n_+]$.
\end{algorithmic}
\end{algorithm}

Since $V_-(c)$ available from forward propagation, $\cbr{\eta_-(b)}$ can be computed in $O(n^2)$ time.  So the only problem left is to compute $T_+^k(b)$.  $T_+^k(b)$ has a very intuitive interpretation in Figure \ref{fig:f1score_marginal}: the total weight of all paths from $(0,0)$ to $(n_+, b)$ with the $k$-th step (\ie\ between the horizontal coordinate $k-1$ and $k$) going horizontal (not diagonal).  Let $\beta^b_k(v)$ denote the total weight of all paths from $(k, v)$ to $(n_+, b)$.  Then
\[
T_+^k(b) = \sum_{v = 0}^{k-1} \alpha_{k-1}(v) c_k \beta_{k}^b(v).
\]
So
\begin{align*}
  Z_k &=  \sum_{b=0}^{n_+} T_+^k(b) \eta_-(b) = \sum_{b=0}^{n_+} \sum_{v=0}^{k-1} \alpha_{k-1}(v) c_k \beta_k^b (v) \eta_-(b) \\
  &= c_k \sum_{v=0}^{k-1} \alpha_{k-1}(v) \underbrace{\sum_{b=0}^{n_+} \beta_k^b(v) \eta_-(b)}_{=:\xi_k(v)}.
\end{align*}
Therefore as long as $\xi_k(v)$ can be updated efficiently, so is $Z_k$.  Fortunately, $\beta_k^b(v)$ has a recursive form
\[
\beta_{k}^b(v) = c_{k+1} \beta_{k+1}^b(v) + \frac{1}{c_{k+1}} \beta_{k+1}^b(v+1),
\]
for all $0 \le k \le n_+ - 1$, $0 \le v \le k$ and $0 \le b \le n_+$.  This implies for all $0 \le k \le n_+ - 1$ and $0 \le v \le k$
\begin{align*}
  \xi_k(v) &= \sum_{b=0}^{n_+} \beta_k^b(v) \eta_-(b) \\
  &= \sum_{b=0}^{n_+} [c_{k+1} \beta_{k+1}^b(v) + \frac{1}{c_{k+1}} \beta_{k+1}^b(v+1)] \eta_-(b) \\
  &= c_{k+1} \xi_{k+1}(v) + \frac{1}{c_{k+1}} \xi_{k+1}(v+1).
\end{align*}
The final algorithm is summarized in Algorithm \ref{algo:backward_f1score}.  Its time and space cost is both $O(n^2)$.  The initialization of $\xi_k$ therein is based on initializing $\beta_{n_+}^b(v) = \delta(v=b)$ for all $b, v = 0, 1, \ldots, n_+$.

The gradient of $\gstar(A \wvec)$ for \mcn s can also be computed efficiently by dynamic programming, but the key structure it exploits is the clique decomposition in graphical models.  Details can be found in \citep{ZhaSahVis10b}.

\subsection{Minimizing the model efficiently}

In this section, we show that the model $\psi_k$ can be minimized efficiently.

\subsubsection{Diagonal quadratic constrained to a box and a hyperplane}
\label{sec:proj_simplex}

When AGM-EF is applied to solve the dual optimization problem $D(\alphab)$ for SVM in \eqref{eq:adjoint_linear_SVM}, each iteration needs to solve the model subject to $Q_2$.  This can be reduced to a box constrained diagonal QP with a single linear equality constraint:

\begin{align}
\label{eq:simple_qp_projection}
  \min \frac{1}{2} \sum_{i=1}^n &d_i^2 (\alpha_i - m_i)^2 \\
  s.t. \qquad l_i \le &\alpha_i \le u_i \quad \forall i \in [n];  \nonumber \\
  \sum_{i=1}^n & \sigma_i \alpha_i = z. \nonumber
\end{align}

Similarly, when solving in the primal with smoothing in \eqref{eq:reg_primal}, the gradient query also involves an optimization in this form.  In this section, we focus on the following the QP in \eqref{eq:simple_qp_projection}.  The algorithm we describe below stems from \cite{ParKov90} and finds the exact optimal solution in $O(n)$ time, faster than the $O(n \log n)$ complexity in \cite{DucShaSigCha08}.  \cite{DucShaSigCha08} also proposes a median finding based algorithm which has linear time complexity \emph{in expectation}.  In contrast, our method is deterministic and linear.  \citet{LiuYe09} tackle this problem too, but they use the mean bisection and apply Newton's method to find a solution up to an inexact prespecified accuracy $\delta$.  The resulting total cost is $O(n \log \frac{1}{\delta})$.

Without loss of generality, we assume $l_i < u_i$ and $d_i \neq 0$ for all $i$.
Also assume $\sigma_i \neq 0$ because otherwise $\alpha_i$ can be solved independently.
  To make the feasible region nonempty, we also assume
\begin{align*}
z &\ge \sum_i \sigma_i (\delta(\sigma_i > 0) l_i + \delta(\sigma_i < 0) u_i) \\
\text{and} \quad z &\le \sum_i \sigma_i (\delta(\sigma_i > 0) u_i + \delta(\sigma_i < 0) l_i).
\end{align*}

With a simple change of variable $\beta_i = \sigma_i (\alpha_i - m_i)$, the problem \eqref{eq:simple_qp_projection} is simplified as
  \begin{align*}
    \min \qquad \frac{1}{2} \sum_{i=1}^n & \dbar^{2}_i \beta_i^2 \\
    s.t. \qquad l'_i \le & \beta_i \le u'_i \quad \forall i \in [n];  \\
    \sum_{i=1}^n & \beta_i = z',
  \end{align*}
where $\dbar^2_i = \frac{d_i^2}{\sigma_i^2}$, $l'_i = \left \{ \begin{array}{ll}
   \sigma_i (l_i - m_i) & \text{if } \sigma_i > 0  \\
   \sigma_i (u_i - m_i) & \text{if } \sigma_i < 0  \\
\end{array} \right.$, $u'_i = \left\{ \begin{array}{ll}
   \sigma_i (u_i - m_i) & \text{if } \sigma_i > 0  \\
   \sigma_i (l_i - m_i) & \text{if } \sigma_i < 0  \\
\end{array} \right.$, and $z' = z - \sum_i \sigma_i m_i$.
Write out its partial Lagrangian:
\[
\min_{\beta_i \in [l'_i, u'_i]} \max_{\lambda \in \RR} \sum_{i=1}^n  \frac{1}{2} \dbar^{2}_i \beta_i^2 - \lambda \rbr{\sum_{i=1}^n  \beta_i - z'}.
\]

{
\begin{figure}[t]
\vspace{2em}
\centering
    \includegraphics[width=0.4\textwidth]{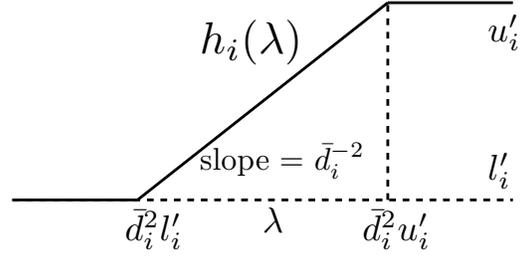}
\caption{$h_i(\lambda)$}
\label{fig:hi_lambda}
\end{figure}
}

Due to strong duality, we can swap the $\min$ and $\max$:
\begin{align}
&\max_{\lambda \in \RR} \min_{\beta_i \in [l'_i, u'_i]}  \sum_{i=1}^n \frac{1}{2} \dbar^{2}_i \beta_i^2 - \lambda \rbr{\sum_{i=1}^n  \beta_i - z'} \nonumber \\
&= \max_{\lambda \in \RR} \sum_i \min_{\beta_i \in [l'_i, u'_i]} \rbr{ \frac{1}{2} \dbar_i^2 \beta_i^2 - \lambda \beta_i} + \lambda z' \nonumber \\
\label{eq:minimax_dual}
& \Leftrightarrow \min_{\lambda \in \RR} \sum_i \underbrace{\max_{\beta_i \in [l'_i, u'_i]} \rbr{ -\frac{1}{2} \dbar_i^2 \beta_i^2 + \lambda \beta_i}}_{:=H_i(\lambda)} - \lambda z'
\end{align}
Clearly, the optimal $\beta_i^*(\lambda)$ in the definition of $H_i(\lambda)$ can be solved analytically, and this gives
\begin{align*}
  H_i(\lambda) &= \begin{cases}
    -\frac{1}{2} \dbar^{2}_i u'^{2}_i + \lambda u'_i & \text{if } \lambda > u'_i \dbar^{2}_i \\
    -\frac{1}{2} \dbar^{2}_i l'^{2}_i + \lambda l'_i & \text{if } \lambda < l'_i \dbar^{2}_i \\
    \frac{\lambda^2}{2 \dbar^{2}_i} & \text{if } \lambda \in [l'_i \dbar^{2}_i, u'_i \dbar^{2}_i]
  \end{cases} \\
  \text{with} \quad
  \beta_i^*(\lambda) &= \begin{cases}
    u'_i & \text{if } \lambda > u'_i \dbar^{2}_i \\
    l'_i & \text{if } \lambda < l'_i \dbar^{2}_i \\
    \frac{\lambda}{2 \dbar^{2}_i} & \text{if } \lambda \in [l'_i \dbar^{2}_i, u'_i \dbar^{2}_i]
  \end{cases}.
\end{align*}

{
\begin{figure*}[t]
\begin{centering}
\subfloat[$\min S < \dbar^{2}_i l'_i < \dbar^{2}_i u'_i < \max S$]{
    \includegraphics[width=5.4cm]{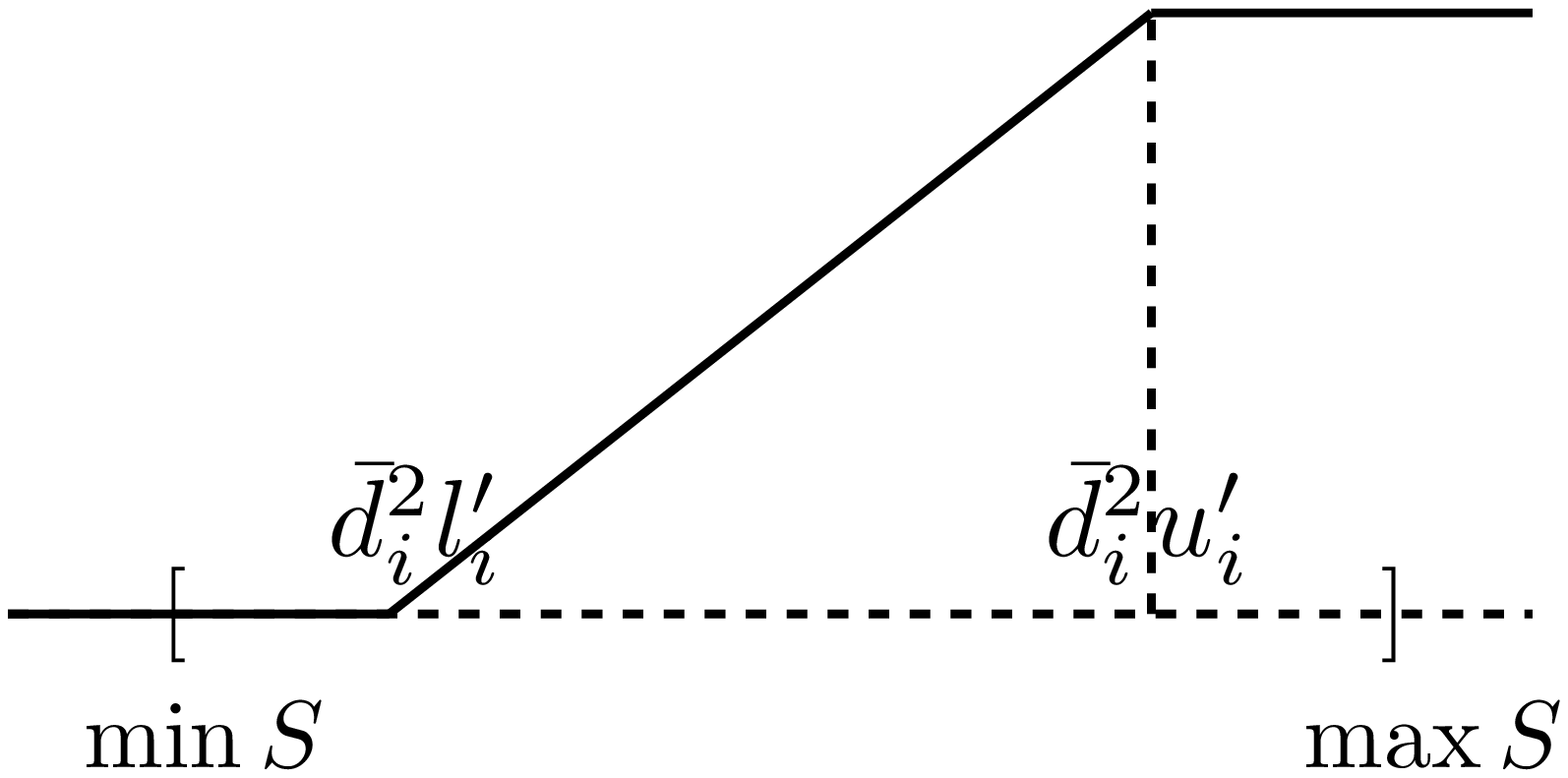}} ~~
\subfloat[$\min S < \dbar^{2}_i l'_i < \max S \le \dbar^{2}_i u'_i$]{
    \includegraphics[width=5.4cm]{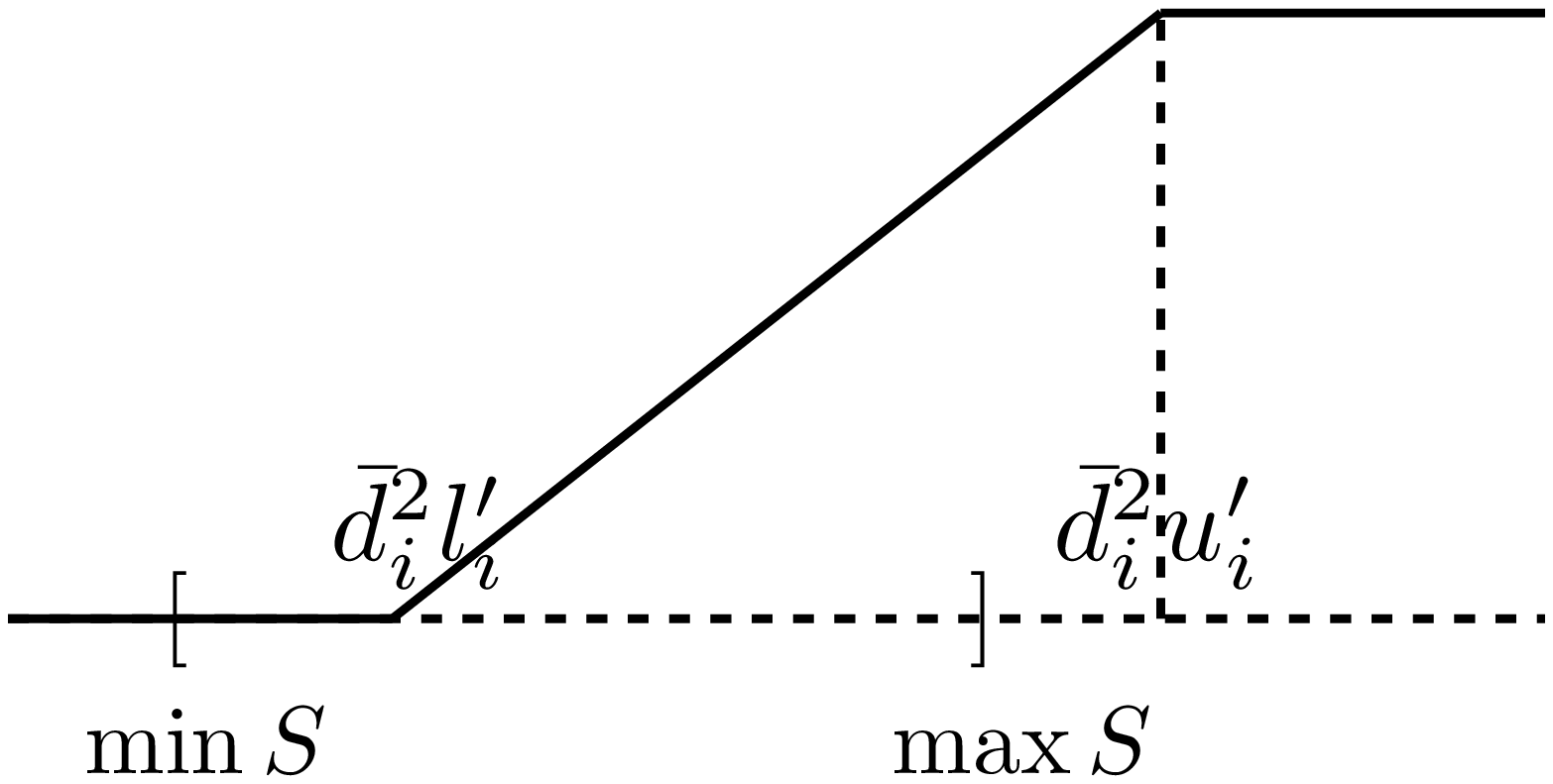}} ~~
\subfloat[$\dbar^{2}_i l'_i \le \min S < \dbar^{2}_i u'_i < \max S$]{
    \includegraphics[width=5.4cm]{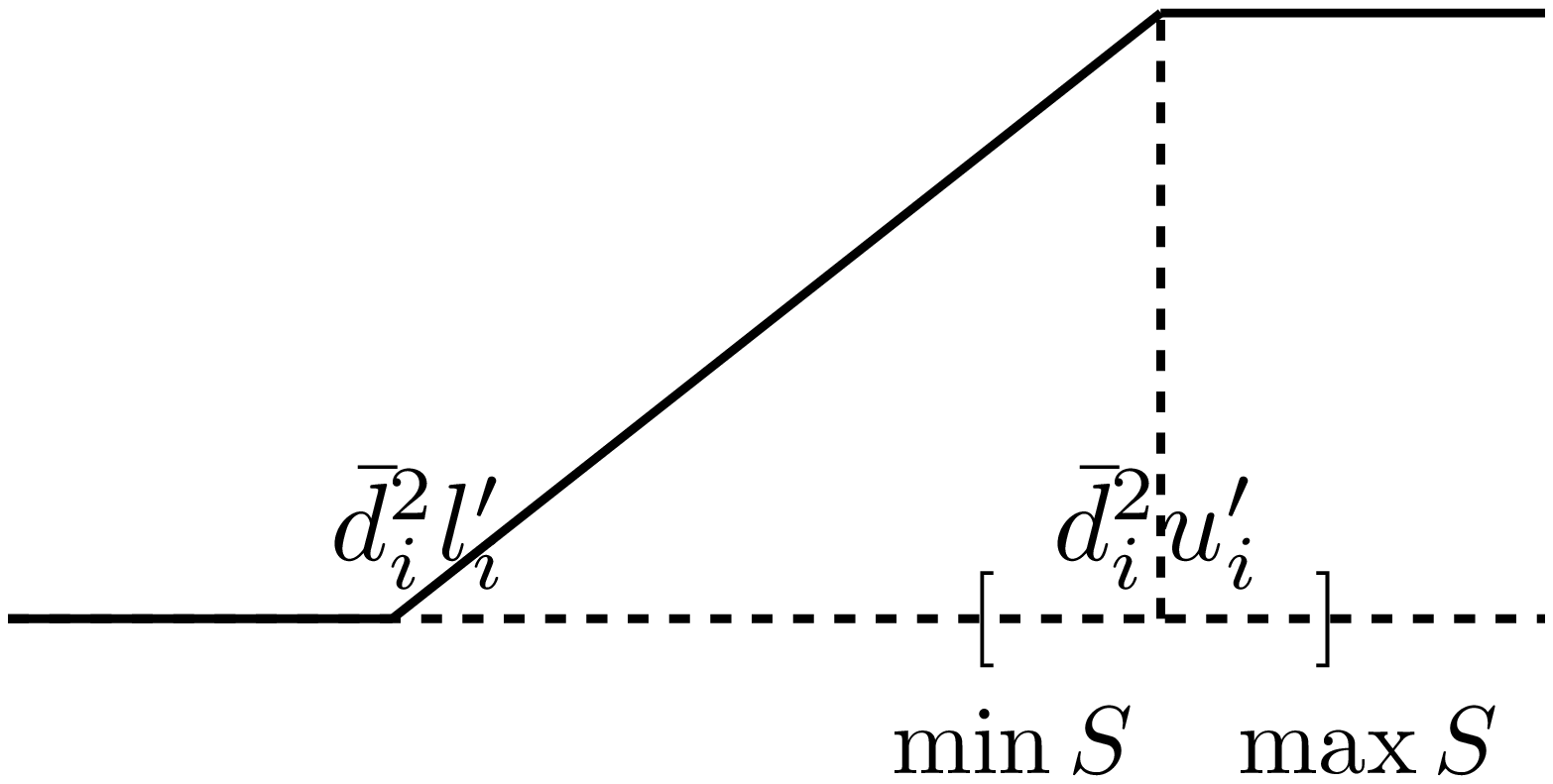}} ~~ \\
\subfloat[$\min S < \max S \le \dbar^{2}_i l'_i$]{
    \includegraphics[width=5.4cm]{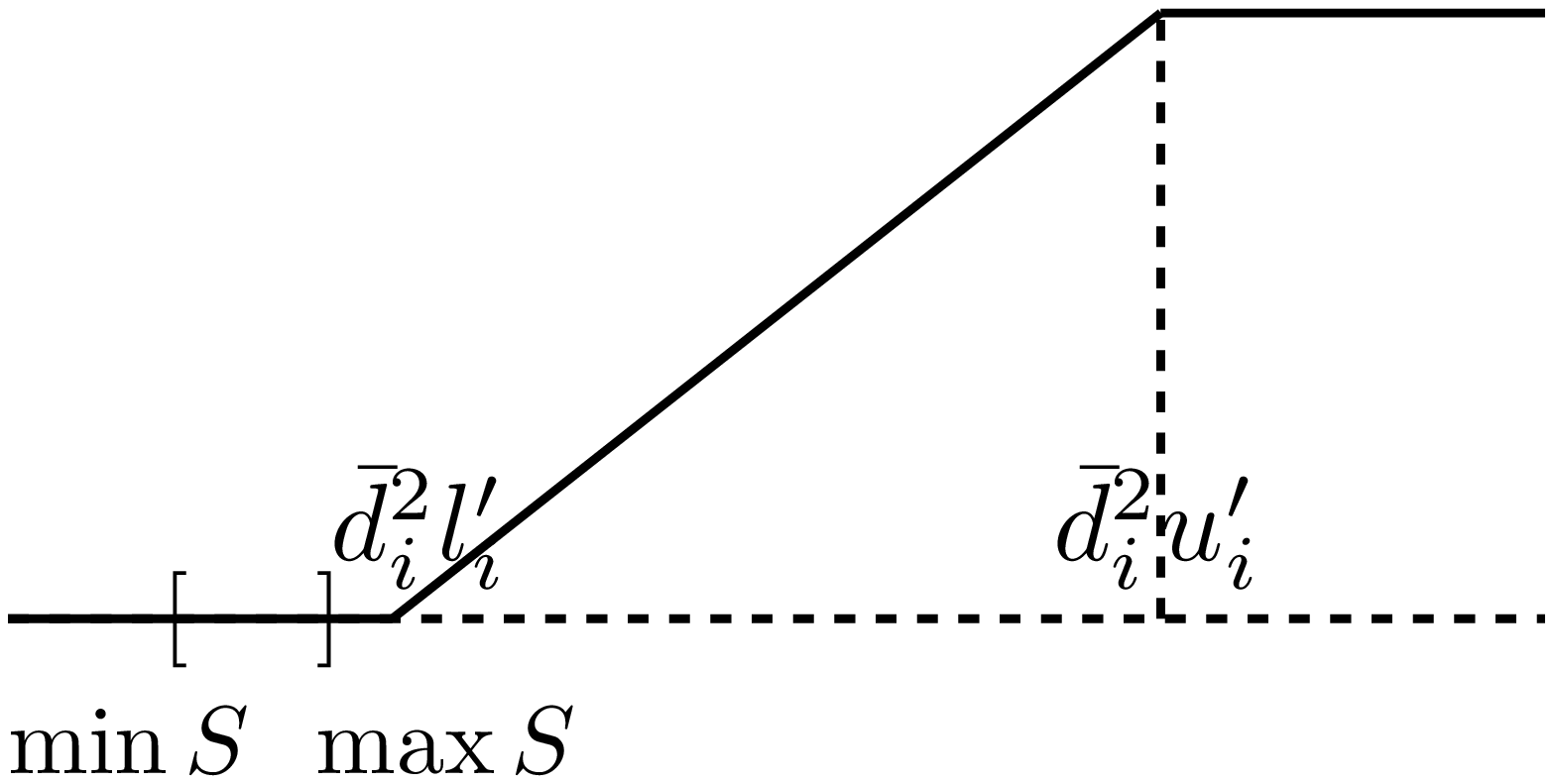}} ~~
\subfloat[$\dbar^{2}_i l'_i \le \min S < \max S \le \dbar^{2}_i u'_i $]{
    \includegraphics[width=5.4cm]{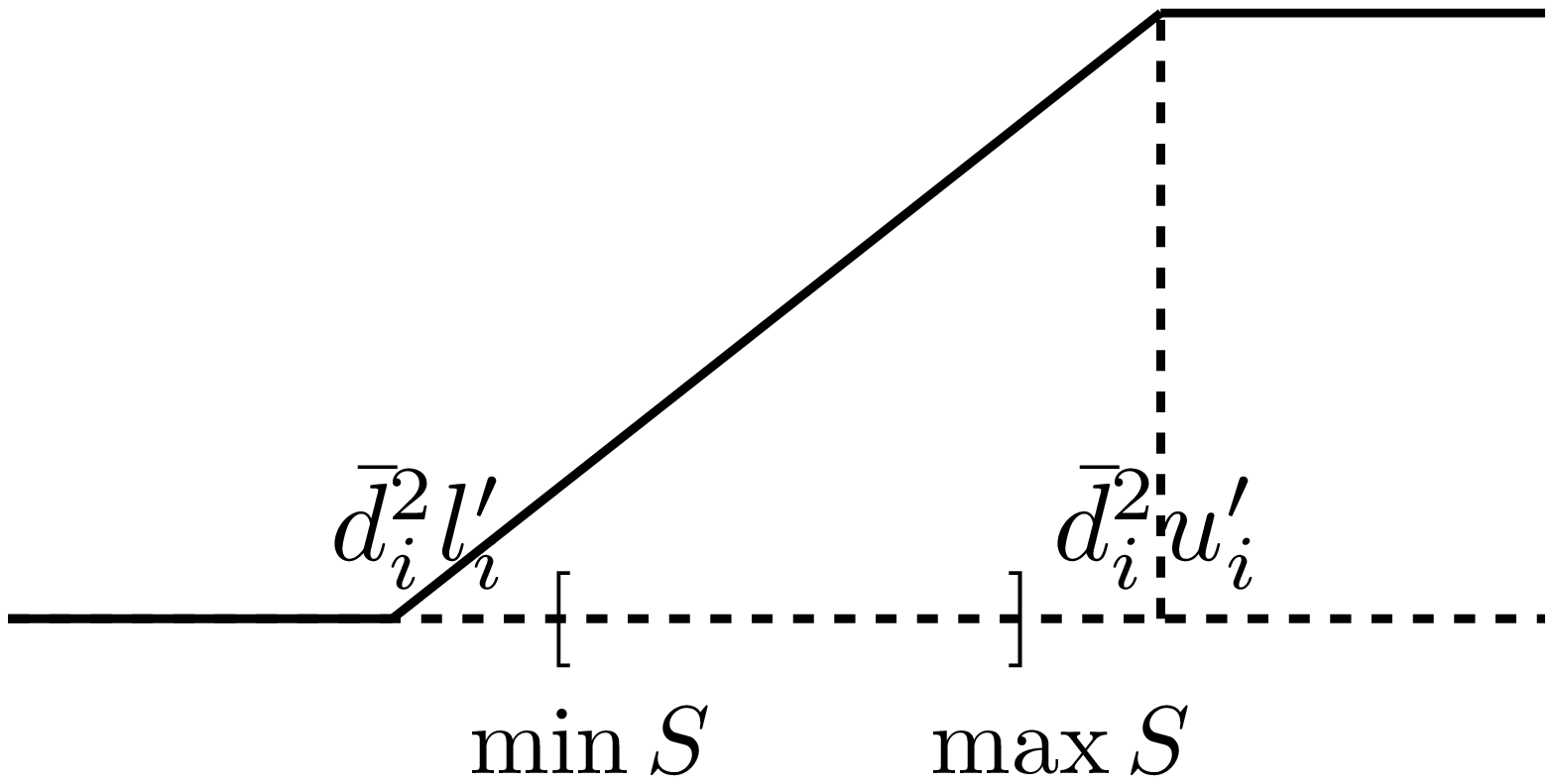}} ~~
\subfloat[$\dbar^{2}_i u'_i \le \min S < \max S$]{
    \includegraphics[width=5.4cm]{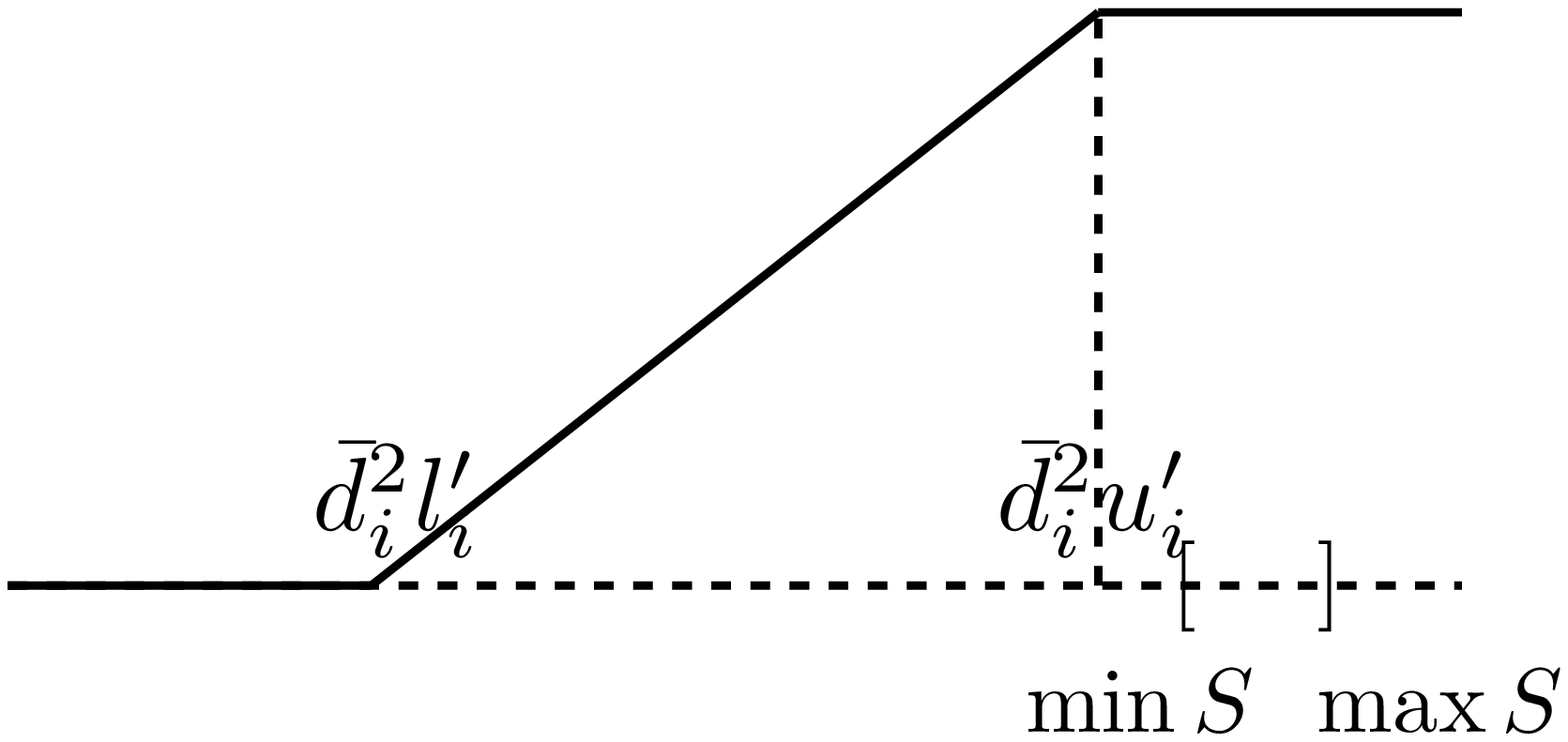}} ~~ \\
\caption{All possible locations of $\min S$ and $\max S$ on $h_i(\lambda)$.}
\label{fig:hi_interval}
\end{centering}
\end{figure*}
}

To minimize the objective in \eqref{eq:minimax_dual} as a function of $\lambda$, we notice that $H_i(\lambda)$ is convex and differentiable. Thus, the minimizer of \eqref{eq:minimax_dual} is exactly the root of its gradient.  Note the gradient of $H_i$:
\begin{align*}
  h_i(\lambda) = \begin{cases}
    u'_i & \text{if } \lambda > u'_i \dbar^{2}_i \\
    l'_i & \text{if } \lambda < l'_i \dbar^{2}_i  \\
    \frac{\lambda}{\dbar^{2}_i} & \text{if } \lambda \in [l'_i \dbar^{2}_i, u'_i \dbar^{2}_i]
  \end{cases}.
\end{align*}
See Figure \ref{fig:hi_lambda} for the plot of $h_i(\lambda)$.  So we need to find the root of the gradient of \eqref{eq:minimax_dual}:
\begin{align}
\label{eq:lambda_find_root}
  f(\lambda) := \sum_{i=1}^n h_i(\lambda) - z' = 0.
\end{align}

Note that $h_i(\lambda)$ is a monotonically increasing function of $\lambda$, so the
whole $f(\lambda)$ is monotonically increasing in $\lambda$.  Since $f(\infty) \ge 0$
by $z' \le \sum_i u'_i$ and $f(-\infty) \le 0$ by $z' \ge \sum_i l'_i$, the root must
exist.  Considering that $f$ has at most $2n$ kinks (nonsmooth points) and is linear
between two adjacent kinks, the simplest idea is to sort $\cbr{\dbar^{2}_i l'_i,
\dbar^{2}_i u'_i : i \in [n]}$ into $s^{(1)} \le \ldots \le s^{(2n)}$.  If $f(s^{(i)})$
and $f(s^{(i+1)})$ have different signs, then the root must lie between them and can be
easily found because $f$ is linear in $[s^{(i)}, s^{(i+1)}]$.  This algorithm takes at
least $O(n \log n)$ time because of sorting.


\begin{algorithm}[t]
\begin{algorithmic}[1]
   \caption{\label{algo:linear_simple_qp}$O(n)$ algorithm to find the root of $f(\lambda)$. Do not allow duplicate points in $S$.}

   \STATE Initialize kink set $S \leftarrow \cbr{\dbar_i^2 l'_i, \dbar_i^2
      u'_i: i \in [n]}$.  Remove duplicates if any.

   \WHILE{$\abr{S} > 2$}

       \STATE Find the median of $S$: $m \leftarrow \MED(S)$ \;

       \IF{$f(m) = 0$}
            \STATE \textbf{Return} $m$.
       \ELSIF{$f(m) > 0$}
            \STATE $S \leftarrow \cbr{x \in S : x \le m}$.
       \ELSE
            \STATE $S \leftarrow \cbr{x \in S : x \ge m}$.
       \ENDIF

   \ENDWHILE

   \STATE \textbf{Return} $\frac{l f(u) - u f(l)}{f(u) - f(l)}$ if $S = \cbr{l,u : f(l) \neq f(u)}$, or any value in $[l, u]$ if $S = \cbr{l < u : f(l) = f(u)}$.
\end{algorithmic}
\end{algorithm}

However, this cost can be reduced to $O(n)$ by making use of the fact that the
median of $n$ (unsorted) elements can be found in $O(n)$ time.  Notice that due to the
monotonicity of $f$, $f$ evaluated at the median of a set $S$ is exactly the median of function values, \ie, $f(\MED(S)) = \MED(\cbr{f(x):x \in S})$.  Algorithm \ref{algo:linear_simple_qp}
shows the binary search.  Let $\abr{S}$ denote the cardinality of $S$.  The while loop must terminate in order $\log_2 (2n)$ iterations because in each iteration the cardinality of set $S$ is reduced to at most $\frac{\abr{S}}{2} + 1$ (we will call it ``almost halves").  So if $f(m)$ can be evaluated in $O(\abr{S})$ time, then the time complexity of each iteration is linear in $|S|$, and the total complexity of Algorithm \ref{algo:linear_simple_qp} is $O(n)$.  Step 7 and 9 ensure that $\abr{S} = 2$ at step 12.

The evaluation of $f(m)$ potentially involves summing up $n$ terms as in
\eqref{eq:lambda_find_root}.  However by carefully aggregating the slope and offset, this can be reduced to $O(|S|)$ too.  In more detail, let us first consider all the possible locations of $\min S$ and $\max S$ on $h_i(\lambda)$ as illustrated in Figure \ref{fig:hi_interval}.  By halving the set $S$, the possible transfers of situation are shown in Figure \ref{fig:state_transfer}.  Once the set $S$ gets into the states $(d), (e), (f)$, its state will never change with the shrinking of $S$, and the contribution of $h_i(\lambda)$ to $f(\lambda)$ will be determined by: $l'_i$ for case $(d)$, $u'_i$ for case $(f)$ and $\frac{\lambda}{\dbar^{2}_i}$ for case $(e)$.  So we keep two buffers: $c_g \in \RR$ which aggregates the contribution by all the $h_i$ ending in state $(d)$ or $(f)$, and $s_g \in \RR$ which aggregates the slope $\frac{1}{\dbar^{2}_i}$ for all $h_i$ ending in state $(e)$.  In other words, to evaluate $f(m)$ we only need to visit those $h_i$ which are still in state $(a)$, $(b)$ and $(c)$ (called undetermined states).  But how many such $i$ can there be?  By Figure \ref{fig:hi_interval}, these $h_i$ all contribute at least one kink point in $S$ (state $(a)$ contributes two).  If $\cbr{\dbar_i^2 l'_i, \dbar_i^2 u'_i: i \in [n]}$ are distinct, then the points in $S$ has one-to-one correspondence to the kink points of $h_i$.  Therefore, the number of $h_i$ in undetermined states must be upper bounded by the size of $S$.  Since the size of $S$ almost halves in each iteration, so is number of $h_i$ in undetermined states.  As a result, the cost for computing $f(m)$ halves too.  Overall, running Algorithm \ref{algo:linear_simple_qp} to completion, the total time spent on evaluating $f(m)$ in step 4 is $O(n)$.

The analysis becomes a bit more complicated when $\cbr{\dbar_i^2 l'_i, \dbar_i^2 u'_i: i \in [n]}$ contains duplicate points.  In this case, one point in $S$ may correspond to kink points of \emph{multiple} $h_i$, and so the above argument can no longer be used to upper bound the number of $h_i$ in undetermined states.  The simplest patch is to add small perturbations to the duplicate points and make them different.  A more principled solution is given in Algorithm \ref{algo:linear_simple_qp_dup}.  The key idea is to allow duplicates in $S$, and replace $S \leftarrow \cbr{x \in S : x \le m}$ in step 7 of Algorithm \ref{algo:linear_simple_qp} by $S \leftarrow \cbr{x \in S : x < m}$ (and similarly step 9).  An additional level of if-then-else check is introduced so as not to miss out the solution.  Clearly, the size of $S$ still halves in Algorithm \ref{algo:linear_simple_qp_dup}.  More importantly, because we do allow the duplicates in $S$, so the size of $S$ is an upper bound of the number of $h_i$ which is in undetermined states.  Therefore, the cost for computing $f(m)$ and $f(y)$ halves through iterations, and the total time spent on evaluating $f(m)$ and $f(y)$ is $O(n)$.

Note that the duplication removal in Algorithm \ref{algo:linear_simple_qp_dup} actually cannot be done in $O(n)$ time, and is subject to numerical precision.  In our experiment, we used Algorithm \ref{algo:linear_simple_qp_dup} which does not remove duplicates.  The correctness is easy to prove, and in practice there is almost no duplicates and it works very well.

\begin{algorithm}[t]
\begin{algorithmic}[1]
   \caption{\label{algo:linear_simple_qp_dup}$O(n)$ algorithm to find the root of $f(\lambda)$. Allow duplicate kink points in $S$.}

   \STATE Initialize kink set $S \leftarrow \cbr{\dbar_i^2 l'_i, \dbar_i^2
      u'_i: i \in [n]}$.  Keep duplications and so $\abr{S} = 2n$.

   \WHILE{$\abr{S} > 2$}

       \STATE Find the median of $S$: $m \leftarrow \MED(S)$. \;

       \IF{$f(m) = 0$}
            \STATE \textbf{Return} $m$.

       \ELSIF{$f(m) > 0$}

            \STATE Find $y := \max \cbr{x \in S : x < m}$.

            // $\cbr{x \in S : x < m}$ must be nonempty.

            \IF{$f(y) > 0$}
                \STATE $S \leftarrow \cbr{x \in S : x < m}$.
            \ELSE
                \STATE $S \leftarrow \cbr{y, m}$.  // Root lies in $[y, m]$, so exit the while loop immediately.
            \ENDIF

       \ELSE

            \STATE Find $y := \min \cbr{x \in S : x > m}$.

            // $\cbr{x \in S : x > m}$ must be nonempty.

            \IF{$f(y) < 0$}

                \STATE $S \leftarrow \cbr{x \in S : x > m}$.

            \ELSE
                \STATE $S \leftarrow \cbr{m, y}$.  // Root lies in $[m, y]$, so exit the while loop immediately.
            \ENDIF
        \ENDIF
   \ENDWHILE

   \STATE \textbf{Return} $\frac{l f(u) - u f(l)}{f(u) - f(l)}$ if $S = \cbr{l,u : f(l) \neq f(u)}$, or any value in $[l, u]$ if $S = \cbr{l < u : f(l) = f(u)}$.
\end{algorithmic}
\end{algorithm}

\begin{figure}[t]
\begin{centering}
\includegraphics[width=5cm]{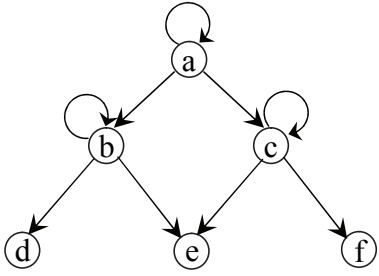}
\caption{All possible transitions of state.}
\label{fig:state_transfer}
\end{centering}
\end{figure}

\subsubsection{Elastic net}

For the first type of elastic net \eqref{eq:obj_elastic_net}, the composite optimization is easy thanks to the separability.  The second type which uses constraints is much more challenging, and we show in this section how to solve this constrained optimization in linear time.  Our approach is similar to the previous Section \ref{sec:proj_simplex}.

At each iteration of AGM-EF-$\infty$ or AGM-EF-1, we need to solve
\begin{align*}
  \min_{\wvec} \lambda \rbr{\gamma \nbr{\wvec}_1 + \frac{1}{2}\nbr{\wvec}_2^2} + \frac{L}{2} \nbr{\wvec - \gvec_i}^2.
\end{align*}
Since all dimensions of $\wvec$ are decoupled, each $w_i$ can be solved separately as a one dimensional optimization problem.  In fact, its solution enjoys a simple closed form \citep[][p. 384]{Passty79}:
\begin{align*}
  \wvec_i^* = \begin{cases}
    \frac{L g_i - \gamma \lambda}{\lambda+ L} & \text{if } \lambda < L g_i / \gamma \\
    0 & \text{if } \lambda \ge L g_i / \gamma
  \end{cases}.
\end{align*}

A more difficult version of elastic net is based on constraints, where in each iteration one needs to solve
\begin{align}
\label{eq:elas_net_per_iter}
  \min_{\wvec} \quad &\frac{1}{2} \nbr{\wvec - \gvec}^2  \\
  s.t. \quad &\gamma\nbr{\wvec}_1 + \frac{1}{2} \nbr{\wvec}_2^2 \le \lambda.
\end{align}

Clearly the optimal $w_i$ has the same sign as $g_i$, hence we can assume $g_i \ge 0$ without loss of generality.  Next we follow the same idea as in Section \ref{sec:proj_simplex} and reformulate \eqref{eq:elas_net_per_iter} into a one dimensional root finding problem.  First write out the Lagrangian:
\begin{align*}
  &\min_{\wvec} \max_{\lambda \ge 0} \frac{1}{2} \nbr{\wvec - \gvec}^2 + \lambda \rbr{\gamma \nbr{\wvec}_1 + \frac{1}{2} \nbr{\wvec}_2^2 - r} \\
  \Leftrightarrow \ \ & \max_{\lambda \ge 0} \min_{\wvec} \underbrace{\frac{1}{2} \nbr{\wvec - \gvec}^2 + \lambda \rbr{\gamma \nbr{\wvec}_1 + \frac{1}{2} \nbr{\wvec}_2^2 - r}}_{=:f_{\lambda}(\wvec)}
\end{align*}
where the equivalence is based on a simple check of Slater's condition.  For each fixed $\lambda$, the optimal $\wvec$ can be found by setting the subgradient to $\zero$.
\begin{align*}
  \frac{\partial}{\partial w_i} f_{\lambda}(\wvec) = w_i - g_i + \lambda w_i + \lambda \gamma \cdot \begin{cases}
    1 & \text{if } w_i > 0 \\
    -1 & \text{if } w_i < 0 \\
    [-1, 1] & \text{if } w_i = 0
  \end{cases}.
\end{align*}
Therefore, the optimal solution is
\begin{align}
\label{eq:elas_net_lambda_to_w}
w^*_i = \begin{cases}
  \frac{g_i - \lambda \gamma}{1 + \lambda} & \text{if } \lambda \le g_i / \gamma \\
  0 & \text{if } \lambda > g_i / \gamma
\end{cases}.
\end{align}
Plugging it back to $f_{\lambda}(\wvec)$ we get the one dimensional optimization problem in $\lambda$:
\begin{align*}
  H(\lambda) = -r\lambda \gamma + \sum_{i=1}^p \begin{cases}
    \frac{\lambda (-\lambda \gamma^2 + g_i^2 + 2 g_i \gamma)}{2(1 + \lambda)} &\text{if } \lambda \le g_i / \gamma  \\
    \frac{g_i^2}{2} & \text{if } \lambda > g_i / \gamma
  \end{cases}.
\end{align*}
It is easy to see that $H(\lambda)$ is concave in $[-1, \infty)$.  So its maximizer is $0$ or the root of its derivative.
\begin{align*}
  H'(\lambda) &= - r \gamma + \sum_{i=1}^p \begin{cases}
    \frac{-\gamma^2 \lambda^2 - 2 \gamma^2 \lambda + 2 g_i \lambda + g_i^2}{2(\lambda + 1)^2}  & \text{if } \lambda \le g_i / \gamma \\
    0 & \text{if } \lambda > g_i / \gamma
  \end{cases} \\
  &= \frac{1}{2(\lambda + 1)^2} h(\lambda)
\end{align*}
where
\begin{align*}
  h(\lambda) &= -2 \gamma r \rbr{\lambda + 1}^2 \\
  & \quad + \sum_{i=1}^p \begin{cases}
    -\gamma^2 \rbr{\lambda + 1}^2 + \rbr{\gamma + g_i}^2 & \text{if } \lambda \le g_i / \gamma \\
    0 & \text{if } \lambda > g_i / \gamma
  \end{cases}.
\end{align*}
Clearly $h(\lambda)$ is monotonically decreasing in $[0, \infty)$.  So $H(\lambda)$ is maximized at $0$ if $h(0) \le 0$, \ie
\[
r \ge -\frac{p}{2} \gamma + \frac{1}{2 \gamma} \sum_{i=1}^p (\gamma + g_i)^2.
\]
Otherwise, $h(\lambda)$ has a root in $[0, \infty)$.  Since it monotonically decreases, the binary search trick in Section \ref{sec:proj_simplex} can also be applied here.  Once it is determined that the optimal $\lambda$ is less than a set of $g_i$, these quadratics can be aggregated by summing up the $\gamma \rbr{g_i + \frac{1}{\gamma}}^2$.  Finally, $\wvec$ is recovered by \eqref{eq:elas_net_lambda_to_w}.

\subsection{Optimizing the Prox-function}

When smoothing $\gstar$, we have often used prox-function $d(\xvec) = \sum_i x^2_i$.  However, it is possible to improve the condition number by using an optimized prox-function.  This idea was used by \citep{Nesterov07} where the \lcg\ constant of a quadratic $\frac{1}{2} \xvec^{\top} H \xvec + \inner{\bvec}{\xvec}$ ($\xvec \in \RR^p$) is upper bounded by $p$ when the norm is chosen as $\nbr{\xvec}^2 = \sum_i H_{ii} x_i^2$, \ie\ rescaling all dimensions.

Using this idea, we show in this section that a data dependent optimization of the prox-function can improve the condition number of the smoothed variant of the primal objective as discussed in Section \ref{sec:ml_sol_primal_smooth}.

Let us consider the following simple but illustrative example.  Suppose $Q_2 = [0, c]^n$, $g(\uvec) = -\sum_{i=1}^n u_i$.  Denote $A = (\avec_1, \ldots, \avec_n)^{\top}$.  We adopt a prox-function
\[
d(\uvec) = \frac{1}{2} \sum_{i=1}^n b_i^2 u_i^2,
\]
and we can derive $\gstar_{\mu} = (g + \mu d)^{\star}$.  The diameter of $Q_2$ under $d$ is
\begin{align}
\label{eq:diameter_q2_opt_bi}
  D = \max_{\xvec \in Q_2} d(\uvec) = \frac{c^2}{2} \sum_i b_i^2.
\end{align}
For any prescribed accuracy $\epsilon > 0$, we first choose $\mu$ such that $\mu D \le \epsilon$, \ie\ $\mu \le \frac{\epsilon}{D}$.  Then our goal is to find the $b_i$ which minimizes the Lipschitz constant of the gradient of $\gstar_{\mu}(A \wvec)$ wrt $\wvec$.

First compute $\gstar_{\mu}(A \wvec)$:
\[
\gstar_{\mu}(A \wvec) = \sup_{\uvec \in Q_2} \inner{A \wvec}{\uvec} + \sum_i u_i - \frac{\mu}{2} \sum_i b^2_i u_i^2.
\]
It is easy to see that the optimal $\uvec^*$ is
\[
u^*_i = \text{MED}\rbr{0, c, \frac{\inner{\avec_i}{\wvec} + 1}{\mu b_i^2}}.
\]
where $\text{MED}$ stands for the median.  So the gradient of $\gstar_{\mu}(A \wvec)$ wrt $\wvec$ can be calculated by
\begin{align*}
\frac{\partial}{\partial \wvec} \gstar_{\mu}(A \wvec) \! = \! \! \sum_{i=1}^n \gvec_i, \ \text{where }
\gvec_i \! = \! \begin{cases}
    0 & \text{if } u^*_i = 0 \\
    \avec_i & \text{if } u^*_i = c \\
    \frac{\inner{\avec_i}{\wvec} + 1}{\mu b_i^2} \avec_i & \text{else}
  \end{cases}.
\end{align*}
So the Hessian of $\gstar_{\mu}(A \wvec)$ in $\wvec$ can only take value in
\[
H_{\delta} = \frac{1}{\mu} \sum_i  \frac{\delta_i}{b_i^2} \avec_i \avec_i^{\top}, \where \delta_i \in \cbr{0, 1}.
\]
Now for any $\wvec_1, \wvec_2 \in Q_1$, denote $l = \nbr{\wvec_1 - \wvec_2}$ and $\vvec = (\wvec_2 - \wvec_1) / l$ (so $\nbr{\vvec} = 1$).  Denote $\hvec(t) = \frac{\partial}{\partial \wvec} \gstar_{\mu} (A(\wvec_1 + t \vvec))$.  So
\begin{align*}
&\frac{\nbr{\frac{\partial}{\partial \wvec} \gstar_{\mu}(A \wvec_1) - \frac{\partial}{\partial \wvec} \gstar_{\mu}(A \wvec_2)}}{\nbr{\wvec_1 - \wvec_2}} = \frac{\nbr{\hvec(l) - \hvec(0)}}{l} \\
&\overset{(a)}{\le} \nbr{\grad \hvec(\xi)} \overset{(b)}{=} \nbr{H_{\delta} \vvec} \overset{(c)}{\le} \lambda_{\max} (H_{\delta}).
\end{align*}
Here, (a) is by the mean value theorem with $\xi \in [0, l]$. (b) is by the chain rule and the $\delta$ for $H_{\delta}$ is determined by $\xi$.  (c) is because for any real positive semi-definite matrix $H$, $\max_{\nbr{\vvec}=1} \nbr{H \vvec} = \lambda_{\max}(H)$.

Clearly $\lambda_{\max}(H_{\delta})$ is maximized when all $\delta_i=1$ and let us call it $H_{\one}$. In conjunction with \eqref{eq:diameter_q2_opt_bi} and \eqref{eq:choice_smooth_mu}, we minimize $\lambda_{\max} (H_{\one})$ wrt $b_i$:
\begin{align}
&\min_{b_i} \lambda_{\max} (H_{\one}) = \min_{b_i} \frac{D}{\epsilon} \lambda_{\max} \rbr{\sum_i \frac{1}{b_i^2} \avec_i \avec_i^{\top}} \nonumber \\
&= \frac{c^2}{2 \epsilon} \min_{b_i} \cbr{\rbr{\sum_i b_i^2} \max_{\vvec: \nbr{\vvec}=1} \vvec^{\top} \sum_i b_i^{-2} \avec_i \avec_i^{\top} \vvec} \nonumber \\
\label{eq:obj_fix_bi_one}
& = \frac{c^2}{2 \epsilon} \max_{\nbr{\vvec}=1} \min_{b_i} \rbr{\sum_i b_i^2} \sum_i b_i^{-2} (\avec_i^{\top} \vvec)^2 \\
\label{eq:obj_max_eig_opt}
&= \frac{c^2}{2 \epsilon} \max_{\nbr{\vvec}=1} \rbr{\sum_i \abr{\avec_i^{\top} \vvec}}^2 \text{(Cauchy-Schwartz)} \\
& \Leftrightarrow \max_{\nbr{\vvec}=1} \sum_i \abr{\avec_i^{\top} \vvec}. \nonumber
\end{align}
However, this last optimization problem is hard so we maximize an approximation of it
\begin{align*}
  \max_{\nbr{\vvec}=1} \sum_i \abr{\avec_i^{\top} \vvec}^2 = \max_{\nbr{\vvec}=1} \vvec^{\top} \rbr{\sum_i \avec_i \avec_i^{\top}} \vvec.
\end{align*}
The solution is the eigenvector $\vvec^*$ corresponding to the maximum eigenvalue of $\sum_i \avec_i \avec_i^{\top}$.  Then $b^2_i$ can be recovered by using the optimality condition of Cauchy-Schwartz in \eqref{eq:obj_max_eig_opt}:
\[
b_i^2 = \abr{\avec_i^{\top} \vvec^*}.
\]
Note \citep{ZhoTaoWu10} used the heuristic that $b_i^2 = \nbr{\avec_i}_{\infty}$.  We can also compare with the isotropic $d$, \ie\ $b_i = 1$.  Simply plug $b_i = 1$ into \eqref{eq:obj_fix_bi_one}, and we get
\[
\frac{c^2}{2 \epsilon} n \sum_i (\avec_i^{\top} \vvec)^2
\]
which must be greater than or equal to
\[
\frac{c^2}{2 \epsilon} \rbr{\sum_i \abr{\avec_i^{\top} \vvec}}^2
\]
in \eqref{eq:obj_max_eig_opt} for all $\vvec$.  Therefore with a fixed $\epsilon$, our approach does possibly reduce the \lcg\ constant of $\gstar(A\wvec)$ in $\wvec$.  The maximum eigenvector can be found very efficiently by using the power iteration, and usually 5 to 6 iterations is enough.

\section{Experimental Results}
\label{sec:experiment}

We will present the experimental result in a later version.

\section{Discussion}


A lot of efforts (\eg, \citep{HuKwoPan09,Xiao10}) have been devoted to making Nesterov's method online, \ie\ use a stochastic gradient oracle and preserve the $1/\sqrt{\epsilon}$ rate of convergence for the expected gap.  This however turns out hopeless as was shown by the lower bounds in \citep{Lan10,GhaLan10}.

\bibliographystyle{unsrtnat}
\bibliography{./bibfile}

\begin{thebibliography}{47}
\providecommand{\natexlab}[1]{#1}
\providecommand{\url}[1]{\texttt{#1}}
\expandafter\ifx\csname urlstyle\endcsname\relax
  \providecommand{\doi}[1]{doi: #1}\else
  \providecommand{\doi}{doi: \begingroup \urlstyle{rm}\Url}\fi

\bibitem[Zhang et~al.(2011)Zhang, Saha, and Vishwanathan]{ZhaSahVis10a}
Xinhua Zhang, Ankan Saha, and S.V.N. Vishwanathan.
\newblock Lower bounds on rate of convergence of cutting plane methods.
\newblock In \emph{Advances in Neural Information Processing Systems 23}, 2011.

\bibitem[Lu(2009)]{Lu09}
Zhaosong Lu.
\newblock Smooth optimization approach for sparse covariance selection.
\newblock \emph{SIAM Journal on Optimization}, 19\penalty0 (4):\penalty0
  1807--1827, 2009.

\bibitem[Liu et~al.(2009)Liu, Chen, and Ye]{LiuCheYe09}
Jun Liu, Jianhui Chen, and Jieping Ye.
\newblock Large-scale sparse logistic regression.
\newblock In \emph{ACM SIGKDD Conference on Knowledge Discovery and Data
  Mining}, 2009.

\bibitem[Gilpin et~al.(2008)Gilpin, Sandholm, and Sorensen]{GilSanSor08}
Andrew Gilpin, Tuomas Sandholm, and Troels~Bjerre Sorensen.
\newblock A heads-up no-limit texas hold'em poker player: Discretized betting
  models and automatically generated equilibrium-finding programs.
\newblock In \emph{International Joint Conference on Autonomous Agents and
  Multiagent Systems}, 2008.

\bibitem[Zhang et~al.(2009)Zhang, Saha, and Vishwanathan]{ZhaSahVis09}
Xinhua Zhang, Ankan Saha, and S.V.N. Vishwanathan.
\newblock Lower bounds for {BMRM} and faster rates for training {SVMs}.
\newblock Technical report arXiv:0909.1334, 2009.
\newblock URL \url{http://arxiv.org/abs/0909.1334}.

\bibitem[Nesterov(2005{\natexlab{a}})]{Nesterov05}
Yurii Nesterov.
\newblock Smooth minimization of non-smooth functions.
\newblock \emph{Math. Program.}, 103\penalty0 (1):\penalty0 127--152,
  2005{\natexlab{a}}.

\bibitem[Platt(1999)]{Platt99a}
John~C. Platt.
\newblock Fast training of support vector machines using sequential minimal
  optimization.
\newblock In \emph{Advances in Kernel Methods\,---\,Support Vector Learning},
  pages 185--208. {MIT} Press, 1999.

\bibitem[Bartlett et~al.(2006)Bartlett, Jordan, and McAuliffe]{JorBarMcA06}
Peter~L. Bartlett, Michael~I. Jordan, and Jon~D. McAuliffe.
\newblock Convexity, classification, and risk bounds.
\newblock \emph{Journal of the American Statistical Association}, 101\penalty0
  (473):\penalty0 138--156, 2006.

\bibitem[Tseng(2009)]{Tseng08}
Paul Tseng.
\newblock On accelerated proximal gradient methods for convex-concave
  optimization.
\newblock \emph{submitted to SIAM Journal on Optimization}, 2009.

\bibitem[Beck and Teboulle(2009)]{BecTeb09}
Amir Beck and Marc Teboulle.
\newblock A fast iterative shrinkage-thresholding algorithm for linear inverse
  problems.
\newblock \emph{SIAM Journal on Imaging Sciences}, 2\penalty0 (1):\penalty0
  183--202, 2009.

\bibitem[Zou and Hastie(2005)]{ZouHas05}
Hui Zou and Trevor Hastie.
\newblock Regularization and variable selection via the elastic net.
\newblock \emph{Journal of Royal Statistics Society. B}, 67\penalty0
  (2):\penalty0 301--320, 2005.

\bibitem[Mairal et~al.(2010)Mairal, Bach, Ponce, and Sapiro]{MaiBacPonSap10}
Julien Mairal, Francis Bach, Jean Ponce, and Guillermo Sapiro.
\newblock Online learning for matrix factorization and sparse coding.
\newblock \emph{Journal of Machine Learning Research}, 11:\penalty0 19--60,
  2010.

\bibitem[Warmuth et~al.(2008)Warmuth, Glocer, and Vishwanathan]{WarGloVis08}
Manfred~K. Warmuth, Karen~A. Glocer, and S.~V.~N. Vishwanathan.
\newblock Entropy regularized {LPBoost}.
\newblock In Yoav Freund, Yoav~L\`{a}szl\`{o} Gy\"{o}rfi, and Gy\"{o}rgy
  Tur\`{a}n, editors, \emph{Proc.\ Intl.\ Conf.\ Algorithmic Learning Theory},
  number 5254 in Lecture Notes in Artificial Intelligence, pages 256 -- 271,
  Budapest, October 2008. Springer-Verlag.

\bibitem[d'Aspremont et~al.(2008)d'Aspremont, Banerjee, and
  Ghaoui]{dAsBanElG08}
Alexandre d'Aspremont, Onureena Banerjee, and Laurent~El Ghaoui.
\newblock First-order methods for sparse covariance selection.
\newblock \emph{SIAM Journal on Matrix Analysis and Applications}, 30\penalty0
  (1):\penalty0 56--66, 2008.

\bibitem[Jain et~al.(2009)Jain, Kulis, Dhillon, and Grauman]{JaiKulDhiGra09}
Prateek Jain, Brian Kulis, Inderjit~S. Dhillon, and Kristen Grauman.
\newblock Online metric learning and fast similarity search.
\newblock In \emph{Advances in Neural Information Processing Systems}, 2009.

\bibitem[Kulis and Bartlett(2010)]{KulBar10}
Brian Kulis and Peter~L Bartlett.
\newblock Implicit online learning.
\newblock In \emph{Proc.\ Intl.\ Conf.\ Machine Learning}, 2010.

\bibitem[Nesterov(2007)]{Nesterov07}
Yurii Nesterov.
\newblock Gradient methods for minimizing composite objective function.
\newblock Technical Report~76, CORE Discussion Paper, UCL, 2007.

\bibitem[Nesterov(1983)]{Nesterov83}
Yurii Nesterov.
\newblock A method for unconstrained convex minimization problem with the rate
  of convergence {$O$}$(1/k^2)$.
\newblock \emph{Soviet Math. Docl.}, 269:\penalty0 543--547, 1983.

\bibitem[Nesterov(2003)]{Nesterov03a}
Yurii Nesterov.
\newblock \emph{Introductory Lectures On Convex Optimization: A Basic Course}.
\newblock Springer, 2003.

\bibitem[Nesterov(2005{\natexlab{b}})]{Nesterov05a}
Yurii Nesterov.
\newblock Excessive gap technique in nonsmooth convex minimization.
\newblock \emph{SIAM J. on Optimization}, 16\penalty0 (1):\penalty0 235--249,
  2005{\natexlab{b}}.
\newblock ISSN 1052-6234.

\bibitem[Nemirovski(1994)]{Nemirovski94}
Arkadi Nemirovski.
\newblock Efficient methods in convex programming.
\newblock Lecture notes, 1994.

\bibitem[Pong et~al.(2010)Pong, Tseng, Ji, and Ye]{PonTseJiYe10}
Ting~Kei Pong, Paul Tseng, Shuiwang Ji, and Jieping Ye.
\newblock Trace norm regularization: Reformulations, algorithms, and multi-task
  learning.
\newblock \emph{SIAM Journal on Optimization}, 2010.

\bibitem[Auslender and Teboulle(2006)]{AusTeb06}
Alfred Auslender and Marc Teboulle.
\newblock Interior gradient and proximal methods for convex and conic
  optimization.
\newblock \emph{SIAM Journal on Optimization}, 16\penalty0 (3):\penalty0
  697--725, 2006.

\bibitem[Shalev-Shwartz(2007)]{Shalev-Shwartz07}
Shai Shalev-Shwartz.
\newblock \emph{Online Learning: Theory, Algorithms, and Applications}.
\newblock PhD thesis, The Hebrew University of Jerusalem, July 2007.

\bibitem[Teo et~al.(2010)Teo, Vishwanthan, Smola, and Le]{TeoVisSmoLe10}
Choon~Hui Teo, S.~V.~N. Vishwanthan, Alex~J. Smola, and Quoc~V. Le.
\newblock Bundle methods for regularized risk minimization.
\newblock \emph{J. Mach. Learn. Res.}, 11:\penalty0 311--365, January 2010.

\bibitem[Lan et~al.(2009)Lan, Lu, and Monteiro]{LanLuMon09}
Guanghui Lan, Zhaosong Lu, and Renato D.~C. Monteiro.
\newblock Primal-dual first-order methods with $\mathcal{O}(1/\epsilon)$
  iteration complexity for cone programming.
\newblock \emph{Mathematical Programming}, 2009.

\bibitem[Bertsekas(1995)]{Bertsekas95}
D.~P. Bertsekas.
\newblock \emph{Nonlinear Programming}.
\newblock Athena Scientific, Belmont, MA, 1995.

\bibitem[Auslender and Teboulle(2004)]{AusTeb04}
Alfred Auslender and Marc Teboulle.
\newblock Interior gradient and epsilon-subgradient descent methods for
  constrained convex minimization.
\newblock \emph{Mathematics of Operations Research}, 29\penalty0 (1):\penalty0
  1--26, 2004.

\bibitem[Borwein and Lewis(2000)]{BorLew00}
J.~M. Borwein and A.~S. Lewis.
\newblock \emph{Convex Analysis and Nonlinear Optimization: Theory and
  Examples}.
\newblock CMS books in Mathematics. Canadian Mathematical Society, 2000.

\bibitem[Bennett and Mangasarian(1992)]{BenMan92}
K.~P. Bennett and O.~L. Mangasarian.
\newblock Robust linear programming discrimination of two linearly inseparable
  sets.
\newblock \emph{Optim. Methods Softw.}, 1:\penalty0 23--34, 1992.

\bibitem[Joachims(2005)]{Joachims05}
T.~Joachims.
\newblock A support vector method for multivariate performance measures.
\newblock In \emph{Proc.\ Intl.\ Conf.\ Machine Learning}, pages 377--384, San
  Francisco, California, 2005. Morgan Kaufmann Publishers.

\bibitem[Lafferty et~al.(2001)Lafferty, McCallum, and Pereira]{LafMcCPer01}
J.~D. Lafferty, A.~McCallum, and F.~Pereira.
\newblock Conditional random fields: Probabilistic modeling for segmenting and
  labeling sequence data.
\newblock In \emph{Proceedings of International Conference on Machine
  Learning}, volume~18, pages 282--289, San Francisco, CA, 2001. Morgan
  Kaufmann.

\bibitem[Taskar et~al.(2004)Taskar, Guestrin, and Koller]{TasGueKol04}
B.~Taskar, C.~Guestrin, and D.~Koller.
\newblock Max-margin {M}arkov networks.
\newblock In S.~Thrun, L.~Saul, and B.~Sch\"{o}lkopf, editors, \emph{Advances
  in Neural Information Processing Systems 16}, pages 25--32, Cambridge, MA,
  2004. MIT Press.

\bibitem[Ji and Ye(2009)]{JiYe09}
Shuiwang Ji and Jieping Ye.
\newblock An accelerated gradient method for trace norm minimization.
\newblock In \emph{Proc.\ Intl.\ Conf.\ Machine Learning}, 2009.

\bibitem[Do et~al.(2009)Do, Le, and Foo]{DoLeFoo09}
C.~Do, Q.~Le, and C.S. Foo.
\newblock Proximal regularization for online and batch learning.
\newblock In \emph{International Conference on Machine Learning {ICML}}, 2009.

\bibitem[Chapelle(2007)]{Chapelle06}
Olivier Chapelle.
\newblock Training a support vector machine in the primal.
\newblock \emph{Neural Computation}, 19\penalty0 (5):\penalty0 1155--1178,
  2007.

\bibitem[Zhang et~al.(2010)Zhang, Saha, and Vishwanathan]{ZhaSahVis10b}
Xinhua Zhang, Ankan Saha, and S.V.N. Vishwanathan.
\newblock Faster rates for training max-margin markov networks.
\newblock Technical report arXiv:1003.1354, 2010.
\newblock URL \url{http://arxiv.org/abs/1003.1354}.

\bibitem[Pardalos and Kovoor(1990)]{ParKov90}
P.~M. Pardalos and N.~Kovoor.
\newblock An algorithm for singly constrained class of quadratic programs
  subject to upper and lower bounds.
\newblock \emph{Mathematical Programming}, 46:\penalty0 321--328, 1990.

\bibitem[Duchi et~al.(2008)Duchi, Shalev-Shwartz, Singer, and
  Chandrae]{DucShaSigCha08}
John Duchi, Shai Shalev-Shwartz, Yoram Singer, and Tushar Chandrae.
\newblock Efficient projections onto the $\ell_1$-ball for learning in high
  dimensions.
\newblock In \emph{Proc.\ Intl.\ Conf.\ Machine Learning}, 2008.

\bibitem[Liu and Ye(2009)]{LiuYe09}
Jun Liu and Jieping Ye.
\newblock Efficient euclidean projections in linear time.
\newblock In \emph{Proc.\ Intl.\ Conf.\ Machine Learning}. Morgan Kaufmann,
  2009.

\bibitem[Passty(1979)]{Passty79}
G.~B. Passty.
\newblock Ergodic converence to a zero of the sum of monotone operators in
  {H}ilberts space.
\newblock \emph{Journal of Optimization Theory and Applications}, 72:\penalty0
  383--390, 1979.

\bibitem[Zhou et~al.(2010)Zhou, Tao, and Wu]{ZhoTaoWu10}
Tianyi Zhou, Dacheng Tao, and Xindong Wu.
\newblock {NESVM}: a fast gradient method for support vector machines.
\newblock In \emph{Proc.\ Intl.\ Conf.\ Data Mining}, 2010.

\bibitem[Hu et~al.(2009)Hu, Kwok, and Pan]{HuKwoPan09}
Chonghai Hu, James~T. Kwok, and Weike Pan.
\newblock Accelerated gradient methods for stochastic optimization and online
  learning.
\newblock In \emph{Neural Information Processing Systems}, 2009.

\bibitem[Xiao(2010)]{Xiao10}
Lin Xiao.
\newblock Dual averaging methods for regularized stochastic learning and online
  optimization.
\newblock Technical Report MSR-TR-2010-23, Microsoft Research, 2010.

\bibitem[Lan(2010)]{Lan10}
Guanghui Lan.
\newblock An optimal method for stochastic composite optimization.
\newblock \emph{Mathematical Programming}, 2010.

\bibitem[Ghadimi and Lan(2010)]{GhaLan10}
Saeed Ghadimi and Guanghui Lan.
\newblock "optimal stochastic approximation algorithms for strongly convex
  stochastic composite optimization.
\newblock \emph{Submitted}, 2010.

\bibitem[Hiriart-Urruty and Lemar\'echal(1993)]{HirLem93}
J.B. Hiriart-Urruty and C.~Lemar\'echal.
\newblock \emph{Convex Analysis and Minimization Algorithms, {I} and {II}},
  volume 305 and 306.
\newblock Springer-Verlag, 1993.

\end{thebibliography}

\appendix

\section{Concepts from Convex Analysis}
\label{sec:app:convex_ana}

The following four concepts from convex analysis are used in the paper.
\begin{definition}
  Suppose a convex function $f: \RR^n \to \RRbar$ is finite at
  $\wb$. Then a vector $\gvec \in \RR^n$ is called a subgradient of
  $f$ at $\wb$ if, and only if,
  \begin{align*}
    f(\wb') \geq f(\wb) + \inner{\wb' - \wb}{\gvec} \qquad
    \text{for all } \ \wb'.
  \end{align*}
  The set of all such $\gvec$ vectors is called the subdifferential of
  $f$ at $\wb$, denoted by $\partial_{\wb} f(\wb)$. For any
  convex function $f$, $\partial_{\wb} f(\wb)$ must be nonempty.
  Furthermore if it is a singleton then $f$ is said to be
  \emph{differentiable} at $\wb$, and we use $\nabla f(\wb)$ to
  denote the gradient.
\end{definition}
\begin{definition}
  \label{def:strong-convex}
  A convex function $f:\RR^{n} \to \RRbar$ is strongly
  convex with respect to a norm $\nbr{\cdot}$ if there exists a constant
  $\sigma > 0$ such that $f - \frac{\sigma}{2} \|\cdot\|^{2}$ is convex.
  $\sigma$ is called the modulus of strong convexity of $f$, and for
  brevity we will call $f$ $\sigma$-strongly convex.
\end{definition}
\begin{definition}
  \label{def:lip-cont-grad}
  Suppose a function $f: \RR^n \to \RRbar$ is differentiable on $Q
  \subseteq \RR^n$.  Then $f$ is said to have Lipschitz continuous
  gradient (\lcg) with respect to a norm $\|\cdot\|$ if there exists a
  constant $L$ such that
  \begin{align*}
    \nbr{\nabla f(\wb) - \nabla f(\wb')}^* \leq L \nbr{ \wb - \wb'} \quad
    \forall\ \wb, \wb'\in Q.
  \end{align*}
For brevity, we will call $f$ $L$-\lcg.
\end{definition}
\begin{definition}
  \label{def:fenchel_dual}
  The Fenchel dual of a function $f: \RR^n \to \RRbar$, is a function $f^{\star}:
  \RR^n \to \RRbar$ defined by
  \begin{align*}
    f^{\star}(\wb^{\star}) = \sup_{\wb \in \RR^n}
    \cbr{\inner{\wb}{\wb^{\star}} - f(\wb)}
  \end{align*}
\end{definition}
Strong convexity and \lcg\ are related by Fenchel duality according to
the following lemma:
\begin{theorem}[{\citep[][Theorem 4.2.1 and 4.2.2]{HirLem93}}]
$\phantom{.}$
\label{theorem:SC_LCG}
\begin{enumerate}
\item If $f: \RR^n \to \RRbar$ is $\sigma$-strongly convex, then $f^{\star}$ is
  finite on $\RR^n$ and $f^{\star}$ is $\frac{1}{\sigma}$-\lcg.
\item If $f: \RR^n \to \RR$ is convex, differentiable on $\RR^n$, and $L$-\lcg,
  then $f^{\star}$ is $\frac{1}{L}$-strongly convex.
\end{enumerate}
\end{theorem}
Finally, the following lemma gives a useful characterization of the
minimizer of a convex function.
\begin{lemma}[{\citep[][Theorem 2.2.1]{HirLem93}}]
  A convex function $f$ is minimized at $\wb^*$ if, and only if, $\zero
  \in \partial f(\wb^*)$. Furthermore, if $f$ is strongly convex, then
  its minimizer is unique.
\end{lemma}

\end{document}